\documentclass{article}
\pdfoutput=1

\usepackage[utf8x]{inputenc}
\usepackage[T1]{fontenc}

\usepackage[accepted]{sty/icml2023}

\usepackage{microtype}
\usepackage{subfigure}
\usepackage{booktabs} % for professional tables
\usepackage{amsfonts, amsmath, amssymb, amsthm}
\usepackage{xfrac}
\usepackage{graphicx, xcolor, colortbl}
\usepackage{footnote}
\usepackage{tablefootnote}
\usepackage{natbib}
\usepackage{dsfont}
\usepackage{bbm}
\usepackage{algorithm}
\usepackage{algorithmic}
\usepackage{import}
\usepackage{mathtools}
\usepackage{bm}
\usepackage{xspace}
\usepackage{thmtools,thm-restate} %to restate lemma
\usepackage{accents}
\usepackage{framed}
\usepackage{enumitem}
\usepackage{booktabs}
\usepackage{color, colortbl}
\usepackage{nicefrac}

\usepackage{sty/notations}
\theoremstyle{plain}
\newtheorem{theorem}{Theorem}[section]

\newtheorem{lemma}[theorem]{Lemma}
\newtheorem{corollary}[theorem]{Corollary}
\theoremstyle{definition}
\newtheorem{definition}[theorem]{Definition}

\theoremstyle{remark}
\newtheorem{remark}[theorem]{Remark}

\usepackage{multirow}
\usepackage{framed}
\usepackage{booktabs}
\usepackage{colortbl}
\usepackage{pifont}

\usepackage{minitoc}

\begin{document}

\twocolumn[
\icmltitle{Fast Rates for Maximum Entropy Exploration}



\begin{icmlauthorlist}
\icmlauthor{Daniil Tiapkin}{hse,airi}
\icmlauthor{Denis Belomestny}{essen,hse}
\icmlauthor{Daniele Calandriello}{deepmind}
\icmlauthor{\' Eric Moulines}{polytechnique,mbzuai}
\icmlauthor{R\'emi Munos}{deepmind}
\icmlauthor{Alexey Naumov}{hse}
\icmlauthor{Pierre Perrault}{idemia}
\icmlauthor{Yunhao Tang}{deepmind}
\icmlauthor{Michal Valko}{deepmind}
\icmlauthor{Pierre M\' enard}{lyon}
\end{icmlauthorlist}

\icmlaffiliation{hse}{HSE University}
\icmlaffiliation{airi}{Artificial Intelligence Research Institute}
\icmlaffiliation{essen}{Duisburg-Essen University}
\icmlaffiliation{idemia}{IDEMIA}
\icmlaffiliation{polytechnique}{\' Ecole Polytechnique}
\icmlaffiliation{mbzuai}{Mohamed Bin Zayed University of AI}
\icmlaffiliation{deepmind}{Google DeepMind}
\icmlaffiliation{lyon}{ENS Lyon}

\icmlcorrespondingauthor{Daniil Tiapkin}{dtyapkin@hse.ru}
%\icmlcorrespondingauthor{Firstname2 Lastname2}{first2.last2@www.uk}

\icmlkeywords{Reinforcement Learning, Maximum Entropy Exploration, Regularization in RL}

\vskip 0.3in
]


\printAffiliationsAndNotice{}  % leave blank if no need to mention equal contribution

% For TOC in appendix (https://tex.stackexchange.com/a/419290)
\doparttoc % Tell to minitoc to generate a toc for the parts
\faketableofcontents % Run a fake tableofcontents command for the partocs

\begin{abstract}
We address the challenge of exploration in reinforcement learning (RL) when the agent operates in an unknown environment with sparse or no rewards. 
In this work, we study the maximum entropy exploration problem of two different types. The first type is \textit{visitation entropy maximization}  previously considered by \citet{hazan2019provably} in the discounted setting. For this type of exploration, we propose a game-theoretic algorithm that has $\tcO(H^3S^2A/\varepsilon^2)$ sample complexity thus improving the $\varepsilon$-dependence upon existing results, where $S$ is a number of states, $A$ is a number of actions, $H$ is an episode length, and $\varepsilon$ is a desired accuracy. The second type of entropy we study is the \textit{trajectory entropy}. This objective function is closely related to the entropy-regularized MDPs, and we propose a simple algorithm that has a sample complexity of order $\tcO(\poly(S,A,H)/\varepsilon)$. Interestingly, it is  the first theoretical result in RL literature that establishes the potential statistical advantage of regularized MDPs for exploration.
Finally, we apply developed regularization techniques to reduce sample complexity of visitation entropy maximization to $\tcO(H^2SA/\varepsilon^2)$, yielding a statistical separation between maximum entropy exploration and reward-free exploration.
\end{abstract}

\section{Introduction}\label{sec:intro}

In reinforcement learning (RL), an agent interacts with an environment  aiming to maximize the sum of rewards returned by the environment. When the reward signal is very sparse or completely absent, the agent may experience long periods without any feedback. In these periods exploration is the main challenge. 

This work studies the problem of efficient \emph{exploration in the absence of rewards}. Approaches to solve this problem can be roughly cast into three main groups: The \emph{bonus-based exploration} where the agent maximizes self-defined bonuses or intrinsic rewards collected along trajectory \citep{schmidhuber1991possibility,oudeyer2007intrisic,bellemare2016unifying}. Typically these bonuses are related to the variances of  error-signals from some auxiliary tasks, such as learning the transition probability distributions \citep{schmidhuber1991possibility,chentanez2004intrinsically, pathak2017curiosity,savas2019entropy},
learning the optimal value function for all the possible rewards \citep{jin2020reward-free,kaufmann2020adaptive,menard2021fast}, learning random generated features \citep{burda2019exploration}. A second approach is the \emph{goal-conditioned exploration} where the agent learns to navigate to self-assigned states (or goals). A common goal-selection rule for this class of algorithms is to select as goals the states at the frontier of the visited states \citep{lim2012autonomous,tarbouriech2020improved, escoffet2019goexplore}. Other selection-goal rules include reaching each state a fixed number of times \citep{tarbouriech2021provably} or going to states where the estimation error for the transition probabilities is large \citep{tarbouriech2020active}. The third approach, which has received relatively less attention thus far, is the \emph{maximum entropy exploration} \citep{hazan2019provably,lee2019efficient,mutti2020intristically,mutti2021task}. This approach involves learning a policy that aims to achieve a visitation distribution over state-action pairs that is as uniform as possible. One specific application of this approach is in unsupervised pretraining, where it helps to obtain a better initial policy \cite{seo21state,zhang2021exploration,mutti2022unsupervised}. To achieve this goal, the approach focuses on maximizing entropy-like functionals, which is the main focus of our study.

% \textcolor{red}{ The third approach, which so far received bit less attention, is the \emph{maximum entropy exploration} \citep{hazan2019provably,lee2019efficient,mutti2020intristically,mutti2021task}; where a policy leading to the as uniform as possible visitation distribution over state-action pairs is learned. In particular, obtaining such policy is useful as a unsupervised pretraining phase to obtain better initial policy \cite{seo21state,zhang2021exploration,mutti2022unsupervised}}. Such goal is usually achieved by maximizing  entropy-like functionals and this is what we study. 

In this work, we focus on environments modeled by an episodic, finite, reward-free Markov Decision Process (MDP) with $S$ states, $A$ actions, horizon $H$ and step-dependent transitions. We consider two types of entropy:  the \emph{visitation entropy} and  the \emph{trajectory entropy}. The visitation entropy of a policy is defined as the sum of the entropies of the visitation distributions induced by the given policy at each step. The trajectory entropy of a policy is given by the entropy of a trajectory, generated when one follows the given policy and seen as one random variable on the corresponding path space. We study maximum entropy exploration under the $(\epsilon,\delta)$-PAC framework, that is, we want to learn, with probability $1-\delta,$ a policy leading to  $\epsilon$-optimal maximum visitation or trajectory entropy
and  using as few as possible interactions with the environment. 

\paragraph{Visitation entropy} \citet{hazan2019provably} study maximum visitation entropy\footnote{Note that \citet{hazan2019provably} consider a slightly different entropy; see Remark~\ref{rem:hazan_entropy_vs_us}.} exploration (MVEE) in the more general framework of convex MDPs where the agent wants to maximize a convex function of the visitation distribution. The authors in  \citet{hazan2019provably} propose to apply the Frank-Wolfe algorithm \citep{frank1956algorithm} to a smoothed version of the visitation entropy. Their algorithm, \MaxEnt\footnote{In this work we refer to \MaxEnt as the algorithm by \citet{hazan2019provably} applied to the visitation entropy and not to the reverse entropy as initially proposed by the authors.}, has a sample complexity of order\footnote{We adapt rates from the $\gamma$-discounted setting by replacing $1/(1-\gamma)$ with $H$. To take into account step-dependent transitions we multiply the first order term by $H^2$.} $\tcO(H^4S^2A/\epsilon^3+S^3/\epsilon^6)$, that is, it needs to sample that  number of trajectories in order to find an $\epsilon$-optimal policy for MVEE.
Later, \citet{cheung2019exploration} obtains a better rate of order $\tcO(H^4S^2A/\epsilon^2+H^3 S/\epsilon^3)$ for the \TocUCRL algorithm. Then, building on the ideas introduced by \citet{abernethy2017frankwolfe}, \citet{zahavy2021reward} reinterpret the \MaxEnt algorithm as a method to compute the equilibrium of a particular game induced by the Legendre-Fenchel transform of the smoothed entropy. Using this new point of view, they propose the \MetaEnt algorithm\footnote{We call \MetaEnt the specialization of their general Meta-algorithm to the special case of MVEE. Note that \MaxEnt, \TocUCRL, \MetaEnt could be seen as variations of the same algorithm. We use different names to distinguish, at least, the associated sample complexity.} with a sample complexity of order $\tcO(H^4S^2A/(\delta^2\epsilon^2)+H^3S/\epsilon^3)$. 
In this work, building on the ideas by \citet{grunwald2002game}, we draw a connection between MVEE and another game. In this game, a \emph{forecaster-player tries to predict the state-action pairs visited by a sampler-player} who aims at \emph{surprising the forecaster-player by visiting not well predicted state-action pairs}. We propose the \EntGame algorithm that tackles MVEE by solving this prediction game. We prove that \EntGame has a sample complexity of order $\tcO(H^4S^2A/\epsilon^2+HSA/\epsilon)$, thus improving over the previous rate in terms of its  dependence of $\epsilon$, see Table~\ref{tab:sample_complexity}. The key technical point leading to this improvement is that, contrary to the previous algorithms, \EntGame does not need to estimate accurately the visitation distribution of a policy at each iteration but only needs \emph{one trajectory generated by following this policy}. Moreover, we propose \regalgMVEE, the regularized version of \algMVEE, that achieves sample complexity of order $\tcO(H^2SA/\epsilon^2)$, additionally improving the previous rates in $S$. The main technique behind this improvement is exploiting a strong connection between visitation entropy and regularization in MDPs. As a result, we have shown that \textit{MVEE is statistically strictly simpler than reward-free exploration} \citep{jin2020reward-free}.

\paragraph{Trajectory entropy} The second problem we consider is maximum trajectory entropy exploration (MTEE). 
The entropy of paths of a (Markovian) stochastic process was first introduced in \citet{ekroot1993entropy}. Intuitively maximizing the trajectory entropy of an MDP minimizes the predictability of paths. Also there is a close connection between MTEE and regularized RL, a very popular approach in practical applications of RL.
%Even if trajectory entropy is a natural choice we show that an optimal policy for trajectory entropy will %not induced a visitation distribution as `spread' as the one obtained with a policy optimal for visitation %entropy.\footnote{See Section~\ref{sec:trajectory_entropy}.} However, as we will see the connection %between trajectory entropy and regularized RL justifies our interest for MTEE. 

Contrary to MVEE, the optimal policy for MTEE can easily be obtained by solving \emph{entropy-regularized Bellman equations} with \emph{entropy of the transition probabilities as rewards}. Leveraging this observation, one can proceed in a similar way as for the best policy identification\footnote{Where in this problem the goal is to identify the optimal policy of a given MDP (equipped with a reward function).} (BPI, \citealt{fiechter1994efficient}). Precisely, we propose two algorithms. The first one, \UCBVIEnt is the simplest one and computes a policy by solving optimistic version of the aforementioned Bellman equations and using it to interact with the environment. The algorithm stops when an upper confidence bound on the difference between the maximum trajectory entropy and the trajectory entropy of the current policy is small enough. The second algorithm, \RFExploreEnt, is an adaptation of the  reward-free exploration  by \citet{jin2020reward-free} to our setting. This algorithm has two phases. In the first phase, we compute a \textit{preliminary exploration policy} which is then used to generate independent trajectories (data). In the second phase, a nearly optimal MTEE policy is obtained by solving the empirical Bellman equations  with transitions estimated from the data collected  in the first phase.

% Precisely, we propose the \RFExpressEnt algorithm that first explores MDP in reward-free fashion and then computes a policy solving optimistic version of the aforementioned Bellman equations and uses it to interact with the environment. The algorithm stops when some upper confidence bound on the difference between the maximum trajectory entropy and the trajectory entropy of the current policy is small enough. 
Interestingly, we prove that \RFExploreEnt enjoys a sample complexity of order $\tcO(\poly(S,A,H)/\epsilon)$. The key technical ingredients to obtain such fast rate are exploitation of the \emph{smoothness introduced by the regularization} and the use of the explicit form of the optimal policy.

% showing a separation between MTEE and other exploration problem such that reward-free exploration (RFE) where the optimal sample complexity is of order $\tcO(H^6S^2A/\epsilon)$. 

\paragraph{Regularized MDPs} Notably we can adapt\footnote{That is replace the entropy of the transition probability by an arbitrary reward function.} our algorithms for MTEE to best policy identification in regularized MDPs \citep{neu2017unified, geist2019theory}. Especially, we consider the same entropy-regularized MDPs and associated Bellman equations as in \SoftQlearning \citep{fox2016taming,schulman2017equivalence,haarnoja2017reinforcement} or \SAC\citep{haarnoja2018soft}, see Remark~\ref{rem:regularized_mdp}.
We show that a variation of \RFExploreEnt has a sample complexity of order $\tcO(\poly(S,A,H)/(\epsilon\lambda))$ for BPI and reward-free exploration in regularized MDPs, where $\lambda$ is the regularization parameter. In particular, it exhibits a \emph{statistical separation between BPI in regularized MDP and BPI in the original MDP} since in this case the optimal sample complexity is of order  $\tcO(H^3SA/\epsilon^2)$ \citep{menard2021fast,domingues2021episodic}. Thus, our analysis shows that regularization is an effective way to trade-off bias for sample complexity. Additionally, we show how to use entropy regularization to obtain a theoretically faster version of \EntGame algorithm.

We highlight our main contributions:
%\vspace{-0.2cm}
\begin{itemize}[itemsep=-2pt,leftmargin=6pt]
    \item We propose the \EntGame algorithm for MVEE with a sample complexity of order $\tcO(H^4S^2A/\varepsilon^2)$ thus significantly improving the existing complexity bounds  for MVEE.
    \item We introduce the new MTEE setting for exploration and provide two algorithm: the \algMTEE algorithm for MTEE with a sample complexity of order $\tcO(H^3 S A/\varepsilon^2)$, and \RFExploreEnt with a sample complexity of order $\tcO(\poly(S,A,H)/\varepsilon)$. Up to our knowledge, this is the first time that a fast rate (in $1/\varepsilon$) is obtained thanks to regularization.
    \item We adapt \algMTEE and \RFExploreEnt to solve the entropy-regularized MDPs with a sample complexity of order $\tcO(H^3SA/\epsilon^2)$ and $\tcO(\poly(S,A,H)/(\lambda \varepsilon))$ correspondingly, where $\lambda$ is the regularization parameter. 
    \item We combine \EntGame algorithm with regularization techniques, resulting in a new algorithm \regalgMVEE. This algorithm improves a sample complexity of \EntGame to $\tcO(H^2SA/\varepsilon^2)$ and shows statistical separation of MVEE from reward-free exploration.
\end{itemize}

% \todoDa{Ideas on general structure, similar to \cite{menard2021fast}:}
% \begin{itemize}
%     \item Maximum entropy: two different formulations
%     \item Visitation entropy. Connection to a general convex MDP setting, refinement of meta-algorithm for $\tcO\left(\frac{H^3 S^2 A}{\varepsilon^2}\right)$ complexity for MaxEnt.
%     \item Trajectory entropy. Bellman equations. Connection to regularized MDPs, soft Q-learning thought the lens of new objective function. Two algorithms: regularization-agnostic $\tcO\left(\frac{H^3 S A}{\varepsilon^2} \right)$ and regularization-aware $\tcO \left(\frac{H^6 S^2 A}{\varepsilon} \right)$.
% \end{itemize}

\begin{table}[t]
\setlength{\tabcolsep}{4pt}
\centering
\label{tab:sample_complexity}
\begin{tabular}{@{}lcl}\toprule
\textbf{Algorithm} & \textbf{Setting} & \textbf{Sample complexity} \\
\midrule
{\scriptsize\MaxEnt}~{\tiny{\citep{hazan2019provably}}} & \multirow{5}{*}{MVEE} & {\scriptsize$\tcO(H^4S^2A/\epsilon^3\!+\!S^3/\epsilon^6)$} \\
{\scriptsize\TocUCRL}~{\tiny{\citep{cheung2019exploration}}} & &{\scriptsize$\tcO(H^4S^2A/\epsilon^2\!+\!H^3 S/\epsilon^3)$} \\
{\scriptsize\MetaEnt}~{\tiny{\citep{zahavy2021reward}}} &  & {\scriptsize$\tcO(H^4S^2A/(\delta^2\epsilon^2)\!+\!H^3 S/\epsilon^3)$} \\
 \rowcolor[gray]{.90} {\scriptsize\algMVEE}~{\tiny(this paper)} &  & {\scriptsize$\tcO(H^4S^2A/\epsilon^2\!+\!HSA/\epsilon)$} \\
 \rowcolor[gray]{.90} {\scriptsize\regalgMVEE}~{\tiny(this paper)} &  & {\scriptsize$\tcO(H^2SA/\epsilon^2\!+\!H^8S^4A/\epsilon)$} \\
\midrule 
 \rowcolor[gray]{.90} {\scriptsize\algMTEE}~{\tiny (this paper)} & \multirow{1}{*}{MTEE}  & {\scriptsize$\tcO(H^3SA/\epsilon^2 + H^3S^2A/\varepsilon)$} \\
 \rowcolor[gray]{.90} {\scriptsize\RFExploreEnt}~{\tiny (this paper)} & & {\scriptsize$\tcO(H^8S^4A/\epsilon)$} \\
\bottomrule
\end{tabular}
\caption{Sample complexities for MVEE and MTEE. We convert rates from the $\gamma$-discounted setting by replacing $1/(1-\gamma)$ with $H$ or from the infinite horizon setting by replacing the diameter with $H$. To take into account step-dependent transitions we multiply the first order term by $H^2$. {\scriptsize(For \MetaEnt since they do not precisely specify the cost for estimating a visitation distribution we use the same $1/\epsilon^3$ term as for \TocUCRL.)}}
% \todoPi{Is it just $1/\epsilon^3$ for the second order term for \citet{zahavy2021reward}?}
% \todoDa{They do not consider the complexity of density estimation and I guess that it should be like the second term from \citet{hazan2019provably}, but anyway it will be at least $1/\varepsilon^2$ per each step (and the number of steps is $1/\varepsilon$). We can leave remark like that. }}
%\todom{mention \url{https://arxiv.org/pdf/1905.06466.pdf} - perhaps add to this table too}
\end{table}

\vspace{-0.55cm}
\section{Setting}
\label{sec:setting}

We consider a finite episodic \emph{reward-free} MDP $\cM = \left(\cS, \cA, H, \{p_h\}_{h\in[H]}, s_1\right)$, where $\cS$ is the set of states of size $S$, $\cA$ is the set of actions of size $A$, $H$ is the number of steps in one episode, $p_h(s'|s,a)$ is the probability transition from state~$s$ to state~$s'$ by performing action $a$ in step $h$. And $s_1$ is the fixed initial state.

\paragraph{Policy \& value functions} A general policy $\pi =(\psi_h)_{h\in[H]}$ is a collection of function $\psi_h:\, (\cS\times\cA\times[0,1])^{h-1}\times \cS\times[0,1] \to \cA$ that maps an history $I_h=(s_1,a_1,u_1,\ldots,s_{h-1},a_{h-1},u_{h-1})$ where $u_h$ are i.i.d. uniformly distributed on the unit interval, a state $s_h$ and an auxiliary independent uniformly distributed random variable $u_h$ to an action $a_h= \psi_h(I_h,s_h,u_h)$. A policy is Markovian if the action depends only on the previous state and the auxiliary noise $a_h=\psi_h(s_h,u_h)$. In this case the policy can be represented as $\pi=(\pi_h)_{h\in[H]}$ a collection of mappings from states to probability distributions over actions $\pi_h : \cS \to \simplex_\cA$ for all $h\in [H]$ where $a_h=\psi_h(s_h,u_h)\sim \pi_h(s_h)$.
% A \emph{stochastic} policy $\pi=(\pi_h)_{h\in[H]}$ is a collection of mappings from states to probability distributions over actions $\pi_h : \cS \to \simplex_\cA$ for all $h\in [H]$, where each $\pi_h$ maps each state to a distribution over actions.
Furthermore, $p_{h} f(s, a) \triangleq \E_{s' \sim p_h(\cdot | s, a)} \left[f(s')\right]$   denotes the expectation operator with respect to the transition probabilities $p_h$ and
$(\pi_h g)(s) \triangleq \pi_h g(s) \triangleq \E_{a \sim \pi_h(s)}[g(s,a)]$ denotes the composition with policy~$\pi$ at step $h$. Also, for any distribution over actions $\pi \in \simplex_\cA$ define $\pi g(s) \triangleq \E_{a \sim \pi}[g(s,a)]$.
% \todoPi{Check if we need that latter.}
% The value functions of $\pi$, denoted by $V_h^\pi$, as well as the optimal value functions, denoted by $\Vstar_h$ are given by the Bellman respectively  optimal Bellman equations
% \begin{small}
% \begin{align*}
% 	Q_h^{\pi}(s,a) &= r_h(s,a) + p_h V_{h+1}^\pi(s,a) & V_h^\pi(s) &= \pi_h Q_h^\pi (s)\\
%   Q_h^\star(s,a) &=  r_h(s,a) + p_h V_{h+1}^\star(s,a) & V_h^\star(s) &= \max_a Q_h^\star (s, a)
% \end{align*}
% \end{small}%
% \!where by definition, $V_{H+1}^\star \triangleq V_{H+1}^\pi \triangleq 0$. 

\paragraph{Visitation distribution} The \emph{state-action visitation distribution} of policy $\pi$ at step $h$ is denoted by $d_h^\pi$, where $d_h^\pi(s,a)$ is the probability of reaching the state-action pair $(s,a)$ at step $h$ after policy $\pi$.

\paragraph{Visitation polytopes} All the admissible collection of visitation distributions belong to the following polytope 
{\scriptsize
\begin{align*}
\cK_p &\triangleq \Big\{d=(d_h)_{h\in[H]}: \sum_{a\in\cA} d_1(s,a) = \ind\{s=s_1\}\ \forall s\in\cS\,   \\
\sum_{a\in\cA}& d_{h+1}(s,a) =\!\!\!\!\! \sum_{(s',a')\in\cS\times\cA}\!\!\!\!\! p_h(s|s',a') d_h(s',a')\ \forall s\in\cS, \forall h\geq 1\Big\}\,.
\end{align*}}
\!\!We also denote by $\cK$ the set of collections of probability distributions over state-action pairs without the constraint to be a valid visitation distribution, that is 
{\scriptsize
\begin{align*}
\cK \triangleq \Big\{d=(d_h)_{h\in[H]}&: d_h(s,a) \geq 0,\  \forall (h,s,a)\in[H]\times\cS\times\cA\,   \\
&\sum_{(s,a)\in\cS\times\cA} d_{h}(s,a) = 1,\ \forall h\in[H]\Big\}\,.
\end{align*}}

\paragraph{Trajectory distribution} We denote by {\small$\cT\triangleq (\cS\times\cA)^H = \big\{(s_1,a_1,\ldots,s_H,a_H)\!:\, (s_h,a_h)\in\cS\times\cA,\, \forall h\in[H]\big\}$} the set of all possible trajectories. The probability to generate the trajectory $m=(s_1,a_1,\ldots,s_H,a_H)$ with he policy $\pi$ is {\small$q^\pi(m) \triangleq \pi(a_1|s_1)\prod_{h=2}^H p_{h-1}(s_h|s_{h-1},a_{h-1}) \pi_h(a_h|s_h) $}. 
Note that the visitation distribution at step $h$ of policy $\pi$ is a marginal of the trajectory distribution $d_h^\pi(s,a) = \E_{(s_1,a_1,\ldots,s_H,a_H) \sim q^\pi} [\ind\{(s,a)=(s_h,a_h)\}]$.

\paragraph{Counts and empirical transition probability} the number of times the state action-pair $(s,a)$ was visited in step $h$ in the first $t$ episodes are $n_h^{t}(s,a) \triangleq  \sum_{i=1}^{t} \ind{\left\{(s_h^i,a_h^i) = (s,a)\right\}}$. Next, we define $n_h^{t}(s'|s,a) \triangleq \sum_{i=1}^{t} \ind{\big\{(s_h^i,a_h^i,s_{h+1}^i) = (s,a,s')\big\}}$ the number of transitions from $s$ to $s'$ at step $h$. %\todoRemi{Shall we write this $n_h^{t}(s,a,s')$ instead?}
The empirical distribution is defined as $\hp^{\,t}_h(s'|s,a) = n^{\,t}_h(s'|s,a) / n^{\,t}_h(s,a)$ if $n_h^t(s,a) >0$ and $\hp^{\,t}_h(s'|s,a) \triangleq 1/A$ for all $s'\in \cS$ else.

\paragraph{Additional notation} For $n\in\N_{+}$ we define the set $[n]\triangleq \{1,\ldots,n\}$. For $n\in\N_{+}$ we denote by $\Delta_n$ the probability simplex of dimension $n$. For elements $(p,q)\in\Delta_n$ the entropy of $p$ is denoted by $\cH(p) \triangleq -\sum_{i\in[n]} p_i\log p_i$ and the Kullback-Leibler divergence between $p$ and $q$ by $\KL(p,q) \triangleq \sum_{i\in[n]} p_i \log(p_i/q_i)$. For a number $x$ and any two number $m < M$ define $\clip(x, m, M) \triangleq \max(\min(x, M), m)$.
%The vector of dimension $N$ with all entries one is $\bOne^N \triangleq  (1,\ldots,1)$.  is an pseudo-empirical measure defined as $\up^{\,t}_h(s,a) \triangleq \upn^{\,t}_h(s'|s,a) / \upn^{\,t}_h(s,a)$ 

\section{Visitation Entropy}
\label{sec:visitation_entropy}
In this section we focus on maximizing the visitation entropy defined below.

\paragraph{Visitation entropy} We define the visitation entropy of a policy $\pi$ denoted by  as the sum of the visitation distribution entropies at each steps
\vspace{-0.3cm}
\[
\VE(d^\pi) \triangleq \sum_{h=1}^H \cH(d_h^\pi)\,\cdot
\]
We denote by $\pistarVE\in\argmax_{\pi} \VE(d^\pi)$ a policy that maximizes the visitation entropy. 

% %\todom{can we say which one is preferable? was it easier for them
% %to show something of their AVG version?}
% %\todopp{I am not sure that the visitation entropy of \citet{hazan2019provably} is what is stated here. Maybe I am wrong, but from my understanding, it should be$\frac{1}{H}\sum_{(h,s,a)} d_h^\pi(s,a) \log\frac{1}{\frac{1}{H}\sum_{h} d_h^\pi(s,a)}$. The inequality is still true by concavity though}
% \todoPi{
% %What is the difference between the entropy you wrote and the one below it seems they are the same. 
% For me the difference between averaging the visitation distribution or not is a matter of do we treat each step state action $(h,s,a)$ separately or we only consider only $(s,a)$. The second option makes more sens in the discounted setting where the transition does not depends on the step and there is a natural limit visitation distribution. In the episodic case is not clear if we want to 'merge' all the visitation distributions}
% %\todopp{Bellow, inside the log, the sum is also over $s,a$ which I think is not correct.}
% %\todoPi{Yes you are right!}
% \todom{Thanks Pierre for noticing this! Given this 'discrepancy' of Hazan's defintion, can does it maka Hazan's less natural, since it is not $\sum_{i\in[n]} p_i\log p_i$ for the same $p$, if I understood correctly? 
% And this is something that we can remark? (I'm fishing here for some reason to justify why we do not follow Hazan 100\%)} 
% \todopp{Does our result hold for both entropies? If, for the Hazan's one, we take as rewards $-\log\big(\frac{1}{H}\sum_{h}d_h^t(s,a)\big)$ in \algMVEE}
% \end{remark}

\paragraph{Maximum visitation entropy exploration} In MVEE the agent interacts with the reward-free MDP $\cM$ as follows. At the beginning of episode $t$, the agent picks a policy $\pi^t$ based only on the transitions collected up to episode $t-1$. Then a new reward-free trajectory is sampled following the policy~$\pi^t$ and observed by the agent. At the end of each episode the agent can decide to stop collecting new data, according to a random stopping time $\tau$, the stopping rule, and outputs a (general) policy $\hpi$ based on the observed transitions. An agent for MVEE is therefore made of a triplet $((\pi^t)_{t\in\N},\tau,\hpi)$. 
%\todom{why is $((\pi^t)_{t\in\N},\tau,\hpi)$ and agent? }
%\todoPi{Non-Markovian output}
\begin{definition} (PAC algorithm for MVEE) An algorithm $((\pi^t)_{t\in\N},\tau,\hpi)$ is $(\epsilon,\delta)$-PAC for MVEE if 
{\small\[
\P\Big( \VE\big(d^{\pistarVE}\big) - \VE(d^{\hpi}) \leq \epsilon\Big) \geq 1-\delta.
\]}
\end{definition}
Our goal is to design an algorithm that is $(\epsilon,\delta)$-PAC for MVEE with as sample complexity $\tau$ as small as possible.

\subsection{MVEE by Solving Game}
\label{sec:MVEE_game}

Following the general framework of \citet{hazan2019provably, zahavy2021reward}, it is possible to solve MVEE by applying the Frank-Wolfe algorithm to a smoothed version of the visitation entropy. Interestingly, \citet{abernethy2017frankwolfe} showed that this procedure is equivalent to computing the Nash equilibrium of a particular game induced by the Legendre-Fenchel transform of the smoothed entropy. In fact, as noted by \citet{grunwald2002game}, there exists another game naturally linked to MVEE, stated next.

\paragraph{Prediction game} Maximum visitation entropy is the value of the following prediction game 
\begin{align*}
    \max_{d\in\cK_p} \VE(d) &= \max_{d\in\cK_p} \min_{\bd\in\cK}\sum_{(h,s,a)} d_h(s,a) \log \frac{1}{\bd_h(s,a)}\\
    &=  \min_{\bd\in\cK} \max_{d\in\cK_p} \sum_{(h,s,a)} d_h(s,a) \log \frac{1}{\bd_h(s,a)}\CommaBin
\end{align*}
see Lemma~\ref{lem:prediction_game} in Appendix~\ref{app:technical} for a proof.
This game can be interpreted as follows. On the one hand, the min player, or forecaster player, tries to predict which state-action pairs the max player will visit to minimize $\KL(d_h,\bd_h)$.  On the other hand, the max player, or sampler player, is rewarded for visiting state-action pairs that the forecaster player did not predict correctly.

We now describe the algorithm \algMVEE\ for MVEE. In this algorithm, we let a forecaster player and a sampler player compete for $T$ episodes long. Let us first define the two players.
\paragraph{Forecaster-player} As forecaster-player we use the Mixture-Forecaster for a logarithmic loss, see Section~9 in \citep{cesabianchi2006prediction}. Fix a prior count $n_0$ and their sum $t_0 = S A n_0$. The forecaster-player predicts at episode $t$ the distributions $\bd^t\in\cK$ with $\bd_h^t(s,a) = \bn_h^{t-1}(s,a) /(t+t_0)$ where the pseudo counts are $ \bn_h^t(s,a) = n_h^t(s,a)+n_0$ and $n_h^t(s,a)$ the counts of state-action pairs visited by the sampler-player.
Note that $\bd_h^t$ can be seen as the posterior mean under a Dirichlet distribution prior $\Dir(\{n_0\}_{(s,a) \in \cS \times \cA})$ on $\cS\times\cA$.

\paragraph{Sampler-player} As sampler-player we choose the optimistic best-response. Define the optimistic Bellman equations 
{\small
\begin{align}
\begin{split}\label{eq:optimistic_planning_VE}
\uQ_h^t(s,a) &=  \log\frac{1}{\bd_h^{t+1}(s,a)} + \hp_h^{\,t} \uV^t_{h+1}(s,a) +b_h^t(s,a), \\
\uV_h^t(s) &= \clip\bigg( \max_{a\in\cA}\uQ_h^t(s,a), 0, \log(t/n_0+SA)H\bigg), 
\end{split}
\end{align}\\}
where  $V_{H+1}^t = 0$ and $b_h^t$ are some Hoeffding-like bonuses defined in \eqref{eq:sampler_exploration_bonus} of Appendix~\ref{app:visitation_entropy_proofs}. The sampler player then plays $d^{\pi^{t+1}}$ where $\pi^{t+1}$ is greedy with respect to the optimistic Q-values, that is, $\pi_h^{t+1}(s) \in\argmax_{\pi\in\Delta_A} \pi \uQ_h^{t}(s)$. 

\paragraph{Sampling rule} At each episode $t$ the policy $\pi^t$ of the sampler-player is used as a sampling rule to generate a new trajectory.

\paragraph{Decision rule} After $T$ episodes we output a non-Markovian policy $\hpi$ defined as the mixture of the policies $\{\pi^t\}_{t\in[T]}$, that is, to obtain a trajectory from $\hpi$ we first sample uniformly at random $t\in[T]$ and then follow the policy $\pi^t$. Note that the visitation distribution of $\hpi$ is exactly the average $d^{\hpi} = (1/T)\sum_{t\in[T]} d^{\pi^t}$.

Remark that the stopping rule of \algMVEE is deterministic and equals to $\tau = T$. The complete procedure is detailed in Algorithm~\ref{alg:ourMVEE}.

\begin{algorithm}[h!]
\centering
\caption{\algMVEE}
\label{alg:ourMVEE}
\begin{algorithmic}[1]
  \STATE {\bfseries Input:} Number of episodes $T$, prior counts $n_0$.
      \FOR{$t \in[T]$}
      \STATE \textcolor{blue}{\# Forecaster-player}
      \STATE Update pseudo counts $\bn_h^{t-1}(s,a)$ and predict $\bd_h^t(s,a)$. 
      \STATE \textcolor{blue}{\# Sampler-player}
      \STATE Compute $\pi^t$ by optimistic planning \eqref{eq:optimistic_planning_VE} with rewards $\log\big(1/ \bd_h^t(s,a)\big)$.
    \STATE \textcolor{blue}{\# Sampling}
      \FOR{$h \in [H]$}
        \STATE Play $a_h^t\sim \pi_h^t(s_h^t)$
        \STATE Observe $s_{h+1}^t\sim p_h(s_h^t,a_h^t)$
      \ENDFOR
    \STATE{ Update counts and transition estimates.}
   \ENDFOR
   \STATE Output $\hpi$ the uniform mixture of $\{\pi^t\}_{t\in[T]}$.
\end{algorithmic}
\end{algorithm}

\begin{theorem}
\label{th:MVEE_sample_complexity}
Fix  $\epsilon > 0,$  $\delta\in(0,1)$ and $n_0=1.$ Then under the choice 
\[
T = \tcO\left( \frac{H^4 S^2 A}{\varepsilon^2} + \frac{HSA}{\varepsilon} \right)
\]
the algorithm \algMVEE is $(\epsilon,\delta)$-PAC. See Theorem~\ref{th:MVEE_sample_complexity_full} in Appendix~\ref{app:visitation_entropy_proofs} for a precise bound.
\end{theorem}
Thus the sample complexity of \algMVEE is of order $\tcO(H^4S^2A/\epsilon^2)$. In particular, this result significantly improves over the previous rate for MTEE, see Table~\ref{tab:sample_complexity}. Note that, by using Bernstein-like bonuses~\citep{azar2017minimax} instead of Hoeffding-like ones for the sampler-player would give a sample complexity of order $\tcO(H^3S^2A/\epsilon^2)$ saving one factor $H$. However, in the Section~\ref{sec:faster_rates_mvee} we present a way to use regularization techniques to achieve a sample complexity of order $\tcO(H^2SA/\epsilon^2)$.

\paragraph{Space and time complexity} Since \algMVEE relies on a model-based algorithm for the sampler-player, its space complexity is of order $\cO(HS^2A)$. Because of the value iteration performed by the sampler-player, the time-complexity of one iteration of \algMVEE is of order $\cO(HS^2A)$.

\begin{remark}
\label{rem:hazan_entropy_vs_us} Note that our definition of the visitation entropy slightly differs from the one considered by \citet{hazan2019provably}. Indeed, their definition, translated to the episodic setting, is the entropy of the average of the visitation distributions which is an upper bound on the average of the entropies by concavity of the entropy 
{\small
\[
\cH\left(\frac{1}{H} \sum_{h=1}^H d_h^\pi\right)\geq \frac{1}{H}\VE(d^\pi)\,.\]\vspace{-0.25cm}\\
} 
Even if both definitions make sense in the episodic setting, we think ours is slightly more appropriate in the case of step-dependent transition probabilities. 
Indeed, in this case we want visitation distributions to be close to the uniform distribution over state-action pairs \emph{for all steps}.
 Nevertheless, \EntGame can be adapted to optimize the visitation entropy  used in \citet{hazan2019provably} simply by predicting $\bd_h^t(s,a) = \sum_{h'=1}^H \bn_{h'}^{t-1}(s,a)/(H(t+t_0))$ for the forecaster-player. We conjecture that the sample complexity of this adaptation of \EntGame for the alternative entropy is again of order $\tcO(HS^2A/\epsilon^2)$.
\end{remark}

\paragraph{Comparison with \MaxEnt  and \MetaEnt} All three algorithms, \algMVEE, \MetaEnt\citep{zahavy2021reward}, \MaxEnt\citep{hazan2019provably} rely on the same principle of computing, implicitly or explicitly, the equilibrium of a well chosen game and deduce from it an optimal policy for MVEE. One first difference between \algMVEE and its competitors lies in the choice of the game. While \MetaEnt, \MaxEnt consider the game induced by the Legendre-Fenchel conjugate of a smoothed visitation entropy \citep{zahavy2021reward}, \algMVEE leverages the prediction game which looks more natural for MVEE. One advantage of using this game, is that it allows to avoid the need to smooth the visitation entropy because it is done implicitly by the forecaster-agent with the pseudo-counts. % In fact,  they have to choose a more conservative equivalent of prior count of order $n_0=\cO(\sqrt{T})$ whereas $n_0=1$ is enough for \algMVEE.
More importantly, \MaxEnt and \MetaEnt both needs to accurately estimate at each episode the visitation distributions of the sampler-player $d_h^{\pi^{t}}$, leading to an extra $1/\epsilon^3$ term in the sample complexity. Whereas \algMVEE  needs one trajectory from $\pi^t$ since it only involves the estimation of the averages $1/T \sum_{t=1}^T d_h^{\pi^t}$.

\section{Trajectory Entropy}
\label{sec:trajectory_entropy}

In this section we focus on another type of entropy, the trajectory entropy, that can be efficiently maximized. The entropy of paths of a (Markovian) stochastic process is introduced by \citet{ekroot1993entropy}. It quantifies the randomness of realizations with fixed initial and final states. Later it was extended \citep{savas2019entropy}  to realizations that reach a certain set of states, rather than a fixed final state.   This type of entropy is also closely related to the so-called entropy rate of a stochastic process.

\paragraph{Trajectory entropy} We define the trajectory entropy of a policy $\pi$ as the entropy of a trajectory generated with the policy $\pi$ 
\[
\TE(q^\pi) \triangleq \cH(q^\pi) = \sum_{m\in\cT} q^\pi(m) \log\frac{1}{q^\pi(m)}\,.
\]

We denote by $\pistarTE\in\argmax_{\pi} \TE(q^\pi)$ a policy that maximizes the trajectory entropy. 

\paragraph{Maximum trajectory entropy exploration} MTEE differs from MVEE only in the choice of entropy. In particular an algorithm $((\pi^t)_{t\in\N_+},\tau,\hpi)$ for MTEE is also a combination of a time dependent policy $(\pi^t)_{t\in\N_+}$, a stopping rule $\tau$, and a decision rule $\hpi$. 
\begin{definition}
(PAC algortihm for MTEE) An algorithm $((\pi^t)_{t\in\N},\tau,\hpi)$ is $(\epsilon,\delta)$-PAC for MTEE if 
\[
\P\left( \TE\big(q^{\pistarTE}\big)- \TE(q^{\hpi}) \leq \epsilon \right) \leq 1-\delta\,.
\]
\end{definition}

\vspace{-0.25cm}
As noted by \citet{eysenbach2019if}, MTEE can  also be connected to a prediction game. In this game, the forecaster-player aims to predict the whole trajectory that the sampler-player will generate. Remark that predicting the trajectory implies to predict, in particular, the visited state-action pairs but the reverse is not true in general \footnote{Indeed $d_h^\pi$ are only the marginals of $q^\pi$.}. We could then apply the same strategy as in Section~\ref{sec:visitation_entropy} to solve MTEE. Nevertheless, for trajectory entropy, there is a more direct way to proceed.

\paragraph{Entropy regularized Bellman equations} One big difference between MVEE and MTEE is that the optimal policy can be obtained by solving regularized Bellman equations. Indeed, thanks to the chain rule for the entropy, the trajectory entropy of a policy $\pi$ is $\TE(d^\pi) = V_1^\pi(s_1)$ and the maximum trajectory entropy is $\TE\big(d^{\pistarTE}\big)  = \Vstar_1(s_1)$ where the value functions $V^\pi$ and $\Vstar$ satisfy
{\small
\begin{align*}
	Q_h^{\pi}(s,a) &= \cH\big(p_h(s,a)\big) + p_h V_{h+1}^\pi(s,a) \,,\\
    V_h^\pi(s) &= \pi_h Q_h^\pi (s) +\cH\big(\pi_h(s)\big)\,,\\
  Q_h^\star(s,a) &=  \cH\big(p_h(s,a)\big) + p_h V_{h+1}^\star(s,a) \,,\\
  V_h^\star(s) &= \max_{\pi\in\Delta_A} \{ \pi Q_h^\star (s) + \cH(\pi)\}\,,
\end{align*}}

\vspace{-0.4cm}
\!where by definition, $V_{H+1}^\star \triangleq V_{H+1}^\pi \triangleq 0$. In particular, the maximum trajectory entropy policy is given by $\pistarTE_h(s) = \argmax_{\pi\in\Delta_A} (\pi\Qstar(s)+\cH(\pi))$.
It can be computed explicitly via $\pistarTE_h(a|s) = \exp\!\big(\Qstar_h(s,a) -\Vstar_h(s)\big)$ as well as the optimal value function $\Vstar_h(s) = \log\left(\sum_{a\in\cA} \rme^{\Qstar_h(s,a)}\right)$. We refer to Appendix~\ref{app:reg_bellman_eq} for a complete derivation.

We now describe our algorithm \RFExploreEnt, the description of \algMTEE is postponed to Appendix~\ref{app:regularized_mdp}.  The idea of the algorithm is rather simple: since we need to solve regularized Bellman equations  to obtain a maximum trajectory entropy policy, we can 1) find a \textit{preliminary exploration policy $\pi^{\mathrm{mix}}$} allowing one to construct  estimates of the transition probabilities which are uniformly  good when computing expectations of arbitrary bounded functions over all policies (see Lemma~\ref{lem:sampling_square_value_error_bound}), and 2) solve the regularized Bellman equations based on the estimated model. A similar approach is used in reward-free exploration \citep{jin2020reward-free,kaufmann2020adaptive,menard2021fast}, and, in particular, our algorithm is close to \RFExplore by \citet{jin2020reward-free}. However, the key difference is that in the presence of regularization a much smaller number of transitions (trajectories) needs to be collected in order to obtain a high quality policy.

\paragraph{Exploration phase}

This phase is devoted to learn a simple (non-Markovian) preliminary exploration policy $\pi^{\mathrm{mix}}$ that could be used to construct a accurate enough estimates of transition probabilities. This policy is obtained, as in \RFExplore, by learning for each state $s$ and step $h$, a policy that reliably reaches state $s$ at step $h$. This can be done by running for $N_0$ iterations any regret minimization algorithm, e.g. \EULER \citep{zanette2019tighter}, for the sparse reward function putting reward one at state $s$ at step $h$ and zero otherwise. The policy $\pi^{\mathrm{mix}}$ is defined as the uniform mixture of the aforementioned policies. Then the policy $\pi^{\mathrm{mix}}$ is used to collect $N$ fresh independent trajectories from the MDP.

\paragraph{Planning phase}
For the planning phase, the agent builds a transition model
\begin{equation}\label{eq:empirical_model}
    \hp_h(s'|s,a) = \begin{cases}
        \frac{n_h(s'|s,a)}{n_h(s,a)} & n_h(s,a) > 0 \\
        \frac{1}{S} & n_h(s,a) = 0
    \end{cases}\,,
\end{equation}
where $n_h(s,a)$ is the number of visits of the state-action pair $(s,a)$ at step $h$ for these $N$ sampled trajectories.
The final policy is a solution to the empirical regularized Bellman equations
{\small
\begin{align}\label{eq:empirical_entropy_bellman}
    \begin{split}
      \hQ_h^\star(s,a) &=  \cH\big(\hp_h(s,a)\big) + \hp_h \hV_{h+1}^\star(s,a) \,,\\
      \hV_h^\star(s) &= \max_{\pi\in\Delta_A}\left\{  \pi \hQ_h^\star (s) + \cH(\pi) \right\}\,,\\
      \hpi_h(s) &= \argmax_{\pi\in\Delta_A}\left\{  \pi \hQ_h^\star (s) + \cH(\pi)\right\} \,.
    \end{split}
\end{align}}

The complete procedure is described in Algorithm~\ref{alg:RFExploreEnt}. We now prove that, for the well-calibrated choice of $N$ and $N_0$ of order $\tcO(\poly(S,A,H)/\varepsilon)$, the \RFExploreEnt algorithm is $(\epsilon,\delta)$-PAC for MTEE and provide an upper bound on its sample complexity. For the proof we refer to Appendix~\ref{app:fast_rates_regularized}.

\begin{theorem}
    The algorithm \RFExploreEnt with parameters $N_0 = \Omega\left( \frac{H^7 S^3 A \cdot L^3}{\varepsilon}\right)$ and $N = \Omega\left( \frac{ H^6 S^3 A  L^5 }{\varepsilon}\right)$ is $(\varepsilon,\delta)$-PAC for the MTEE problem, where $L = \log(SAH/(\varepsilon \delta))$. Its total sample complexity $SHN_0 + N$ is bounded by
    \[
        \tcO\left( \frac{H^8 S^4 A}{\varepsilon} \right).
    \]
\end{theorem}

\begin{algorithm}[h!]
\centering
\caption{\RFExploreEnt}
\label{alg:RFExploreEnt}
\begin{algorithmic}[1]
  \STATE {\bfseries Input:} Target precision $\epsilon$, target probability $\delta$, number of episodes for simple exploration policy $N_0$, number of sampled trajectories $N$.
    \FOR{$(s',h') \in \cS \times [H]$}
        \STATE Form rewards $r_h(s,a) = \ind\{ s=s', h=h'\}$.
        \STATE Run \EULER \citep{zanette2019tighter} with rewards $r_h$ over $N_0$ episodes and collect all policies $\Pi_{s',h'}$.
        \STATE Modify $\pi \in \Pi_{s',h'}:\ \pi_{h'}(a|s') = 1/A$ for all $a\in \cA$.
    \ENDFOR
    \STATE Construct a uniform mixture policy $\pi^{\mathrm{mix}}$ over all $\{\pi\in\Pi_{s,h} : (s,h) \in \cS \times [H] \}$.
    \STATE Sample $N$ independent trajectories $(z_n)_{n\in[N]}$ following $\pi^{\mathrm{mix}}$ in the original MDP.
    \STATE Construct from $(z_n)_{n\in[N]}$ an empirical model $\hp_h$ as in \eqref{eq:empirical_model}.
   \STATE Output policy $\hpi$ as a solution to \eqref{eq:empirical_entropy_bellman}.
\end{algorithmic}
\end{algorithm}

\begin{remark} 
\label{rem:regularized_mdp}
(On solving regularized MDPs) Interestingly, our approach for MTEE can be adapted to solve entropy-regularized MDPs. For a reward functions $(r_h)_{h\in[H]}$ and regularization parameter $\lambda>0,$ consider the regularized Bellman equations 
{\small
\begin{align*}
    Q^{\pi}_{\lambda,h}(s,a) &= r_h(s,a) + p_h V^{\pi}_{\lambda,h+1}(s,a)\,,\\
    V^\pi_{\lambda,h}(s) &= \pi_h Q^\pi_{\lambda,h}(s) + \lambda \cH\big(\pi_h(s)\big)\,,\\
    \Qstar_{\lambda,h}(s,a) &= r_h(s,a) + p_h \Vstar_{\lambda,h+1}(s,a)\,,\\
    \Vstar_{\lambda,h}(s) &= \max_{\pi\in\simplex_A} \pi\Qstar_{\lambda,h}(s) + \lambda \cH(\pi),
\end{align*}}
\!where $V_{\lambda, H+1}^\pi = \Vstar_{\lambda, H+1} = 0$. Note that these are the Bellman equations used by \SoftQlearning \citep{fox2016taming,schulman2017equivalence,haarnoja2017reinforcement} and \SAC \citep{haarnoja2018soft} algorithms. We are interested in the best policy identification for this regularized MDP. That is finding an algorithm that will output an $\epsilon$-optimal policy $\hpi$ such that with probability $1-\delta$, it holds $\Vstar_{\lambda,1}(s_1) - V^{\hpi}_{\lambda,1}(s_1) \leq \epsilon$ after a minimal number $\tau$ of trajectories sampled from the MDP $\cM^r = (\cS,\cA, H ,(p_h)_{h\in[H]},(r_h)_{h\in[H]},s_1)$. 
By using similar exploration and planning phases as in \RFExploreEnt, we get an algorithm for BPI in the entropy-regularized MDP that also enjoys the fast rate of order $\tcO\big(H^8S^4A/(\epsilon \lambda)\big)$. Moreover, this algorithm could be used for more general types of regularization and even in reward-free setting with the same order of the sample complexity.
Refer to Appendix~\ref{app:regularized_mdp}-\ref{app:fast_rates_regularized} for precise statements and proofs.

We observe that the sample complexity for solving the regularized MDP is strictly smaller\footnote{For small enough $\epsilon$.} than the sample complexity for solving the original MDP. Indeed, one needs at least $\tcO(H^3SA/\epsilon^2)$ trajectory \citep{domingues2021episodic} to learn a policy $\pi$ which value in the (unregularized) MDP is $\epsilon$-close to the optimal value. Nevertheless, regularizing the MDP introduces a bias in the value function. Precisely we have for all $\pi$, $0\leq V^\pi_{\lambda,1}(s_1)-  V^\pi_1(s_1) \leq \tcO(\lambda H)$ where 
$V^\pi_1(s_1)$ is the value function of $\pi$ at the initial state and MDP $\cM^r$. Thus, to solve BPI in $\cM^r$ through BPI in the regularized MDP, one needs to take $\lambda = \tcO(1/(H\epsilon))$, leading to a sample complexity of order $\tcO(H^9S^4A/\epsilon^2)$. In particular, our fast rate for BPI in regularized MDP does not contradict the lower bound for BPI in the original MDP. However, our analysis shows that regularization is an effective way to trade-off bias for sample complexity.
\end{remark}

\paragraph{Visitation entropy vs trajectory entropy}
We can compare the visitation entropy and the trajectory entropy with
\[
\TE(q^\pi)\leq  \underbrace{\KL(q^\pi,\otimes_{h=1}^H d^\pi_h) + \TE(q^\pi)}_{\VE(d^\pi)} \leq H \TE(q^\pi)\,,
\]
where $\otimes_{h=1}^H d^\pi_h$ is a product measure, see Lemma~\ref{lem:comparison_ent_traj_visit} in Appendix~\ref{app:technical} for a proof. Note also that in general the visitation distributions of an optimal policy for maximum trajectory entropy will be less 'spread' than the one of an optimal policy for MTEE, see Section~\ref{sec:experiments} for an example. In particular one can prove that the optimal policy for MTEE is the uniform policy if the transitions are deterministic, see Lemma~\ref{lem:MTEE_deterministic} of Appendix~\ref{app:technical}.

% \textcolor{red}{
% Regarding the difference in computational hardness, the MVEE problem is challenging even in the case of known transition probabilities. When relying on the dual representation \citep{zimin2013online}, it requires solving a high-dimensional convex optimisation problem, whereas solving the MTEE problem for a known MDP requires only simple dynamic programming. Thus, we may expect that the MVEE problem is more computationally challenging than the MTEE problem.
% }

\subsection{Proof Sketch}\label{sec:proof_mtee}
In this section we sketch the proof of Theorem~\ref{th:mtee_sample_complexity}. 

\paragraph{Properties of entropy}

We start from analysing several properties of the entropy function. First, we notice that the well-known log-sum-exp function is a convex conjugate to the negative entropy defined only on the probability simplex \cite{boyd2004convex}
%{\small
\[
    F(x) \triangleq \log\left( \sum_{a \in \cA} \rme^{x_a} \right) = \max_{\pi \in \simplex_A} \langle \pi, x \rangle + \cH(\pi)
\]
and extend its action to $Q$-functions
\[
    F(Q)(s) \triangleq \max_{\pi \in \simplex_A} \pi Q(s) + \cH(\pi). 
\]
This definition is useful because we can rewrite the optimal value function for MTEE in real and empirical model as follows
\begin{equation}\label{eq:V_star_using_F}
    \Vstar_{h}(s) = F(\Qstar_h)(s), \quad \hV^{\hpi}_h(s) = F(\hQ^{\hpi}_h)(s).
\end{equation}
Additionally, we notice that the gradient of $F$ is equal to the soft-max policy that maximizes the expressions above
\[
    \pistar_h(s) = \nabla F(\Qstar_h)(s) , \quad \hpi_h(s) = \nabla F(\hQ^{\hpi}_h)(s),
\]
and, moreover, since the negative entropy $-\cH(\pi)$ is $1$-strongly convex with respect to $\ell_1$ norm, gradients of $F$ is $1$-Lipschitz with respect to $\ell_\infty$ norm by the properties of the convex conjugate \cite{kakade2009duality}. Combining the gradient properties with  the smoothness definition to $\Qstar$ and $\hQ^{\hpi}$ we obtain
\begin{small}
\begin{align}\label{eq:F_smooth_Q}
    \begin{split}
    F(\Qstar_h)(s) &\leq F(\hQ^{\hpi}_h)(s) + \hpi_h\left( \Qstar_h- \hQ^{\hpi}_h\right)(s) \\
    &+ \frac{1}{2} \norm{ \Qstar_h - \hQ^{\hpi}_h}^2_\infty(s),
    \end{split}
\end{align}\end{small}
where \begin{small}$\norm{ \Qstar_h - \hQ^{\hpi}_h}_\infty(s) = \max_{a\in\cA} \vert \Qstar_h(s,a) -\hQ^{\hpi}_h(s,a) \vert$.\end{small}

\paragraph{Bound on the policy error}

Next we apply the key inequality \eqref{eq:F_smooth_Q} to analyze the error between the optimal policy and policy $\hpi$. Using~\eqref{eq:V_star_using_F} yields
\begin{small}
\begin{align*}
    \Vstar_h(s) - V^{\hpi}_h(s) &\leq \hV^{\hpi}_h(s) - V^{\hpi}_h(s)  + \hpi_h \big(\Qstar_h - \hQ^{\hpi}_h \big)(s) \\
    &+ \frac{1}{2} \max_{a \in \cA} \left( \hQ^{\hpi}_h(s,a) - \Qstar_h(s,a) \right)^2.
\end{align*}
\end{small}

\vspace{-0.4cm}
Next, by definition of $\hpi$  we have
\begin{small}
$
    \hV^{\hpi}_h(s) = \cH(\hpi_h(s)) +  \hpi_h \hQ^{\hpi}_h(s),
$
\end{small}
therefore by the regularized Bellman equations
\begin{small}
\begin{align*}
    \begin{split}
    \Vstar_h(s) - V^{\hpi}_h(s) & \leq \hpi_h p_h \left( \Vstar_{h+1} - V^{\hpi}_{h+1} \right)(s) \\
    &+\frac{1}{2} \max_{a \in \cA} \left( \hQ^{\hpi}_h(s,a) - \Qstar_h(s,a) \right)^2.
    \end{split}
\end{align*}\\
\end{small}

\vspace{-0.8cm}
Finally, rolling out this expression we have
\begin{small}
\begin{align*}
    \Vstar_1(s_1) - V^{\hpi}_1(s_1) \leq \frac{1}{2}\E_{\hpi}\left[ \sum_{h=1}^H \max_{a \in \cA} \left( \hQ^{\hpi}_h - \Qstar_h \right)^2(s_h,a)\right]\,.
\end{align*}\\
\end{small}
Next we may notice that in the generative model setting\footnote{When there is a sampling oracle for each state-action pair.} there is available results that tells us that $\tcO(1/\varepsilon)$ samples are enough to obtain $\norm{\hQ^{\hpi}_h - \Qstar_h}_\infty \lesssim \sqrt{\epsilon}$ \citep{azar2013minimax}, and we can conclude the statement. However, in the online setup the situation is more complicated, and we apply reward-free techniques developed by \citet{jin2020reward-free} to obtain a "surrogate" of the generative model.

\vspace{-0.25cm}
\section{Faster Rates for Visitation Entropy}
\label{sec:faster_rates_mvee}

In this section, we show how to combine the regularization techniques developed in Section~\ref{sec:trajectory_entropy} with \EntGame algorithm presented in Section~\ref{sec:MVEE_game}.

The new algorithm \regalgMVEE is based on exactly the same game-theoretical framework as \EntGame, but uses a \textit{regularized sampler player} instead of the usual one.

\paragraph{Regularized sampler-player} For the sampler player, we shall take advantage of strong convexity of the visitation entropy. Beforehand, we construct an estimate of the model $\{ \hp_h \}_{h\in[H]}$ by reward-free exploration, using $HSN_0$ samples to compute a policy $\pi^{\mathrm{mix}}$ and $N$ samples to estimate transitions. Next, define the empirical regularized Bellman equations 
{\small
\begin{align*}
\begin{split}
\hQ_h^t(s,a) &=  \log\left(\frac{1}{\bd_h^{t+1}(s)}\right) + \hp_h \hV^t_{h+1}(s,a),\\
\hV_h^t(s) &= \max_{\pi\in\Delta_{\cA}}\{ \pi \hQ_h^t(s) + \cH(\pi) \},
\end{split}
\end{align*}\\}
where  $\hV_{H+1}^t = 0$. The sampler player then follows the distribution $d^{\pi^{t+1}}$ where $\pi^{t+1}$ is greedy with respect to the regularized empirical Q-values, that is, $\pi_h^{t+1}(s) \in\argmax_{\pi\in\Delta_A}\{ \pi\hQ_h^t(s) + \cH(\pi) \}$.

\begin{theorem}
\label{th:fast_MVEE_sample_complexity}
Fix $\epsilon > 0$ and $\delta\in(0,1)$. For $n_0=1,$ 
\begin{small}
\[
    N_0 = \Omega\left( \frac{H^7 S^3 A  \cdot L^3}{\varepsilon}\right),
    \quad
    N = \Omega\left( \frac{ H^6 S^3 A L^5 }{\varepsilon}\right), 
    \]
\end{small}
\!and \begin{small}
\[
    T = \Omega\left( \frac{H^3 S A L^3}{\varepsilon^2} + \frac{H^2 S^2 A^2 L^2}{\varepsilon}\right).
\]
\end{small}
\!with $L = \log(SAH/\delta\varepsilon)$ the algorithm \regalgMVEE is $(\epsilon,\delta)$-PAC. The total sample complexity is equal to $SH \cdot N_0 + N + T,$ that is,
\begin{small}
\[
    \tau = \tcO\left( \frac{H^2 SA}{\varepsilon^2} + \frac{H^8 S^4 A}{\varepsilon} \right).
\]
\end{small}
\end{theorem}

\vspace{-0.2cm}
\!Thus, the sample complexity of \regalgMVEE is of order $\tcO(H^2SA/\epsilon^2)$ for $\varepsilon$ large enough. In particular, this result significantly improves over  the previous rates for MVEE, see Table~\ref{tab:sample_complexity}. Moreover, this result shows a rate separation between reward-free exploration \citep{jin2020reward-free}, where the established lower bound on sample complexity $\Omega(H^3S^2A/\varepsilon^2)$ scales with $S^2$, and the visitation entropy maximization problem. 

\paragraph{Proof idea} The main proof idea is to exploit not just strong convexity of the visitation entropy with respect to Euclidean distance but its strong convexity \textit{with respect to trajectory entropy} \citep{bauschke2016descent}. It allows us to use entropy regularization as described in Section~\ref{sec:proof_mtee} for the sampler player resulting in an averaged regret less than $\varepsilon$ for only $\tcO(\poly(S,A,H)/\varepsilon)$ samples. Thus, the density estimation error becomes the leading term in the full error decomposition. For more details refer to Appendix~\ref{app:reg_visitation_entropy_proofs}.

\vspace{-0.3cm}
\section{Experiments}
\label{sec:experiments}

\begin{figure}
    \centering
    \vspace{-8pt}
    
    \includegraphics[width=0.95\linewidth]{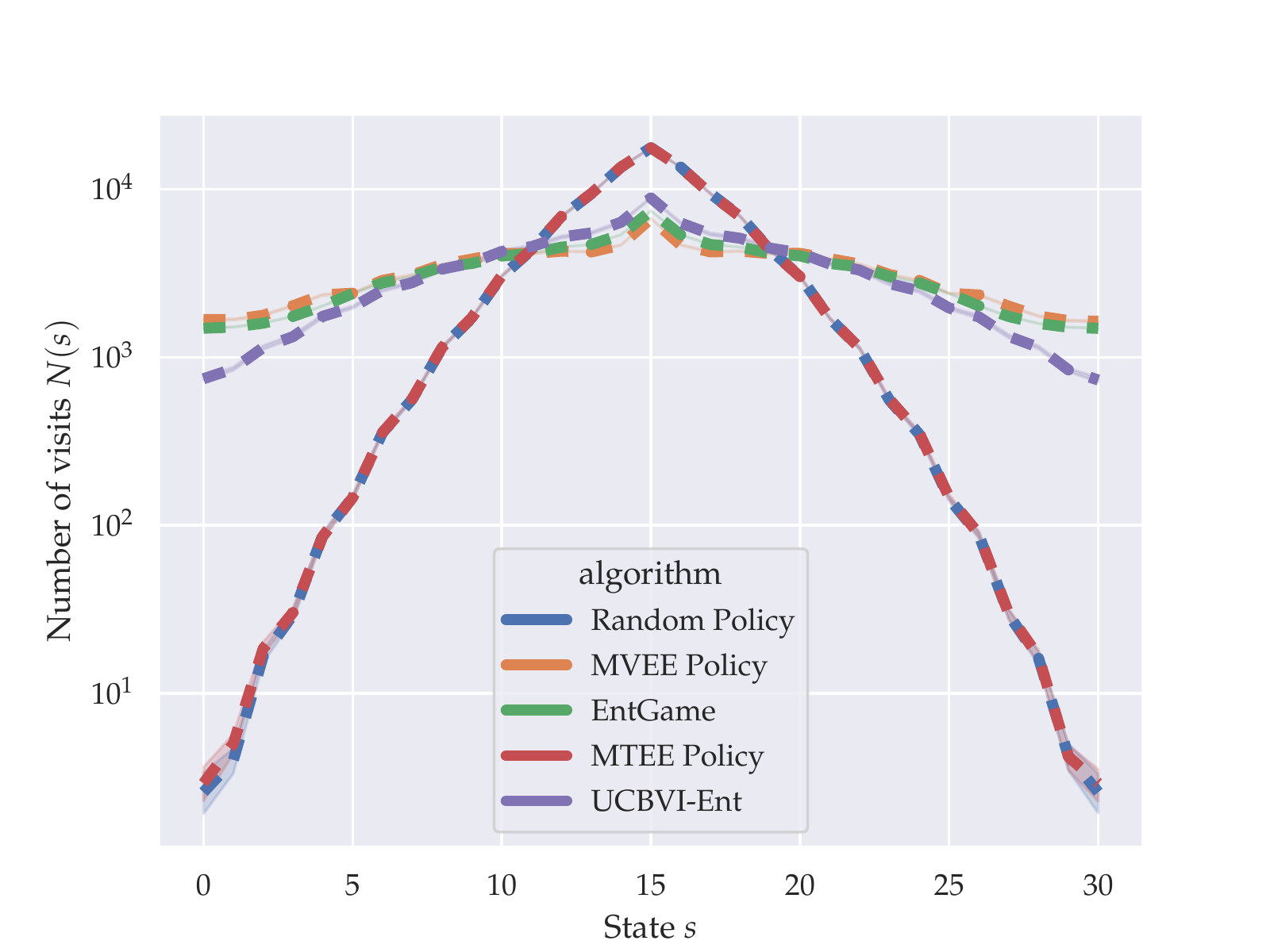}

    \vspace{-10pt}
    \caption{Number of state visits for $N=100000$ samples in Double Chain MDP with $S=31$ states, $A=2$ actions, a horizon $H=20$ and a $0.1$ probability of moving to the opposite direction.} \vspace{-10pt}
    \label{fig:double_chain_main}
\end{figure}

%\vspace{-60pt}

In this section we report experimental results on simple tabular MDP for presented algorithms and show the difference between visitation and trajectory entropies. In particular, we compare \EntGame and \UCBVIEnt algorithms with (a) random agent that takes all actions uniformly at random, (b) an optimal MVEE policy computed by solving the convex program, and (c) an optimal MTEE policy computed by solving the regularized Bellman equations. As an MDP, we choose a stochastic environment called Double Chain as considered by \citet{kaufmann2020adaptive}. 

Since the transition kernel for this environment is stage-homogeneous, for \EntGame and \UCBVIEnt algorithms we joint counters over the different steps $h$. In particular, it changes the objective of the \EntGame algorithm to the objective considered by \citet{hazan2019provably} that makes more sense in the stage-independent setting \footnote{See Remark~\ref{rem:hazan_entropy_vs_us}.}.

In Figure~\ref{fig:double_chain_main} we present the number of state visits for our algorithms and baselines during $N=100000$ interactions with the environment. For \UCBVIEnt algorithm the procedure was separated on two stages: at first we learn MDP with $N$-sample budget and extract the final policy, and then plot the number of visits for the final policy during another $N$ samples. 

In particular, we see that since this environment is almost deterministic the optimal MTEE policy is almost coincides with a random policy. Notably, the policy induced by \UCBVIEnt is more uniform over states because of $1/\sqrt{n^t(s,a)}$ bonuses, that make our algorithm close to \RFUCRL \cite{kaufmann2020adaptive}. Also we see that the optimal MVEE policy is the most uniform over states, that makes it an appropriate target for the pure exploration problem. For more details and additional experiments we refer to Appendix~\ref{app:experiments}.

% Empirical visitation distribution plot (as Fig 1.b and Fig 3.b by \citet{kaufmann2020adaptive}) for the algorithms 
% \begin{itemize}
% \item Random policy
% \item \RFUCRL by \citet{kaufmann2020adaptive}
% \item \algMVEE
% \item \algMTEE
% \end{itemize}
% in the environment 
% \begin{itemize}
%     \item doublechain
%     \item gridworld
% \end{itemize}
%\vspace{-0.25cm}
\section{Conclusion}
\label{sec:conclusion}

In this work we studied MVEE for which we provided the \EntGame algorithm with a sampling complexity significantly smaller than the existing complexity  rates. We also introduced the MTEE problem where the optimal policy can be found using the dynamic programming. We proposed the \UCBVIEnt and \RFExploreEnt algorithms for MTEE that can be adapted to BPI in regularized MDPs. We proved that, in both cases, \RFExploreEnt and its variant enjoy a fast rate. In particular, we observed a statistical separation between BPI in regularized MDP and in the original MDP. Moreover, we show that the regularized version of \EntGame called \regalgMVEE enjoys $\tcO(H^2 SA/\varepsilon^2)$ rates, making dependence in $S$ smaller than in the reward-free exploration problem \cite{jin2020reward-free}.

This work opens the following interesting future research directions:

\vspace{-0.2cm}
\paragraph{Optimal rates for MVEE and MTEE} We are still lacking lower bounds for MTEE and MVEE problems enabling us to determine the optimal rates, especially what the number of states $S$ and the horizon $H$ is concerned. Note that one cannot apply directly the usual lower-bounds techniques for these two problems because of the  entropy  regularization  in both cases. In particular, we conjecture that the optimal rate for MVEE is also of order $\tcO(\text{poly}(H,S,A)/\epsilon)$.

\vspace{-0.1cm}
\paragraph{Optimal rate for entropy-regularized RL} It would be interesting to obtain the optimal rate for BPI in a regularized MDP. In particular to recover the optimal rate for BPI in the original MDP by tuning the regularization parameter $\lambda$. We conjecture that the optimal rate is $\tcO(H^2SA/(\lambda\epsilon))$ for BPI in entropy-regularized MDP. 

\vspace{-0.1cm}
\paragraph{Other types of entropies} Our methodology can be applied to other types of entropies and even to other regularization penalties. One interesting case would be the goal-conditioned trajectory entropy (see \citealt{savas2019entropy}) where one considers only process realizations that reach a certain set of states at time $H$.  This entropy can be applied to goal-conditioned RL. Another type of problem that could be of high interest is visitation entropy maximization under safety constraints \citep{yang2023cem}. 

\section*{Acknowledgements}
The work of D. Tiapkin, A. Naumov, and D. Belomestny were supported by the grant for research centers in the field of AI provided by the Analytical Center for the Government of the Russian Federation (ACRF) in accordance with the agreement on the provision of subsidies (identifier of the agreement 000000D730321P5Q0002) and the agreement with HSE University No. 70-2021-00139. D. Belomestny acknowledges the financial support from Deutsche Forschungsgemeinschaft (DFG), Grant Nr.497300407.
P. M\'enard acknowledges the Chaire SeqALO (ANR-20-CHIA-0020-01).

\bibliographystyle{sty/icml2023}
\bibliography{ref.bib}

\begin{thebibliography}{60}
\providecommand{\natexlab}[1]{#1}
\providecommand{\url}[1]{\texttt{#1}}
\expandafter\ifx\csname urlstyle\endcsname\relax
  \providecommand{\doi}[1]{doi: #1}\else
  \providecommand{\doi}{doi: \begingroup \urlstyle{rm}\Url}\fi

\bibitem[Abernethy \& Wang(2017)Abernethy and Wang]{abernethy2017frankwolfe}
Abernethy, J.~D. and Wang, J.-K.
\newblock On {Frank-Wolfe} and equilibrium computation.
\newblock In \emph{Neural Information Processing Systems}, 2017.
\newblock URL
  \url{https://proceedings.neurips.cc/paper/2017/file/7371364b3d72ac9a3ed8638e6f0be2c9-Paper.pdf}.

\bibitem[Antos \& Kontoyiannis(2001)Antos and
  Kontoyiannis]{antos2001convergence}
Antos, A. and Kontoyiannis, I.
\newblock Convergence properties of functional estimates for discrete
  distributions.
\newblock \emph{Random Structures \& Algorithms}, 19\penalty0 (3-4):\penalty0
  163--193, 2001.
\newblock \doi{https://doi.org/10.1002/rsa.10019}.
\newblock URL \url{https://onlinelibrary.wiley.com/doi/abs/10.1002/rsa.10019}.

\bibitem[Azar et~al.(2013)Azar, Munos, and Kappen]{azar2013minimax}
Azar, M.~G., Munos, R., and Kappen, H.~J.
\newblock {Minimax PAC bounds on the sample complexity of reinforcement
  learning with a generative model}.
\newblock \emph{Machine Learning}, 91\penalty0 (3):\penalty0 325--349, 2013.
\newblock URL \url{https://hal.archives-ouvertes.fr/hal-00831875}.

\bibitem[Azar et~al.(2017)Azar, Osband, and Munos]{azar2017minimax}
Azar, M.~G., Osband, I., and Munos, R.
\newblock {Minimax regret bounds for reinforcement learning}.
\newblock In \emph{International Conference on Machine Learning}, 2017.
\newblock URL \url{https://arxiv.org/pdf/1703.05449.pdf}.

\bibitem[Bauschke et~al.(2017)Bauschke, Bolte, and
  Teboulle]{bauschke2016descent}
Bauschke, H.~H., Bolte, J., and Teboulle, M.
\newblock A descent lemma beyond lipschitz gradient continuity: First-order
  methods revisited and applications.
\newblock \emph{Mathematics of Operations Research}, 42\penalty0 (2):\penalty0
  330--348, 2017.
\newblock \doi{10.1287/moor.2016.0817}.
\newblock URL \url{https://doi.org/10.1287/moor.2016.0817}.

\bibitem[Bellemare et~al.(2016)Bellemare, Srinivasan, Ostrovski, Schaul,
  Saxton, and Munos]{bellemare2016unifying}
Bellemare, M., Srinivasan, S., Ostrovski, G., Schaul, T., Saxton, D., and
  Munos, R.
\newblock Unifying count-based exploration and intrinsic motivation.
\newblock In Lee, D., Sugiyama, M., Luxburg, U., Guyon, I., and Garnett, R.
  (eds.), \emph{Advances in Neural Information Processing Systems}, volume~29.
  Curran Associates, Inc., 2016.
\newblock URL
  \url{https://proceedings.neurips.cc/paper/2016/file/afda332245e2af431fb7b672a68b659d-Paper.pdf}.

\bibitem[Boyd \& Vandenberghe(2004)Boyd and Vandenberghe]{boyd2004convex}
Boyd, S. and Vandenberghe, L.
\newblock \emph{Convex Optimization}.
\newblock Cambridge University Press, 2004.
\newblock \doi{10.1017/CBO9780511804441}.

\bibitem[Bubeck(2015)]{bubeck2015convex}
Bubeck, S.
\newblock Convex optimization: Algorithms and complexity.
\newblock \emph{Found. Trends Mach. Learn.}, 8\penalty0 (3–4):\penalty0
  231–357, nov 2015.
\newblock ISSN 1935-8237.
\newblock \doi{10.1561/2200000050}.
\newblock URL \url{https://doi.org/10.1561/2200000050}.

\bibitem[Burda et~al.(2019)Burda, Edwards, Storkey, and
  Klimov]{burda2019exploration}
Burda, Y., Edwards, H., Storkey, A.~J., and Klimov, O.
\newblock Exploration by random network distillation.
\newblock In \emph{7th International Conference on Learning Representations,
  {ICLR} 2019, New Orleans, LA, USA, May 6-9, 2019}. OpenReview.net, 2019.
\newblock URL \url{https://openreview.net/forum?id=H1lJJnR5Ym}.

\bibitem[Cesa-Bianchi \& Lugosi(2006)Cesa-Bianchi and
  Lugosi]{cesabianchi2006prediction}
Cesa-Bianchi, N. and Lugosi, G.
\newblock \emph{Prediction, learning, and games.}
\newblock Cambridge University Press, 2006.
\newblock ISBN 978-0-511-54692-1.

\bibitem[Cesa-Bianchi et~al.(2017)Cesa-Bianchi, Gentile, Lugosi, and
  Neu]{cesabianchi2017boltzmann}
Cesa-Bianchi, N., Gentile, C., Lugosi, G., and Neu, G.
\newblock Boltzmann exploration done right.
\newblock In \emph{Proceedings of the 31st International Conference on Neural
  Information Processing Systems}, NIPS'17, pp.\  6287–6296, Red Hook, NY,
  USA, 2017. Curran Associates Inc.
\newblock ISBN 9781510860964.

\bibitem[Chentanez et~al.(2004)Chentanez, Barto, and
  Singh]{chentanez2004intrinsically}
Chentanez, N., Barto, A., and Singh, S.
\newblock Intrinsically motivated reinforcement learning.
\newblock In Saul, L., Weiss, Y., and Bottou, L. (eds.), \emph{Advances in
  Neural Information Processing Systems}, volume~17. MIT Press, 2004.
\newblock URL
  \url{https://proceedings.neurips.cc/paper/2004/file/4be5a36cbaca8ab9d2066debfe4e65c1-Paper.pdf}.

\bibitem[Cheung(2019)]{cheung2019exploration}
Cheung, W.~C.
\newblock Exploration-exploitation trade-off in reinforcement learning on
  online markov decision processes with global concave rewards.
\newblock \emph{CoRR}, abs/1905.06466, 2019.
\newblock URL \url{http://arxiv.org/abs/1905.06466}.

\bibitem[Cover \& Thomas(2006)Cover and Thomas]{cover2006elements}
Cover, T.~M. and Thomas, J.~A.
\newblock \emph{{Elements of information theory}}.
\newblock John Wiley {\&} Sons, 2006.
\newblock URL
  \url{https://www.amazon.com/Elements-Information-Theory-Telecommunications-Processing/dp/0471241954}.

\bibitem[Dann et~al.(2017)Dann, Lattimore, and Brunskill]{dann2017unifying}
Dann, C., Lattimore, T., and Brunskill, E.
\newblock {Unifying {PAC} and regret: Uniform {PAC} bounds for episodic
  reinforcement learning}.
\newblock In \emph{Neural Information Processing Systems}, 2017.
\newblock URL \url{https://arxiv.org/pdf/1703.07710.pdf}.

\bibitem[Dann et~al.(2019)Dann, Li, Wei, and Brunskill]{dann2019policy}
Dann, C., Li, L., Wei, W., and Brunskill, E.
\newblock Policy certificates: Towards accountable reinforcement learning.
\newblock In \emph{International Conference on Machine Learning}, pp.\
  1507--1516. PMLR, 2019.

\bibitem[Domingues et~al.(2021{\natexlab{a}})Domingues, M{\'e}nard, Kaufmann,
  and Valko]{domingues2021episodic}
Domingues, O.~D., M{\'e}nard, P., Kaufmann, E., and Valko, M.
\newblock Episodic reinforcement learning in finite mdps: Minimax lower bounds
  revisited.
\newblock In \emph{Algorithmic Learning Theory}, pp.\  578--598. PMLR,
  2021{\natexlab{a}}.

\bibitem[Domingues et~al.(2021{\natexlab{b}})Domingues, Menard, Pirotta,
  Kaufmann, and Valko]{domingues2020regret}
Domingues, O.~D., Menard, P., Pirotta, M., Kaufmann, E., and Valko, M.
\newblock Kernel-based reinforcement learning: A finite-time analysis.
\newblock In Meila, M. and Zhang, T. (eds.), \emph{Proceedings of the 38th
  International Conference on Machine Learning}, volume 139 of
  \emph{Proceedings of Machine Learning Research}, pp.\  2783--2792. PMLR,
  18--24 Jul 2021{\natexlab{b}}.
\newblock URL \url{https://proceedings.mlr.press/v139/domingues21a.html}.

\bibitem[Ecoffet et~al.(2019)Ecoffet, Huizinga, Lehman, Stanley, and
  Clune]{escoffet2019goexplore}
Ecoffet, A., Huizinga, J., Lehman, J., Stanley, K.~O., and Clune, J.
\newblock Go-explore: a new approach for hard-exploration problems.
\newblock \emph{CoRR}, abs/1901.10995, 2019.
\newblock URL \url{http://arxiv.org/abs/1901.10995}.

\bibitem[Ekroot \& Cover(1993)Ekroot and Cover]{ekroot1993entropy}
Ekroot, L. and Cover, T.~M.
\newblock The entropy of markov trajectories.
\newblock \emph{IEEE Transactions on Information Theory}, 39\penalty0
  (4):\penalty0 1418--1421, 1993.

\bibitem[Eysenbach \& Levine(2019)Eysenbach and Levine]{eysenbach2019if}
Eysenbach, B. and Levine, S.
\newblock If maxent {RL} is the answer, what is the question?
\newblock \emph{CoRR}, abs/1910.01913, 2019.
\newblock URL \url{http://arxiv.org/abs/1910.01913}.

\bibitem[Fiechter(1994)]{fiechter1994efficient}
Fiechter, C.-N.
\newblock {Efficient reinforcement learning}.
\newblock In \emph{Conference on Learning Theory}, 1994.
\newblock URL
  \url{http://citeseerx.ist.psu.edu/viewdoc/download;jsessionid=7F5F8FCD1AA7ED07356410DDD5B384FE?doi=10.1.1.49.8652\&rep=rep1\&type=pdf}.

\bibitem[Fox et~al.(2016)Fox, Pakman, and Tishby]{fox2016taming}
Fox, R., Pakman, A., and Tishby, N.
\newblock Taming the noise in reinforcement learning via soft updates.
\newblock In Ihler, A. and Janzing, D. (eds.), \emph{Proceedings of the
  Thirty-Second Conference on Uncertainty in Artificial Intelligence, {UAI}
  2016, June 25-29, 2016, New York City, NY, {USA}}. {AUAI} Press, 2016.
\newblock URL \url{http://auai.org/uai2016/proceedings/papers/219.pdf}.

\bibitem[Frank \& Wolfe(1956)Frank and Wolfe]{frank1956algorithm}
Frank, M. and Wolfe, P.
\newblock An algorithm for quadratic programming.
\newblock \emph{Naval Research Logistics Quarterly}, 3\penalty0 (1-2):\penalty0
  95--110, 1956.
\newblock \doi{https://doi.org/10.1002/nav.3800030109}.
\newblock URL
  \url{https://onlinelibrary.wiley.com/doi/abs/10.1002/nav.3800030109}.

\bibitem[Geist et~al.(2019)Geist, Scherrer, and Pietquin]{geist2019theory}
Geist, M., Scherrer, B., and Pietquin, O.
\newblock A theory of regularized {M}arkov decision processes.
\newblock In Chaudhuri, K. and Salakhutdinov, R. (eds.), \emph{Proceedings of
  the 36th International Conference on Machine Learning}, volume~97 of
  \emph{Proceedings of Machine Learning Research}, pp.\  2160--2169. PMLR,
  09--15 Jun 2019.
\newblock URL \url{https://proceedings.mlr.press/v97/geist19a.html}.

\bibitem[Grill et~al.(2019)Grill, Darwiche~Domingues, Menard, Munos, and
  Valko]{grill2019planning}
Grill, J.-B., Darwiche~Domingues, O., Menard, P., Munos, R., and Valko, M.
\newblock Planning in entropy-regularized markov decision processes and games.
\newblock In Wallach, H., Larochelle, H., Beygelzimer, A., d\textquotesingle
  Alch\'{e}-Buc, F., Fox, E., and Garnett, R. (eds.), \emph{Advances in Neural
  Information Processing Systems}, volume~32. Curran Associates, Inc., 2019.
\newblock URL
  \url{https://proceedings.neurips.cc/paper/2019/file/50982fb2f2cfa186d335310461dfa2be-Paper.pdf}.

\bibitem[Gr{\"{u}}nwald \& Dawid(2002)Gr{\"{u}}nwald and
  Dawid]{grunwald2002game}
Gr{\"{u}}nwald, P.~D. and Dawid, A.~P.
\newblock Game theory, maximum generalized entropy, minimum discrepancy, robust
  bayes and pythagoras.
\newblock In \emph{Proceedings of the 2002 {IEEE} Information Theory Workshop,
  {ITW} 2002, 20-25 October 2002, Bangalore, India}, pp.\  94--97. {IEEE},
  2002.
\newblock \doi{10.1109/ITW.2002.1115425}.
\newblock URL \url{https://doi.org/10.1109/ITW.2002.1115425}.

\bibitem[Haarnoja et~al.(2017)Haarnoja, Tang, Abbeel, and
  Levine]{haarnoja2017reinforcement}
Haarnoja, T., Tang, H., Abbeel, P., and Levine, S.
\newblock Reinforcement learning with deep energy-based policies.
\newblock In \emph{Proceedings of the 34th International Conference on Machine
  Learning - Volume 70}, ICML'17, pp.\  1352–1361. JMLR.org, 2017.

\bibitem[Haarnoja et~al.(2018)Haarnoja, Zhou, Abbeel, and
  Levine]{haarnoja2018soft}
Haarnoja, T., Zhou, A., Abbeel, P., and Levine, S.
\newblock Soft actor-critic: Off-policy maximum entropy deep reinforcement
  learning with a stochastic actor.
\newblock In Dy, J. and Krause, A. (eds.), \emph{Proceedings of the 35th
  International Conference on Machine Learning}, volume~80 of \emph{Proceedings
  of Machine Learning Research}, pp.\  1861--1870. PMLR, 10--15 Jul 2018.
\newblock URL \url{https://proceedings.mlr.press/v80/haarnoja18b.html}.

\bibitem[Hazan et~al.(2019)Hazan, Kakade, Singh, and
  Van~Soest]{hazan2019provably}
Hazan, E., Kakade, S., Singh, K., and Van~Soest, A.
\newblock Provably efficient maximum entropy exploration.
\newblock In Chaudhuri, K. and Salakhutdinov, R. (eds.), \emph{Proceedings of
  the 36th International Conference on Machine Learning}, volume~97 of
  \emph{Proceedings of Machine Learning Research}, pp.\  2681--2691. PMLR,
  09--15 Jun 2019.
\newblock URL \url{https://proceedings.mlr.press/v97/hazan19a.html}.

\bibitem[Jin et~al.(2020)Jin, Krishnamurthy, Simchowitz, and
  Yu]{jin2020reward-free}
Jin, C., Krishnamurthy, A., Simchowitz, M., and Yu, T.
\newblock Reward-free exploration for reinforcement learning.
\newblock In III, H.~D. and Singh, A. (eds.), \emph{Proceedings of the 37th
  International Conference on Machine Learning}, volume 119 of
  \emph{Proceedings of Machine Learning Research}, pp.\  4870--4879. PMLR,
  13--18 Jul 2020.
\newblock URL \url{https://proceedings.mlr.press/v119/jin20d.html}.

\bibitem[Jonsson et~al.(2020)Jonsson, Kaufmann, M{\'e}nard, Darwiche~Domingues,
  Leurent, and Valko]{jonsson2020planning}
Jonsson, A., Kaufmann, E., M{\'e}nard, P., Darwiche~Domingues, O., Leurent, E.,
  and Valko, M.
\newblock Planning in markov decision processes with gap-dependent sample
  complexity.
\newblock \emph{Advances in Neural Information Processing Systems},
  33:\penalty0 1253--1263, 2020.

\bibitem[Kakade et~al.(2009)Kakade, Shalev-Shwartz, Tewari,
  et~al.]{kakade2009duality}
Kakade, S., Shalev-Shwartz, S., Tewari, A., et~al.
\newblock On the duality of strong convexity and strong smoothness: Learning
  applications and matrix regularization.
\newblock \emph{Unpublished Manuscript, http://ttic. uchicago.
  edu/shai/papers/KakadeShalevTewari09. pdf}, 2\penalty0 (1):\penalty0 35,
  2009.

\bibitem[Kaufmann et~al.(2021)Kaufmann, M{\'e}nard, Darwiche~Domingues,
  Jonsson, Leurent, and Valko]{kaufmann2020adaptive}
Kaufmann, E., M{\'e}nard, P., Darwiche~Domingues, O., Jonsson, A., Leurent, E.,
  and Valko, M.
\newblock Adaptive reward-free exploration.
\newblock In Feldman, V., Ligett, K., and Sabato, S. (eds.), \emph{Proceedings
  of the 32nd International Conference on Algorithmic Learning Theory}, volume
  132 of \emph{Proceedings of Machine Learning Research}, pp.\  865--891. PMLR,
  16--19 Mar 2021.
\newblock URL \url{https://proceedings.mlr.press/v132/kaufmann21a.html}.

\bibitem[Lee et~al.(2019)Lee, Eysenbach, Parisotto, Xing, Levine, and
  Salakhutdinov]{lee2019efficient}
Lee, L., Eysenbach, B., Parisotto, E., Xing, E.~P., Levine, S., and
  Salakhutdinov, R.
\newblock Efficient exploration via state marginal matching.
\newblock \emph{CoRR}, abs/1906.05274, 2019.
\newblock URL \url{http://arxiv.org/abs/1906.05274}.

\bibitem[Lim \& Auer(2012)Lim and Auer]{lim2012autonomous}
Lim, S.~H. and Auer, P.
\newblock Autonomous exploration for navigating in mdps.
\newblock In \emph{Conference on Learning Theory}, pp.\  40--1, 2012.

\bibitem[M\'enard et~al.(2021)M\'enard, Domingues, Jonsson, Kaufmann, Leurent,
  and Valko]{menard2021fast}
M\'enard, P., Domingues, O.~D., Jonsson, A., Kaufmann, E., Leurent, E., and
  Valko, M.
\newblock Fast active learning for pure exploration in reinforcement learning.
\newblock In Meila, M. and Zhang, T. (eds.), \emph{Proceedings of the 38th
  International Conference on Machine Learning}, volume 139 of
  \emph{Proceedings of Machine Learning Research}, pp.\  7599--7608. PMLR,
  18--24 Jul 2021.
\newblock URL \url{https://proceedings.mlr.press/v139/menard21a.html}.

\bibitem[Mutti \& Restelli(2020)Mutti and Restelli]{mutti2020intristically}
Mutti, M. and Restelli, M.
\newblock An intrinsically-motivated approach for learning highly exploring and
  fast mixing policies.
\newblock \emph{Proceedings of the AAAI Conference on Artificial Intelligence},
  34\penalty0 (04):\penalty0 5232--5239, Apr. 2020.
\newblock \doi{10.1609/aaai.v34i04.5968}.
\newblock URL \url{https://ojs.aaai.org/index.php/AAAI/article/view/5968}.

\bibitem[Mutti et~al.(2021)Mutti, Pratissoli, and Restelli]{mutti2021task}
Mutti, M., Pratissoli, L., and Restelli, M.
\newblock Task-agnostic exploration via policy gradient of a non-parametric
  state entropy estimate.
\newblock In \emph{Thirty-Fifth {AAAI} Conference on Artificial Intelligence,
  {AAAI} 2021, Thirty-Third Conference on Innovative Applications of Artificial
  Intelligence, {IAAI} 2021, The Eleventh Symposium on Educational Advances in
  Artificial Intelligence, {EAAI} 2021, Virtual Event, February 2-9, 2021},
  pp.\  9028--9036. {AAAI} Press, 2021.
\newblock URL \url{https://ojs.aaai.org/index.php/AAAI/article/view/17091}.

\bibitem[Mutti et~al.(2022)Mutti, Mancassola, and
  Restelli]{mutti2022unsupervised}
Mutti, M., Mancassola, M., and Restelli, M.
\newblock Unsupervised reinforcement learning in multiple environments.
\newblock \emph{Proceedings of the AAAI Conference on Artificial Intelligence},
  36\penalty0 (7):\penalty0 7850--7858, Jun. 2022.
\newblock \doi{10.1609/aaai.v36i7.20754}.
\newblock URL \url{https://ojs.aaai.org/index.php/AAAI/article/view/20754}.

\bibitem[Neu et~al.(2017)Neu, Jonsson, and G{\'{o}}mez]{neu2017unified}
Neu, G., Jonsson, A., and G{\'{o}}mez, V.
\newblock A unified view of entropy-regularized markov decision processes.
\newblock \emph{CoRR}, abs/1705.07798, 2017.
\newblock URL \url{http://arxiv.org/abs/1705.07798}.

\bibitem[Oudeyer et~al.(2007)Oudeyer, Kaplan, and Hafner]{oudeyer2007intrisic}
Oudeyer, P., Kaplan, F., and Hafner, V.~V.
\newblock Intrinsic motivation systems for autonomous mental development.
\newblock \emph{{IEEE} Trans. Evol. Comput.}, 11\penalty0 (2):\penalty0
  265--286, 2007.
\newblock \doi{10.1109/TEVC.2006.890271}.
\newblock URL \url{https://doi.org/10.1109/TEVC.2006.890271}.

\bibitem[Paninski(2003)]{paninski2003estimation}
Paninski, L.
\newblock {Estimation of Entropy and Mutual Information}.
\newblock \emph{Neural Computation}, 15\penalty0 (6):\penalty0 1191--1253, 06
  2003.
\newblock ISSN 0899-7667.
\newblock \doi{10.1162/089976603321780272}.
\newblock URL \url{https://doi.org/10.1162/089976603321780272}.

\bibitem[Pathak et~al.(2017)Pathak, Agrawal, Efros, and
  Darrell]{pathak2017curiosity}
Pathak, D., Agrawal, P., Efros, A.~A., and Darrell, T.
\newblock Curiosity-driven exploration by self-supervised prediction.
\newblock In Precup, D. and Teh, Y.~W. (eds.), \emph{Proceedings of the 34th
  International Conference on Machine Learning}, volume~70 of \emph{Proceedings
  of Machine Learning Research}, pp.\  2778--2787. PMLR, 06--11 Aug 2017.
\newblock URL \url{https://proceedings.mlr.press/v70/pathak17a.html}.

\bibitem[Savas et~al.(2019)Savas, Ornik, Cubuktepe, Karabag, and
  Topcu]{savas2019entropy}
Savas, Y., Ornik, M., Cubuktepe, M., Karabag, M.~O., and Topcu, U.
\newblock Entropy maximization for markov decision processes under temporal
  logic constraints.
\newblock \emph{IEEE Transactions on Automatic Control}, 65\penalty0
  (4):\penalty0 1552--1567, 2019.

\bibitem[Schmidhuber(1991)]{schmidhuber1991possibility}
Schmidhuber, J.
\newblock A possibility for implementing curiosity and boredom in
  model-building neural controllers.
\newblock In \emph{Proc. of the international conference on simulation of
  adaptive behavior: From animals to animats}, pp.\  222--227, 1991.

\bibitem[Schulman et~al.(2017)Schulman, Abbeel, and
  Chen]{schulman2017equivalence}
Schulman, J., Abbeel, P., and Chen, X.
\newblock Equivalence between policy gradients and soft q-learning.
\newblock \emph{CoRR}, abs/1704.06440, 2017.
\newblock URL \url{http://arxiv.org/abs/1704.06440}.

\bibitem[Seo et~al.(2021)Seo, Chen, Shin, Lee, Abbeel, and Lee]{seo21state}
Seo, Y., Chen, L., Shin, J., Lee, H., Abbeel, P., and Lee, K.
\newblock State entropy maximization with random encoders for efficient
  exploration.
\newblock In Meila, M. and Zhang, T. (eds.), \emph{Proceedings of the 38th
  International Conference on Machine Learning}, volume 139 of
  \emph{Proceedings of Machine Learning Research}, pp.\  9443--9454. PMLR,
  18--24 Jul 2021.
\newblock URL \url{https://proceedings.mlr.press/v139/seo21a.html}.

\bibitem[Sion(1958)]{sion1958general}
Sion, M.
\newblock {On general minimax theorems.}
\newblock \emph{Pacific Journal of Mathematics}, 8\penalty0 (1):\penalty0 171
  -- 176, 1958.
\newblock \doi{pjm/1103040253}.
\newblock URL \url{https://doi.org/}.

\bibitem[Sutton(1990)]{sutton1990integrated}
Sutton, R.~S.
\newblock Integrated architectures for learning, planning, and reacting based
  on approximating dynamic programming.
\newblock In Porter, B. and Mooney, R. (eds.), \emph{Machine Learning
  Proceedings 1990}, pp.\  216--224. Morgan Kaufmann, San Francisco (CA), 1990.
\newblock ISBN 978-1-55860-141-3.
\newblock \doi{https://doi.org/10.1016/B978-1-55860-141-3.50030-4}.
\newblock URL
  \url{https://www.sciencedirect.com/science/article/pii/B9781558601413500304}.

\bibitem[Talebi \& Maillard(2018)Talebi and Maillard]{talebi2018variance}
Talebi, M.~S. and Maillard, O.-A.
\newblock Variance-aware regret bounds for undiscounted reinforcement learning
  in mdps.
\newblock In \emph{Algorithmic Learning Theory}, pp.\  770--805, 2018.

\bibitem[Tarbouriech et~al.(2020{\natexlab{a}})Tarbouriech, Pirotta, Valko, and
  Lazaric]{tarbouriech2020improved}
Tarbouriech, J., Pirotta, M., Valko, M., and Lazaric, A.
\newblock Improved sample complexity for incremental autonomous exploration in
  mdps.
\newblock In Larochelle, H., Ranzato, M., Hadsell, R., Balcan, M., and Lin, H.
  (eds.), \emph{Advances in Neural Information Processing Systems}, volume~33,
  pp.\  11273--11284. Curran Associates, Inc., 2020{\natexlab{a}}.
\newblock URL
  \url{https://proceedings.neurips.cc/paper/2020/file/81e793dc8317a3dbc3534ed3f242c418-Paper.pdf}.

\bibitem[Tarbouriech et~al.(2020{\natexlab{b}})Tarbouriech, Shekhar, Pirotta,
  Ghavamzadeh, and Lazaric]{tarbouriech2020active}
Tarbouriech, J., Shekhar, S., Pirotta, M., Ghavamzadeh, M., and Lazaric, A.
\newblock Active model estimation in markov decision processes.
\newblock In Peters, J. and Sontag, D. (eds.), \emph{Proceedings of the 36th
  Conference on Uncertainty in Artificial Intelligence (UAI)}, volume 124 of
  \emph{Proceedings of Machine Learning Research}, pp.\  1019--1028. PMLR,
  03--06 Aug 2020{\natexlab{b}}.
\newblock URL \url{https://proceedings.mlr.press/v124/tarbouriech20a.html}.

\bibitem[Tarbouriech et~al.(2021)Tarbouriech, Pirotta, Valko, and
  Lazaric]{tarbouriech2021provably}
Tarbouriech, J., Pirotta, M., Valko, M., and Lazaric, A.
\newblock A provably efficient sample collection strategy for reinforcement
  learning.
\newblock In Ranzato, M., Beygelzimer, A., Dauphin, Y., Liang, P., and Vaughan,
  J.~W. (eds.), \emph{Advances in Neural Information Processing Systems},
  volume~34, pp.\  7611--7624. Curran Associates, Inc., 2021.
\newblock URL
  \url{https://proceedings.neurips.cc/paper/2021/file/3e98410c45ea98addec555019bbae8eb-Paper.pdf}.

\bibitem[Tirinzoni et~al.(2023)Tirinzoni, Al-Marjani, and
  Kaufmann]{tirinzoni2022optimistic}
Tirinzoni, A., Al-Marjani, A., and Kaufmann, E.
\newblock Optimistic pac reinforcement learning: the instance-dependent view.
\newblock In \emph{International Conference on Algorithmic Learning Theory},
  pp.\  1460--1480. PMLR, 2023.

\bibitem[van Handel(2014)]{van2014probability}
van Handel, R.
\newblock Probability in high dimensions. manuscript, 2014, 2014.

\bibitem[Yang \& Spaan(2023)Yang and Spaan]{yang2023cem}
Yang, Q. and Spaan, M.~T.
\newblock Cem: Constrained entropy maximization for task-agnostic safe
  exploration.
\newblock In \emph{The Thirty-Seventh AAAI Conference on Artificial
  Intelligence}, 2023.

\bibitem[Zahavy et~al.(2021)Zahavy, O'Donoghue, Desjardins, and
  Singh]{zahavy2021reward}
Zahavy, T., O'Donoghue, B., Desjardins, G., and Singh, S.
\newblock Reward is enough for convex {MDP}s.
\newblock In Beygelzimer, A., Dauphin, Y., Liang, P., and Vaughan, J.~W.
  (eds.), \emph{Advances in Neural Information Processing Systems}, 2021.
\newblock URL \url{https://openreview.net/forum?id=ELndVeVA-TR}.

\bibitem[Zanette \& Brunskill(2019)Zanette and Brunskill]{zanette2019tighter}
Zanette, A. and Brunskill, E.
\newblock Tighter problem-dependent regret bounds in reinforcement learning
  without domain knowledge using value function bounds.
\newblock In Chaudhuri, K. and Salakhutdinov, R. (eds.), \emph{Proceedings of
  the 36th International Conference on Machine Learning}, volume~97 of
  \emph{Proceedings of Machine Learning Research}, pp.\  7304--7312. PMLR,
  09--15 Jun 2019.
\newblock URL \url{https://proceedings.mlr.press/v97/zanette19a.html}.

\bibitem[Zhang et~al.(2021)Zhang, Cai, Huang, and Li]{zhang2021exploration}
Zhang, C., Cai, Y., Huang, L., and Li, J.
\newblock Exploration by maximizing renyi entropy for reward-free rl framework.
\newblock \emph{Proceedings of the AAAI Conference on Artificial Intelligence},
  35\penalty0 (12):\penalty0 10859--10867, May 2021.
\newblock \doi{10.1609/aaai.v35i12.17297}.
\newblock URL \url{https://ojs.aaai.org/index.php/AAAI/article/view/17297}.

\end{thebibliography}

\newpage
\appendix
\onecolumn

\part{Appendix}

\vspace{-0.15cm}
\parttoc
\newpage
\section{Notation}\label{app:notations}

\begin{table}[h!]
	\centering
	\caption{Table of notation use throughout the paper}
	\begin{tabular}{@{}l|l@{}}
		\toprule
		{Notation} & \thead{Meaning} \\ \midrule
	$\cS$ & state space of size $S$\\
	$\cA$ & action space of size $A$\\
	$H$ & length of one episode\\
    $s_1$ & initial state \\ 
    $\tau$ & stopping time \\
    $\cT$ & trajectory space, $\cT \triangleq (\cS \times \cA)^H$ \\
    $\varepsilon$ & desired accuracy of solving the problem \\
    $\delta$ & desired upper bound on failure probability \\
    $\lambda$ & regularization parameter in regularized MDPs (see Appendix~\ref{app:regularized_mdp}). \\
    $\kappa$ & weight of transition entropy in reward in regualrized MDPs (see Appendix~\ref{app:regularized_mdp})\\
    \hline 
	$p_h(s'|s,a)$ & probability transition \\
    $r_h(s,a)$ & reward function \\
    $d^{\pi}_h(s,a)$ & state-action visitation distribution at step $h$ for the policy $\pi$ \\
    $q^\pi(m)$ & visitation probability of trajectory $m \in \cT$ by policy $\pi$ \\
    $\cK_p$ & polytope of all admissible state-action visitation distributions \\
    $\cK$ &  polytope of all admissible distributions over state-actions, $\cK \triangleq (\simplex_{SA})^H$ \\
	$\VE(d^\pi)$ & visitation entropy $\VE(d^\pi) \triangleq \sum_{h=1}^H \cH(d^\pi_h)$ for $d^\pi \in \cK_p$, \\
    $\pistarVE$ & policy that maximizes $\VE(d^\pi)$, a solution to the MVEE problem \\
    $\TE(q^\pi)$ & trajectory entropy $\TE(q^\pi) \triangleq \cH(q^\pi)$ \\
    $\pistarTE$ & policy that maximizes $\TE(q^\pi)$, a solution to the MTEE problem \\
	\hline
	$n_0$ & number of prior visits for the forecaster-player in \EntGame  \\
    $t_0$ & total number of prior visits \\ 
	$s^{\,t}_h$ & state that was visited at $h$ step during $t$ episode \\
	$a^{\,t}_h$ & action that was picked at $h$ step during $t$ episode \\
	$n_h^t(s,a)$ & number of visits of state-action $n_h^t(s,a) = \sum_{k = 1}^t  \ind{\left\{(s_h^k,a_h^k) = (s,a)\right\}}$\\
	$n_h^t(s'|s,a)$ & number of transition to $s'$ from state-action {\small $n_h^t(s'|s,a) = \sum_{k = 1}^t  \ind{\left\{(s_h^k,a_h^k, s_{h+1}^k) = (s,a,s')\right\}}$}. \\
	$\upn_h^t(s,a)$ & pseudo number of visits of state-action $\upn_h^t(s,a)=n_h^t(s,a)+n_0$\\
	$\hp_h^{\,t}(s'|s,a)$ & empirical probability transition $\hp_h^{\,t}(s'|s,a) = n_h^t(s'|s,a) / n_h^t(s,a)$ \\
	\hline
    $\bd^t_h(s,a)$ & {\small predicted distribution by the forecaster-player in \EntGame $\bd^t_h(s,a) \triangleq \upn^{t-1}_h(s,a) / (t + t_0)$} \\
    $\uQ_h^t(s,a)$,$\uV_h^t(s,a)$ & {\small for \EntGame: upper bound on the optimal Q/V-functions in a MDP with rewards $\log(1/\bd^{t+1}_h(s,a)))$}\\
    \hline 
    $Q^{\pi}_h(s,a)$, $V^{\pi}_h(s,a)$ & Q- and V-functions for the MTEE problem \\
    $\Qstar_h(s,a)$, $\Vstar_h(s,a)$ & optimal Q- and V-function for the MTEE problem \\
    $\uQ^t_h(s,a)$, $\uV^t_h(s,a)$ & for \UCBVIEnt: the upper bound on the optimal Q/V-functions for the MTEE problem \\
    $\lQ^t_h(s,a)$, $\lV^t_h(s,a)$ & for \UCBVIEnt: the lower bound on the optimal Q/V-functions for the MTEE problem \\
    $Q^{\pi}_{\lambda,h}(s,a)$, $V^{\pi}_{\lambda,h}(s,a)$ & Q- and V-functions in a regularized MDP \\
    $\Qstar_{\lambda,h}(s,a)$, $\Vstar_{\lambda,h}(s,a)$ & optimal Q- and V-function in a regularized MDP \\
    $\hQ^\pi_{\lambda, h}(s,a), \hV^\pi_{\lambda,h}(s,a)$ & for \RFExploreEnt: the empirical Q- and V-functions in a regularized MDP \\
    \hline
    $\R_+$ & non-negative real numbers  \\
    $\N_+$ & positive natural numbers \\
    $[n]$ & set $\{1,2,\ldots, n\}$\\
    $\rme$ & Euler's number \\
    $\simplex_d$ & $d+1$-dimensional probability simplex: $\simplex_d \triangleq \{x \in \R_{+}^{d}: \sum_{j=1}^{d} x_j = 1 \}$ \\ 
    $\simplex_{\cX}$ & set of distributions over a finite set $\cX$ : $\simplex_\cX = \simplex_{\vert \cX \vert}$. \\
    $\cH(p)$ & Shannon entropy for $p \in \simplex_{\cX}$, $\cH(p) \triangleq \sum_{i \in \cX} p_i \log(1/p_i)$ \\
    $\clip(x,m,M)$ & clipping procedure $\clip(x,m,M) \triangleq \max(\min(x,M), m)$ \\
    \bottomrule
	\end{tabular}
\end{table}

Let $(\Xset,\Xsigma)$ be a measurable space and $\Pens(\Xset)$ be the set of all probability measures on this space. For $p \in \Pens(\Xset)$ we denote by $\E_p$ the expectation w.r.t. $p$. For random variable $\xi: \Xset \to \R$ notation $\xi \sim p$ means $\operatorname{Law}(\xi) = p$. For any measures $p,q \in \Pens(\Xset)$ we denote their product measure by $p \otimes q$. We also write $\E_{\xi \sim p}$ instead of $\E_{p}$.  For any $p, q \in \Pens(\Xset)$ the Kullback-Leibler divergence $\KL(p, q)$ is given by
$$
\KL(p, q) = \begin{cases}
\E_{p}[\log \frac{\rmd p}{\rmd q}], & p \ll q \\
+ \infty, & \text{otherwise}
\end{cases} 
$$
For any $p \in \Pens(\Xset)$ and $f: \Xset \to \R$, $p f = \E_p[f]$. In particular, for any $p \in \simplex_d$ and $f: \{0, \ldots, d\}   \to  \R$, $pf =  \sum_{\ell = 0}^d f(\ell) p(\ell)$. Define $\Var_{p}(f) = \E_{s' \sim p} \big[(f(s')-p f)^2\big] = p[f^2] - (pf)^2$. For any $(s,a) \in \cS$, transition kernel $p(s,a) \in \Pens(\cS)$ and $f \colon \cS \to \R$ define $pf(s,a) = \E_{p(s,a)}[f]$ and $\Var_{p}[f](s,a) = \Var_{p(s,a)}[f]$. For any $s\in \cS$, policy $\pi(s) \in \Pens(\cS)$ and $f \colon \cS \times \cA \to \R$ define $\pi f(s) = \E_{a \sim \pi(s)}[f(s,a)]$ and $\Var_{\pi} f(s) = \Var_{a \sim \pi(s)}[f(s,a)]$.

For a MDP $\cM$,a policy $\pi$ and a sequence of function $f_{h}$ define \(\E_{\pi}[ \sum_{h'=h}^H f(s_{h'}, a_{h'}) | s_h] \) as a conditional expectation of $\sum_{h'=h}^H f(s_{h'}, a_{h'})$ with respect to the sigma-algebra $\cF_h = \sigma\{ (s_{h'}, a_{h'}) | h' \leq h \}$, where for any $h\in[H]$ we have $a_h \sim \pi(s_h), s_{h+1} \sim p_h(s_h, a_h)$.

We write $f(S,A,H,\varepsilon) = \cO(g(S,A,H,\varepsilon,\delta))$ if there exist $ S_0, A_0, H_0, \varepsilon_0, \delta_0$ and constant $C_{f,g}$ such that for any $S \geq S_0, A \geq A_0, H \geq H_0, \varepsilon < \varepsilon_0, \delta < \delta_0, f(S,A,H,T,\delta) \leq C_{f,g} \cdot g(S,A,H,T,\delta)$. We write $f(S,A,H,\varepsilon,\delta) = \tcO(g(S,A,H,\varepsilon,\delta))$ if $C_{f,g}$ in the previous definition is poly-logarithmic in $S,A,H,1/\varepsilon,1/\delta$.

%For $\lambda > 0$ we define $\Exponential(\lambda)$ as an exponential distribution with a parameter $\lambda$. For $k, \theta > 0$ define $\Gamma(k,\theta)$ as a gamma-distribution with a shape parameter $k$ and a rate parameter $\theta$. For set $\Xset$ such that $\vert \Xset \vert < \infty$ define $\Unif(\Xset)$ as a uniform distribution over this set. In particular, $\Unif[N]$ is a uniform distribution over a set $[N]$.

% For independent (resp. i.i.d.) random variables $\xi_\ell \mysim p_\ell$ (resp. $\xi_\ell \mysimiid p$), $\ell = 1, \ldots, d$, we will write $\E_{\xi_\ell \mysim p_\ell}$ (resp. $\E_{\xi_\ell \mysimiid p}$), to denote expectation w.r.t. product measure on $(\Xset^d, \Xsigma^{\otimes d})$.
\newpage
\section{Proofs for Visitation Entropy}
\label{app:visitation_entropy_proofs}

We first define the regrets of each players obtained by playing $T$ times the games. For the forecaster-player, for any $\bd\in\cK$ we define 
\[
\regret_{\fore}^T(\bd) \triangleq \sum_{t=1}^T \sum_{h,s,a} \td_h^t(s,a) \left(\log\frac{1}{\bd_h^t(s,a)} -\log\frac{1}{\bd_h(s,a)}\right)
\]
where $\td_h^t(s,a) \triangleq \ind\big\{(s_h^t,a_h^t)=(s,a)\big\}$ is a sample from $d_h^{\pi^t}(s,a)$.
Similarly for the sampler-player, for any $d\in\cK_p$ we define 
\[
\regret_{\samp}^T(d)\triangleq \sum_{t=1}^T \sum_{h,s,a} \big(d_h(s,a) - d_h^{\pi^t}(s,a) \big) \log\frac{1}{\bd_h^t(s,a)}\,.
\]

Recall that the visitation distribution of the policy $\pi$ returned by \algMVEE is the average of the visitation distributions of the sampler-player 
$d_h^{\hpi}(s,a) = \hd^{\,T}_h(s,a) \triangleq (1/T) \sum_{t=1}^T d_h^{\pi^t}(s,a)$.  We also denote by $\rd^T_h(s,a)\triangleq (1/T) \sum_{t=1}^T \td^t(s,a)$ the average of the 'sample' visitation distributions.

We now relate the difference between the optimal visitation entropy and the visitation entropy of the outputted policy $\hpi$ with the regrets of the two players. Indeed, using $\cH(p) = \sum_{i\in[n]} p_i \log(1/q_i) -\KL(p,q) \leq  \sum_{i\in[n]} p_i \log(1/q_i)$ for all $(p,q)\in(\Delta_n)^2$, we obtain 
\begin{align*}
T\big(\VE(d^{\pistar}) -\VE(d^{\hpi})\big) &\leq \sum_{t=1}^T \sum_{h,a,s} d_h^{\pistar}(s,a) \log\frac{1}{\bd_h^t(s,a)}  - \td^t_h(s,a) \log\frac{1}{\rd_h^{\,T}(s,a)} + T\big(\VE(\rd^T) - \VE(\hd^T)\big)\\
& = \regret_{\samp}^T(d^{\pistar})+ \underbrace{\sum_{t=1}^T \sum_{h,s,a} \big(d_h^{\pi^t}(s,a) - \td_h^t(s,a) \big) \log\frac{1}{\bd_h^t(s,a)}}_{\mathrm{Bias}_1} + \regret_{\fore}^T(\rd^T) \\
&\quad+ \underbrace{T\big(\VE(\rd^T) - \VE(\hd^T)\big)}_{\mathrm{Bias}_2}\,.
\end{align*}
It remains to upper bound each terms separately in order to obtain a bound on the gap. We first bound the two regrets terms. The first bias term is  martingale and can easily be bounded with a deviation inequality, whereas for the second one we introduce just instrumentally smoothing of the entropy.

\subsection{Regret of the Forecaster-Player}

We prove in this section a regret-bound for the mixture forecaster.
\begin{lemma}
\label{lem:regret_forecaster}
For $n_0=1$, for any $\bd\in\cK$ it holds almost surely 
\[
\regret_{\fore}^T(\bd) \leq  HSA\log\big(\rme(T+1)\big) - T\sum_{h=1}^H\KL(\rd_h^T,\bd_h)\,.
\]
\end{lemma} 
\begin{proof}
We will bound the regret at step $h$,
\[
\regret^T_{\fore,h}(\bd) \triangleq \sum_{t=1}^T \sum_{s,a} \td_h^t(s,a) \left(\log\frac{1}{\bd_h^t(s,a)} -\log\frac{1}{\bd_h(s,a)}\right)
\]
and then sum the upper bounds. Recall 
\[
    \bd^t_h(s,a) = \frac{n^{t-1}_h(s,a) + 1}{t-1 + SA},
\]
and for $(s,a) = (s^t_h, a^t_h)$ and any $t\in[T], h \in [H]$ we have $n^{t-1}_h(s,a) + 1 = n^t_h(s,a)$. Since $n_0 = 1$, we have $\bd^t_h(s^t_h,a^t_h) = n^t_h(s^t_h, a^t_h) / (SA + t-1)$. Armed with this observation we can rewrite the regret as follows 
\begin{align*}
    \regret^T_{\fore,h}(\bd) &= -T\KL(\rd_h^T,\bd_h) - T\cH(\rd_h^T) -\sum_{t=1}^T \log\big( \bd_h^t(s_h^t,a_h^t)\big)\\
    &= -T\KL(\rd_h^T,\bd_h) - T\cH(\rd_h^T) -\log\left( \prod_{t=1}^T \bd_h^t(s_h^t,a_h^t)\right).
\end{align*}
Then we have an explicit formula for the product of $\bd^t_h$
\begin{align*}
    \prod_{t=1}^T \bd^t_h(s^t_h, a^t_h) &= \prod_{t=1}^T \frac{n^t_h(s^t_h, a^t_h)}{SA + t-1} = \frac{(SA-1)!}{(SA+T-1)!} \prod_{(s,a)\in \cS \times \cA}  [n^T_h(s, a)]! \\
    &= \frac{1}{\binom{T}{(n_h^T(s,a))_{(s,a)\in\cS\times\cA}}}\frac{1}{\binom{T+SA-1}{SA-1}}\\
    &\geq \exp\left(-T\cH(\rd_h^T) - (T+SA-1)\cH\left(\frac{SA-1}{T+SA-1}\right)\right)
\end{align*}
where in the last inequality we used Theorem 11.1.3 by \citet{cover2006elements} and overload the entropy notation $\cH(p) =- p\log(p) - (1-p)\log(1-p)$ for $p\in[0,1]$.
 Putting all together we get
\[
\regret^T_{\fore,h}(\bd) \leq (T+SA-1)\cH\left(\frac{A-1}{T+A-1}\right)-T\KL(\rd_h^T,\bd_h)\,.
\]
Bounding the entropic term
\begin{align*}
(T+SA-1)\cH\left(\frac{SA-1}{T+SA-1}\right) &= (SA-1)\log\frac{T+SA-1}{SA-1}+T\log\frac{T+SA-1}{T}\\
&\leq(SA-1)\log\frac{T+SA-1}{SA-1}+T\log\left(1+\frac{SA-1}{T}\right)\\
&\leq (SA-1)\log\frac{\rme(T+SA-1)}{SA-1}\\
&\leq SA\log\big(\rme(T+1)\big)\,,
\end{align*}
and summing over $h$ allows us to conclude.
\end{proof}

\subsection{Regret of the Sampler-Player}

We start from introducing new notation. Let $\cM_t = (\cS, \cA, \{ p_h \}_{h\in[H]}, \{r^t_h\}_{h\in[H]}, s_1)$ be a sequence of MDPs where reward defined as follows $r^t_h(s,a) = \log(1/ \bd^t_h(s,a))$. Define $Q^{\pi, t}_h(s,a)$ and $V^{\pi, t}_h(s,a)$ as a action-value and value functions of a policy $\pi$ on a MDP $\cM_t$. Notice that the value-function of initial state in this case could be written as follows
\[
    V^{\pi,t}_1(s_1) = \sum_{h,s,a} d^{\pi}_h(s,a) \log\left( \frac{1}{\bd^t_h(s,a)} \right)
\]
therefore, the regret for the sampler-player could be rewritten in the terms of the regret for this sequence of MDPs
\[
    \regret_{\samp}^T(d^\pi) =  \sum_{t=1}^T V^{\pi,t}_1(s_1) - V^{\pi^t,t}_1(s_1).
\]
Since the rewards are changing in each episode and depending on the full history on interaction during previous episodes, we have to handle more uniform approach as in usual \UCBVI proofs \cite{azar2017minimax}.

\paragraph{Concentration}

Let $\alpha^{\KL}, \alpha^{\cnt}: (0,1) \times \R_{+} \to \R_{+}$ be some functions defined later on in Lemma \ref{lem:sampler_proba_master_event}. We define the following favorable events
\begin{align*}
\cE^{\KL}(\delta) &\triangleq \Bigg\{ \forall t \in \N, \forall h \in [H], \forall (s,a) \in \cS\times\cA: \quad \KL(\hp^{\,t}_h(s,a), p_h(s,a)) \leq \frac{\alpha^{\KL}(\delta, n^{\,t}_h(s,a))}{n^{\,t}_h(s,a)} \Bigg\},\\
\cE^{\cnt}(\delta) &\triangleq \Bigg\{ \forall t \in \N, \forall h \in [H], \forall (s,a) \in \cS\times\cA: \quad n^t_h(s,a) \geq \frac{1}{2} \upn^t_h(s,a) - \alpha^{\cnt}(\delta) \Bigg\},
\end{align*}
\begin{lemma}\label{lem:sampler_proba_master_event}
For any $\delta \in (0,1)$ and for the following choices of functions $\alpha,$
\begin{align*}
    \alpha^{\KL}(\delta, n)  \triangleq \log(2SAH/\delta) + S\log\left(\rme(1+n) \right),  \quad
    \alpha^{\cnt}(\delta) \triangleq \log(2SAH/\delta), 
\end{align*}
it holds that
\begin{align*}
\P[\cE^{\KL}(\delta)] \geq 1-\delta/2, \qquad \P[\cE^\cnt(\delta)]\geq 1-\delta/2
\end{align*}
In particular, $\P[\cG(\delta)] \geq 1-\delta$.
\end{lemma}
\begin{proof}
Applying Theorem~\ref{th:max_ineq_categorical} and the union bound over $h \in [H], (s,a) \in \cS \times \cA$ we get $\P[\cE^{\KL}(\delta)]\geq 1-\delta/2$.  By Theorem~\ref{th:bernoulli-deviation} and union bound,  $\P[\cE^{\cnt}(\delta)]\geq 1 - \delta/2$. The union bound yields $\P[\cG(\delta)] \geq 1- \delta$.
\end{proof}

\paragraph{Optimism}
Next we define the exploration bonuses $b^t_h(s,a)$ for the sampler-player for $n_0 = 1$
\begin{equation}\label{eq:sampler_exploration_bonus}
    b^t_h(s,a) = \sqrt{\frac{2 H^2 \log^2(t+SA) \cdot \alpha^{\KL}(\delta, n^t_h(s,a))}{n^t_h(s,a)}}
\end{equation}

\begin{lemma}\label{lem:sampler_optimism}
    For any $t \in [T]$ and any policy $\pi$, the following holds on event $\cG(\delta)$
    \[
        \uQ^t_h(s,a) \geq Q^{\pi, t+1}_h(s,a), \qquad \uV^t_h(s) \geq V^{\pi, t+1}_h(s).
    \]
\end{lemma}
\begin{proof}
    Proceed by backward induction over $h$. For $h = H+1$ the statement trivially holds. Next we assume that the statement holds for any $h' > h$. Then we have by induction hypothesis and Hölder's inequality
    \begin{align*}
        \uQ^t_h(s,a) - Q^{\pi,t+1}_h(s,a) &= \hp^t_h \uV^t_{h+1}(s,a) - p_h V^{\pi,t+1}_h(s,a) + b^t_h(s,a)\\
        & \geq [\hp^t_h - p_h] V^{\pi,t+1}_h(s,a) + b^t_h(s,a) \geq - \norm{V^{\pi,t+1}_h}_\infty \norm{\hp^t_h - p_h}_1 + b^t_h(s,a).
    \end{align*}
    The fact that $\norm{V^{\pi,t+1}_h}_\infty \leq H \log(t+SA)$, Pinsker's inequality and the definition of the event $\cE^{\KL}(\delta)$ yields
    \[
        \norm{V^{\pi,t+1}_h}_\infty \norm{\hp^t_h - p_h}_1 \leq H\log(t+SA) \sqrt{\frac{2\alpha^{\KL}(\delta,n^t_h(s,a))}{n^t_h(s,a)}}  = b^t_h(s,a)
    \]
    that shows $\uQ^t_h(s,a) - Q^{\pi,t+1}_h(s,a) \geq 0$. The inequality on $V$-functions could be derived as follows
    \[
        \uV^t_h(s) \geq \pi \uQ^t_h(s) \geq \pi Q^{\pi,t+1}_h(s) = V^{\pi,t+1}_h(s).
    \]
\end{proof}

\paragraph{Regret Bound}

\begin{lemma}\label{lem:regret_sampler}
    Let $\pi$ be any fixed policy. Then for any $\delta \in(0,1)$ with probability at least $1-\delta$ the following holds
    \[
        \regret_{\samp}^T(d^\pi)  \leq 10\log(T+SA)\sqrt{2H^4 S A T \cdot \left( \log(2SAH/\delta) + S\log(\rme T)\right)\log(T)}.
    \]
\end{lemma}
\begin{proof}
    Assume that the event $\cG(\delta)$ holds.   By Lemma~\ref{lem:sampler_optimism} for any $t \in [T]$ and $h\in[H]$ we have
    \[
        V^{\pi, t}_t(s_h) - V^{\pi^t, t}_h(s^t_h) \leq \uV^{t-1}_h(s^t_h) - V^{\pi^t, t}_h(s^t_h) = \pi^t_h(\uQ^{t-1}_h - Q^{\pi^t, t}_h)(s),
    \]
    thus we can define $\delta^t_h(s,a) = \uQ^{t-1}_h(s,a) - Q^{\pi^t, t}_h(s,a)$ and upper bound the regret as follows
    \[
        \regret_{\samp}^T(d^\pi) \leq \sum_{t=1}^T \pi^t_1 \delta^t_1(s_1).
    \]

    Next we analyze $\delta^t_h(s^t_h)$. By the same argument as in Lemma~\ref{lem:sampler_optimism}
    \[
        \delta^t_h(s,a) = [\hp^{t-1}_h - p_h] \uV^{t-1}_{h+1}(s,a) + b^t_h(s,a) + p_h[\uV^{t-1}_{h+1} - V^{\pi^t, t}_{h+1}](s,a) \leq 2b^{t-1}_h(s,a) + p_h\pi^{t}_{h+1} [\uQ^{t-1}_{h+1} - Q^{\pi^t, t}_{h+1}](s,a)
    \]
    that could be rewritten as follows
    \[
        \delta^t_h(s,a) \leq \E_{\pi^t}\left[ 2b^{t-1}_h(s,a) + \delta^t_{h+1}(s_{h+1}, a_{h+1}) | (s_h, a_h) = (s,a)\right],
    \]
    thus, rolling out the initial bound on regret we have
    \[
        \regret_{\samp}^T(d^\pi) \leq H\log(T+SA) \sum_{t=1}^{T-1} \E_{\pi^t}\left[ \sum_{h=1}^H 2\sqrt{\frac{2\alpha^{\KL}(\delta, n^t_h(s_h, a_h))}{n^t_h(s_h,a_h)} \wedge 1} \bigg| s_1 \right] + H\log(T+SA).
    \]
    By Lemma~\ref{lem:cnt_pseudo} and Jensen's inequality we have
    \[
        \regret_{\samp}^T(d^\pi) \leq 5 H^{3/2}\log(T+SA) \sqrt{2T} \sqrt{ \sum_{h,s,a} \sum_{t=1}^{T-1} d^{\pi^t}_h(s,a)  \frac{\alpha^{\KL}(\delta, \upn^t_h(s, a)) }{\upn^t_h(s,a) \vee 1}}.
    \]
    Notice that $d^{\pi^t}_h(s,a) = \upn^{t+1}_h(s,a) - \upn^t_h(s,a)$ and $\alpha^{\KL}(\delta, \upn^t_h(s,a)) \leq \alpha^{\KL}(\delta, T-1)$. Combined with Lemma~\ref{lem:sum_1_over_n} it implies
    \[
        \regret_{\samp}^T(d^\pi) \leq 10\log(T+SA)\sqrt{2H^4 S A T \cdot \left( \log(2SAH/\delta) + S\log(\rme T)\right)\log(T)}.
    \]

    Finally, the fact that $\P[\cG(\delta)] \geq 1 -\delta$ concludes the statement of theorem.
\end{proof}
\begin{remark}
    It is possible to improve the $H$-dependence by introducing Bernstein-type bonuses, however, we are focused on improvement in a dependence in $\varepsilon^{-1}$ and leave this regret bound as simple as possible.
\end{remark}

\subsection{Bias Terms}

\begin{lemma}\label{lem:bias_terms}
    Let $\delta \in (0,1)$ and $n_0 = 1$. Then with probability at least $1-\delta$ the following two bounds hold
    \begin{align*}
        \mathrm{Bias}_1 &\triangleq \sum_{t=1}^T \sum_{h,s,a} \big(d_h^{\pi^t}(s,a) - \td_h^t(s,a) \big) \log\frac{1}{\bd_h^t(s,a)} \leq  \log(T+SA) \sqrt{2TH \log(2/\delta)} \\
        \mathrm{Bias}_2 &\triangleq T(\VE(\rd^T) - \VE(\hd^T)) \leq \log(SAT)\left(\sqrt{2TH\log(2/\delta)} + 3H\sqrt{SAT\log(3T)} \right).
    \end{align*}
\end{lemma}
\begin{proof}
    Let us define the lexicographic order on the set $[T] \times [H]$ with an additional convention $(t,0) = (t-1, H)$.
    
    Then we can define a filtration $\cF_{t,h} = \sigma\left\{ (s^{t'}_{h'}, a^{t'}_{h'}) \ \forall t \leq t,  \forall h' \in [H] \} \cup \{ (s^t_{h'}, a^t_{h'})\  \forall h' \leq h \} \right\}$ that consists of the all history of interactions of the \algMVEE algorithm with an environment up to the $h$-th step of the episode $t$. The most important fact is that $\pi^t$ and $\bd^t_h(s,a)$ are $\cF_{t,h-1}$-measurable for $h > 1$ and $\cF_{t-1,H}$-measurable for $h=1$. 
    
    Therefore,  for any $t \in [T], h \in [H]$
    \[
        \E\left[ \sum_{s,a} (d_h^{\pi^t}(s,a) - \td_h^t(s,a)) \log \frac{1}{\bd_h^t(s,a)} \bigg| \cF_{t,h-1} \right] = 0.
    \]
    Therefore $X_{t,h} = \sum_{s,a} (d_h^{\pi^t}(s,a) - \td_h^t(s,a)) \log \frac{1}{\bd_h^t(s,a)}$ is a martingale-difference sequence adapted to the filtration $\cF_{t,h}$. Also we notice that a.s. the following bound holds
    \[
        \vert X_{t,h} \vert \leq \log(T+SA)
    \]

    All these facts combined with Azuma-Hoeffding inequality implies that with probability at least $1-\delta/2$
    \[
        \mathrm{Bias}_1 = \sum_{t=1}^T \sum_{h=1}^H X_{t,h} \leq \log(T+SA) \sqrt{2TH \log(2/\delta)}.
    \]

    To show the second part of the statement we notice that
    \[
        \mathrm{Bias}_2 = T \sum_{h=1}^H (\cH(\rd^T_h) - \cH(\hd^T_h)).
    \]
    Let us introduce the smoothed entropy as it was done by \citet{hazan2019provably}. 
    \[
        \forall d \in \simplex_{SA}: \cH_\sigma(d) = \sum_{s,a} d(s,a) \log \frac{1}{d(s,a) + \sigma}.
    \]
    The key difference with our approach and approach of \citet{hazan2019provably} that we need the smoothing only instrumentally to provide a bound on $\mathrm{Bias}_2$.
    
    It is easy to see that $\cH_\sigma$ is concave and, moreover $d \in \simplex_{SA}$
    \[
        \vert \cH(d) - \cH_\sigma(d)  \vert \leq \sum_{s,a} d(s,a) \log \frac{d(s,a) +\sigma}{d(s,a) }  \leq \sigma SA,
    \]
    where we used inequality $\log(1+x) \leq x$ for all $x \geq 0$, and also for $\sigma < \rme^{-1}$
    \[
        \norm{ \nabla \cH_\sigma(d) }_\infty = \sup_{x \in (0,1)} \left\vert \log(x + \sigma) + \frac{x}{x + \sigma} \right\vert \leq \log(\sigma^{-1}).
    \]
    By replacing an entropy with a smoothed entropy
    \[
        \mathrm{Bias}_2 \leq T\sum_{h\in H} (\cH_\sigma(\rd^T_h) - \cH_\sigma(\hd^T_h)) + 2\sigma \cdot TSAH.
    \]
    To analyze the first term we use that $\cH_\sigma$ is concave, therefore
    \[
        \sum_{h=1}^H \cH_\sigma(\rd^T_h) - \cH_\sigma(\hd^T_h) \leq \sum_{h=1}^H \langle \nabla \cH_\sigma(\hd^T_h), \rd^T_h - \hd^T_h \rangle = \frac{1}{T} \sum_{s,a}\sum_{t=1}^T \sum_{h=1}^H (\td_h^t(s,a) - d_h^{\pi^t}(s,a)) \cdot \nabla \cH_\sigma(\hd^T_h)(s,a)
    \]

     For this term situation is more involved than for $\mathrm{Bias}_1$ because $\hd^T_h$ is dependent on all generated policies. Therefore we have to preform uniform bounds. Define $\cW = \{ w \in \R^{HSA} \mid \vert w_h(s,a) \vert \leq 1\} $ as a unit $\ell_\infty$-ball in $\R^{HSA}$. Then we have
     \[
        T\sum_{h=1}^H \cH_\sigma(\rd^T_h) - \cH_\sigma(\hd^T_h) \leq \log(\sigma^{-1}) \cdot \sup_{w \in \cW } \sum_{t=1}^T \sum_{h=1}^H \left( \sum_{s,a} (\td_h^t(s,a) - d_h^{\pi^t}(s,a)) \cdot w_h(s,a) \right).
     \]
     
     Define $N(\varepsilon, \norm{\cdot}_\infty, \cW)$ as $\epsilon$-covering number for a set $\cW$ with $\ell_\infty$-norm as a distance, and $\cW_\varepsilon$ as a minimal $\varepsilon$-net. Next we can use the well-known result on upper bound on the covering number (e.g. see Exercise 5.5 by \citet{van2014probability})
    \[
        N(\epsilon, \norm{\cdot}_\infty, \cW) \leq \left(\frac{3}{\varepsilon}\right)^{SAH},
    \]
    and replace our maximization problem with the maximization over $\varepsilon$-net
    \[
        \sup_{w \in \cW } \sum_{t=1}^T \sum_{h=1}^H \left( \sum_{s,a} (\td_h^t(s,a) - d_h^{\pi^t}(s,a)) \cdot w_h(s,a) \right) \leq \sup_{\hat w \in \cW_\varepsilon } \sum_{t=1}^T \sum_{h=1}^H \left( \sum_{s,a} (\td_h^t(s,a) - d_h^{\pi^t}(s,a)) \cdot \hat w_h(s,a) \right) + \varepsilon TH.
    \]
    For any fixed $\hat w \in \cW_{\varepsilon}$ we apply Azuma-Hoeffding inequality exactly in the same manner as in the bound for $\mathrm{Bias}_1$-term. We have that with probability at least $1-\delta/(2N_\varepsilon)$ for $N = N(\varepsilon, \norm{\cdot}_\infty, \cW)$ we have
    \[
        \sum_{t=1}^T \sum_{h=1}^H \left( \sum_{s,a} (\td_h^t(s,a) - d_h^{\pi^t}(s,a)) \cdot \hat w_h(s,a) \right) \leq \sqrt{2TH\log(2 N_{\varepsilon} / \delta)}.
    \]
    Thus, by union bound we have with probability at least $1-\delta/2$
    \[
         \sup_{w \in \cW } \sum_{t=1}^T \sum_{h=1}^H \left( \sum_{s,a} (\td_h^t(s,a) - d_h^{\pi^t}(s,a)) \cdot w_h(s,a) \right) \leq \sqrt{2TH(\log(2 / \delta) + SAH\log(3/\varepsilon)) } + \varepsilon TH.
    \]
    Taking $\varepsilon = 1/T$ we have
    \[
        \mathrm{Bias}_2 \leq \log(\sigma^{-1})(\sqrt{2TH(\log(2 / \delta) + SAH\log(3T)) } + H) + \sigma SATH.
    \]
    Next we choose $\sigma = 1/SAT$ and by inequality $\sqrt{a+b} \leq \sqrt{a} +\sqrt{b}$ obtain
    \[
        \mathrm{Bias}_2 \leq \log(SAT)\left(\sqrt{2TH\log(2/\delta)} + 3H\sqrt{SAT\log(3T)} \right).
    \]
\end{proof}

\subsection{Proof of Theorem~\ref{th:MVEE_sample_complexity}}

We state the version of this theorem with all prescribed dependencies factors.
\begin{theorem}\label{th:MVEE_sample_complexity_full}
For all $\epsilon > 0$ and $\delta\in(0,1)$. For $n_0=1$ and 
\[
T \geq 1 + \frac{648 (\log(SA) + L)H^4 SA \cdot (\log(4SAH/\delta) + S + L) \cdot L}{\varepsilon^2} + \frac{2 HSA(2+L)}{\varepsilon}
\]
for $L = 9 \log\left( 1010 \sqrt{H^4 S^{8/3} A^{8/3} \log(4SAH/\delta)} / \varepsilon \right)$ the algorithm \algMVEE is $(\epsilon,\delta)$-PAC.
\end{theorem}
\begin{proof}
    We start from writing down the decomposition defined in the beginning of the appendix
    \[
        T(\VE(d^{\pistarVE}) - \VE(d^{\hpi})) \leq \regret_{\samp}^T(d^{\pistarVE}) + \regret_{\fore}^T(\rd^T) + \mathrm{Bias}_1 + \mathrm{Bias}_2.
    \]
    By Lemma~\ref{lem:regret_sampler} with probability at least $1-\delta/2$ it holds
    \[
        \regret_{\samp}^T(d^{\pistarVE}) = 10\log(T+SA)\sqrt{2H^4 S A T \cdot \left( \log(4SAH/\delta) + S\log(\rme T)\right)\log(T)}
    \]
    By Lemma~\ref{lem:regret_forecaster} 
    \[
        \regret_{\fore}^T(\rd^T) \leq HSA\log\big(\rme(T+1)\big).
    \]
    By Lemma~\ref{lem:bias_terms} with probability at least $1-\delta/2$
    \[
         \mathrm{Bias}_1 + \mathrm{Bias}_2 \leq 3\log(SAT)\left(\sqrt{TH\log(4/\delta)} + H\sqrt{SAT\log(3T)} \right).
    \]
    By union bound all these inequalities hold simultaneously with probability at least $1-\delta$. Combining all these bounds we get
    \[
        T(\VE(d^{\pistarVE}) - \VE(d^{\hpi})) \leq 18 \log(SAT) \sqrt{ H^4 SAT (\log(4SAH/\delta) + S \log(\rme T)) \log(T) } + HSA\log(\rme (T+1)).
    \]
    Therefore, it is enough to choose $T$ such that $\VE(d^{\pistarVE}) - \VE(d^{\hpi})$ is guaranteed to be less than $\varepsilon$. In this case \algMVEE become automatically $(\varepsilon,\delta)$-PAC.
    
    It is equivalent to find a maximal $T$ such that
    \[
        \varepsilon T \leq 18 \log(SAT) \sqrt{ H^4 SAT (\log(4SAH/\delta) + S \log(\rme T)) \log(T) } + HSA\log(\rme (T+1))
    \]
    and add $1$ to it.

    We start from obtaining a loose bound to eliminate logarithmic factors in $T$.
    
    First, we assume that $T \geq 1$, thus $T+1 \leq 2T$. Additionally, let us use inequality $\log(x) \leq x^{\beta}/\beta$ for any $x > 0$ and $\beta > 0$. We obtain
    \begin{align*}
        \varepsilon T &\leq 18 \frac{(SAT)^{1/3}}{1/3} \sqrt{ H^4 S^2 AT \log(4SAH/\delta) \frac{(\rme T)^{1/18}}{1/18} \frac{T^{1/18}}{1/18} } + HSA \frac{(2\rme T)^{8/9}}{8/9}\\
        &\leq T^{8/9} \left( 1010 \sqrt{H^4 S^2 A^2 \log(2SAH/\delta)} \right),
    \end{align*}
    thus we can define $L = 9 \log\left( 1010 \sqrt{H^4 S^{8/3} A^{8/3} \log(4SAH/\delta)} / \varepsilon \right)$ for which $\log(T) \leq L$. Thus we have
    \[
        \varepsilon T \leq 18 (\log(SA) + L) \sqrt{H^4 SAT(\log(4SAH/\delta) + S + L) L} + HSA(2+ L).
    \]
    Solving this quadratic inequality, we obtain the minimal required $T$ to guarantee $\VE(d^{\pistarVE}) - \VE(d^{\hpi}) \leq \varepsilon$. In particular,
    \[
        T \geq 1 + \frac{648 (\log(SA) + L)H^4 SA \cdot (\log(4SAH/\delta) +S + L) \cdot L}{\varepsilon^2} + \frac{2 HSA(2+L)}{\varepsilon}.
    \]
\end{proof}
\newpage
\section{Regularized Bellman Equations}\label{app:reg_bellman_eq}

In this section we provide complete proofs for regularized Bellman Equations in the general setting. Let $\Phi \colon \Delta_{\cA} \to \R$ be a strictly convex function.

Then we can define the regularized value function as follows
\begin{equation}\label{eq:reg_value_func}
    V^{\pi}_{\lambda,h}(s) \triangleq \E_\pi\left[ \sum_{h'=h}^H r_{h'}(s_{h'}, a_{h'}) -\lambda \Phi(\pi_{h'}(s_{h'}))  \mid s_h = s \right].
\end{equation}
Notably, for a specific choice of rewards $r_h(s,a) = \cH(p_h(s,a))$, the regularizer is equal to the negative entropy $\Phi(\pi) = -\cH(\pi)$, and $\lambda = 1$ we have $V^{\pi}_{\lambda,1}(s_1) = \TE(q^\pi)$, see Lemma~\ref{lem:trajectory_entropy_formula} In more general setting let $r_h(s,a)$ be equal to the sum of deterministic reward and $\lambda \cH(p_h(s,a))$. In this case we have $V^{\pi}_{\lambda,1}(s_1) = V^\pi_1(s_1) + \lambda\TE(q^\pi)$ in terms of a usual non-regularized value function.

Let us define an entropy action-value function as follows
\begin{equation}\label{eq:reg_q_func}
    Q^{\pi}_{\lambda,h}(s,a) \triangleq \E_{\pi}\left[ r_h(s_h,a_h) + \sum_{h'=h+1}^H \left[  r_{h'}(s_{h'}, a_{h'}) - \lambda \Phi(\pi_h(s_{h'})) \right] \mid (s_h,a_h) = (s,a) \right].
\end{equation}

Additionally, we define an optimal entropy-regularized value functions a follows
\[
    \Vstar_{\lambda,h}(s) \triangleq \max_{\pi} V^{\pi}_{\lambda,h}(s), \quad \Qstar_{\lambda,h}(s,a) \triangleq \max_{\pi} Q^{\pi}_{\lambda,h}(s,a) \quad \forall (s,a,h) \in \cS \times \cA \times [H].
\]

\subsection{Proof of Entropy-Regularized Bellman Equations}\label{app:proof_entropic_bellman_eq}

\begin{theorem}[Regularized Bellman Equations]
    For any stochastic policy $\pi$ the following decomposition of the entropy-regularized value function holds
    \begin{align}\label{eq:reg_bellman_equation}
    \begin{split}
        Q^{\pi}_{\lambda,h}(s,a) &= r_h(s,a) + p_h V^{\pi}_{\lambda,h+1}(s,a), \\
        V^{\pi}_{\lambda, h}(s) &= \pi_h Q^{\pi}_{\lambda,h}(s) - \lambda \Phi(\pi_h(s)).
    \end{split}
    \end{align}

    Moreover, for optimal $Q$- and $V$-functions we have
    \begin{align}\label{eq:opt_reg_bellman_equation}
    \begin{split}
        \Qstar_{\lambda,h}(s,a) &= r_h(s,a) + p_h \Vstar_{\lambda,h+1}(s,a), \\
        \Vstar_{\lambda,h}(s) &= \max_{\pi \in \Delta_{\cA}}\left\{ \pi \Qstar_h(s) - \lambda \Phi(\pi) \right\}.
    \end{split}
    \end{align}
\end{theorem}
\begin{remark}
    For the case of interest $\Phi(\pi) = - \cH(\pi)$ the expression for the $V$-function allows the closed-form formula by a well-known LogSumExp smooth maximum approximation
    \[
        \Vstar_{\lambda,h}(s) = \lambda \log\left( \sum_{a \in \cA} \exp\left( 
        \frac{1}{\lambda} \Qstar_{\lambda,h}(s,a) \right) \right),
    \]
    and as $\lambda \to 0$ we see that entropy-regularized value function tends to a usual value function without regularization.
\end{remark}
\begin{proof}
    We proceed by induction. For $h=H+1$ the equation is trivial. By definition and tower property of conditional expectation
    \begin{align*}
        Q^{\pi}_{\lambda,h}(s,a) &= r_h(s,a) + \E\left[ \sum_{t=h+1}^{H} r_t(s_t, a_t) - \lambda \Phi(\pi_t(s_t)) \biggl| s_h=s, a_h=a \right] \\
        &= r(s,a) + \E\left[ \E\left[\sum_{t=h+1}^{H} r_t(s_t, a_t) - \lambda \Phi(\pi_t(s_t)) \biggl| s_{h+1}\right] \biggl| s_h=s, a_h=a \right] \\
        &= r_h(s,a) + p_h V^{\pi}_{\lambda, h+1}(s,a).
    \end{align*}
    Next we provide the second Bellman equation by tower property and the definition of the regularized $Q$-function
    \begin{align*}
        V^{\pi}_{\lambda,h}(s) &=  - \lambda \Phi(\pi_h(s)) + \E\left[ r_h(s_h, a_h) + \sum_{t=h+1}^{H} r_t(s_t, a_t) - \lambda \Phi(\pi_t(s_t))  \biggl| s_h=s \right] \\
        &= \pi_h Q^{\pi}_{\lambda,h}(s) - \lambda \Phi(\pi_h(s)).
    \end{align*}

    For optimal Bellman equation we proceed by induction. For $h=H+1$ the equation is also trivial. By Bellman equations
    \[
        \Qstar_{\lambda,h}(s,a) = \max_{\pi} \left\{ r_h(s,a) + p_h V^{\pi}_{\lambda,h+1}(s,a) \right\} = r_h(s,a) + p_h \Vstar_{\lambda,h+1}(s,a),
    \]
    and, finally
    \[
        \Vstar_{\lambda,h}(s) = \max_{\pi_1,\ldots,\pi_H \in \simplex_{\cA}}\left\{ \pi_h \Qstar_{\lambda,h}(s)  - \lambda \Phi(\pi_h(s))\right\} = \max_{\pi \in \Delta_{\cA}}\left\{ \pi \Qstar_{\lambda,h}(s) - \lambda \Phi(\pi) \right\}.
    \]
\end{proof}

\subsection{A Bellman-type Equations for Variance}\label{app:Bellman_variance}
For a stochastic policy $\pi$ we define Bellman-type equations for the variances as follows
\begin{align*}
  \Qvar_{\lambda,h}^{\pi}(s,a) &\triangleq \Var_{p_h}{V_{\lambda,h+1}^{\pi}}(s,a) + p_h \Vvar^{\pi}_{\lambda,h+1}(s,a)\\
  \Vvar_{\lambda,h}^{\pi}(s) &\triangleq \Var_{\pi_h}{Q^{\pi}_{\lambda,h}}(s) + \pi_h \Qvar^{\pi}_{\lambda,h} (s)\\
  \Vvar_{\lambda,H+1}^{\pi}(s)&\triangleq0,
\end{align*}
where $\Var_{p_h}(f)(s,a) \triangleq \E_{s' \sim p_h(\cdot | s, a)} \big[(f(s')-p_h f(s,a))^2\big]$ denotes the \emph{variance operator over transitions} and $\Var_{\pi_h}(f)(s) \triangleq \E_{a' \sim \pi_h(s)}\big[ (f(s,a') - \pi_h f(s))^2 \big]$ denoted the \emph{variance operator over the policy}.
 In particular, the function $s \mapsto \Vvar_{\lambda,1}^{\pi}(s)$ represents the average sum of the local variances $\Var_{p_h}{V_{\lambda, h+1}^{\pi}}(s,a)$ and $ \Var_{\pi_h}{Q^{\pi}_{\lambda,h}}(s) $ over a trajectory following the policy $\pi$, starting from $(s, a)$. Indeed, the definition above implies that
 \[\Vvar_{\lambda,1}^{\pi}(s_1) = \sum_{h=1}^H \sum_{s\in\cS} d_h^\pi(s) \Var_{\pi_h}{Q^{\pi}_{\lambda,h}}(s) +  \sum_{h=1}^H\sum_{s,a} d_h^\pi(s,a) \Var_{p_h}(V_{\lambda,h+1}^{\pi})(s,a).
 \]
 The lemma below shows that we can relate the global variance of the cumulative reward over a trajectory to the average sum of local variances.
\begin{lemma}[Law of total variance]\label{lem:law_of_total_variance}  For any stochastic policy $\pi$ and for all $h\in[H]$,
\begin{align*}
   \Qvar_{\lambda,h}^{\pi}(s,a) &= \E_\pi\!\left[  \left(\sum_{h'=h}^H \left( r_{h'}( s_{h'},a_{h'}) -\lambda \Phi(\pi_{h'}(s_{h'})) \right) - \left( Q_{\lambda,h}^{\pi}(s_h,a_h) -\lambda \Phi(\pi_h(s_h)) \right) \right)^{\!\!2}\middle| (s_h,a_h)=(s,a) \right], \\
  \Vvar_{\lambda,h}^{\pi}(s) &= \E_\pi\!\left[  \left(\sum_{h'=h}^H \left( r_{h'}( s_{h'},a_{h'}) -\lambda \Phi(\pi_{h'}(s_{h'})) \right) - V_{\lambda,h}^{\pi}(s_h) \right)^{\!\!2}\middle| s_h=s \right].
\end{align*}
\end{lemma}
\begin{proof}
	We proceed by induction. The statement in Lemma~\ref{lem:law_of_total_variance} is trivial for $h=H+1$. We now assume that it holds for $h+1$ and prove that it also holds for $h$. For this purpose, we compute
	\begin{align*}
		&
		%\Qvar_h^\pi(s,a) =
		\E_\pi\left[\!
		 \left(\sum_{h'=h}^H \left( r_{h'}( s_{h'},a_{h'}) -\lambda \Phi(\pi_{h'}(s_{h'})) \right) -  \left( Q_{\lambda,h}^{\pi}(s_h,a_h) -\lambda \Phi(\pi_h(s_h))\right) \right)^{\!\!2}
		 \middle| (s_h,a_h) \right] \\
		& =
		\E_\pi\left[\! \left( V_{\lambda, h+1}^{\pi}(s_{h+1}) - p_h V_{\lambda, h+1}^{\pi}(s_h,a_h) + \sum_{h'=h+1}^H \left( r_{h'}( s_{h'},a_{h'}) -\lambda \Phi(\pi_{h'}(s_{h'})) \right) - V_{\lambda, h+1}^{\pi}(s_{h+1})\right)^{\!\!2}
		\middle| (s_h,a_h) \right] \\
		& =
		\E_\pi\left[\!
			\left(  V_{\lambda, h+1}^{\pi}(s_{h+1}) - p_h V_{\lambda, h+1}^{\pi}(s_h,a_h) \right)^{\!\!2}
		\middle| (s_h,a_h) \right] \\
		&
		+ \E_\pi\left[\!
		\left( \sum_{h'=h+1}^H \left( r_{h'}( s_{h'},a_{h'}) -\lambda \Phi(\pi_{h'}(s_{h'})) \right) - V_{\lambda, h+1}^{\pi}(s_{h+1})\right)^{\!\!2}
		\middle| (s_h,a_h) \right] \\
		& + 2 \E_\pi\left[\!
			\left(  \sum_{h'=h+1}^H \left( r_{h'}( s_{h'},a_{h'}) -\lambda \Phi(\pi_{h'}(s_{h'})) \right) - V_{\lambda, h+1}^{\pi}(s_{h+1}) \right)
			\left( V_{\lambda, h+1}^{\pi}(s_{h+1}) - p_h V_{\lambda, h+1}^{\pi}(s_h,a_h) \right)
		\middle| (s_h,a_h) \right].
	\end{align*}
	The definition of  $V_{\lambda,h+1}^\pi(s_{h+1})$ implies that \[\E_\pi\!\left[  \sum_{h'=h+1}^H \left( r_{h'}( s_{h'},a_{h'}) -\lambda \Phi(\pi_{h'}(s_{h'})) \right) - V_{\lambda, h+1}^{\pi}(s_{h+1})	\middle| s_{h+1} \right] = 0.\]
	Therefore, the tower property of conditional expectation gives us
	\begin{align*}
		&\E_\pi\left[\!
		 \left(\sum_{h'=h}^H \left( r_{h'}( s_{h'},a_{h'}) -\lambda \Phi(\pi_{h'}(s_{h'})) \right) -  \left( Q_{\lambda,h}^{\pi}(s_h,a_h) -\lambda \Phi(\pi_h(s_h))\right) \right)^{\!\!2}
		 \middle| (s_h,a_h) \right]  \\ 
        &= \E_\pi\left[\!
			\left(  V_{\lambda, h+1}^{\pi}(s_{h+1}) - p_h V_{\lambda, h+1}^{\pi}(s_h,a_h) \right)^{\!\!2}
		\middle| (s_h,a_h) \right] \\
		&
		+ \E\left[\! \E_\pi\left[
		\left( \sum_{h'=h+1}^H \left( r_{h'}( s_{h'},a_{h'}) -\lambda \Phi(\pi_{h'}(s_{h'})) \right) - V_{\lambda, h+1}^{\pi}(s_{h+1})\right)^{\!\!2}
		\middle| s_{h+1} \right] \middle| (s_h,a_h) \right] \\
		&  = \Var_{p_h}{V_{\lambda, h+1}^{\pi}}(s_h,a_h) + p_h \Vvar^{\pi}_{\lambda,h+1}(s_h,a_h) = \Qvar_{\lambda,h}^{\pi}(s_h,a_h)
	\end{align*}
	where in the third equality we used the inductive hypothesis and the definition of $\sigma V_{h+1}^{\pi}$. To prove the second equation we use the entropy-regularized Bellman equations
    \begin{align*}
        & \E_\pi\!\left[  \left(\sum_{h'=h}^H \left( r_{h'}( s_{h'},a_{h'}) -\lambda \Phi(\pi_{h'}(s_{h'})) \right) - V_{\lambda,h}^{\pi}(s_h) \right)^{\!\!2}\middle| s_h=s \right] \\
        &= \E_\pi\!\left[  \left(\sum_{h'=h}^H \left( r_{h'}( s_{h'},a_{h'}) -\lambda \Phi(\pi_{h'}(s_{h'})) \right) - (Q_{\lambda,h}^{\pi}(s_h, a_h) - \lambda \Phi(\pi_{h}(s_h)) ) \right)^{\!\!2}\middle| s_h=s \right] \\
        &+ 2\E_\pi\!\left[  \left(\sum_{h'=h}^H \left( r_{h'}( s_{h'},a_{h'}) -\lambda \Phi(\pi_{h'}(s_{h'})) \right) - (Q_{\lambda,h}^{\pi}(s_h, a_h) - \lambda \Phi(\pi_{h}(s_h)) )\right) \left( \pi_h Q_{\lambda,h}^{\pi}(s_h) - Q_{\lambda,h}^{\pi}(s_h, a_h)  \right)\middle| s_h=s \right] \\
        &+ \E_{\pi}\!\left[\left( \pi_h Q_{\lambda,h}^{\pi}(s_h) - Q_{\lambda,h}^{\pi}(s_h, a_h)  \right)^{2} \middle| s_h=s\right].
    \end{align*}
    By definition of $Q^{\pi}_{\lambda,h}$ we have
    \[
        \E_\pi\left[ \left(\sum_{h'=h}^H \left( r_{h'}( s_{h'},a_{h'}) -\lambda \Phi(\pi_{h'}(s_{h'})) \right) - (Q_{\lambda,h}^{\pi}(s_h, a_h) - \lambda \Phi(\pi_{h}(s_h)) )\right) \middle| (s_h,a_h) = (s,a) \right] = 0,
    \]
    thus, by the tower property
    \begin{align*}
        &\E_\pi\!\left[  \left(\sum_{h'=h}^H \left( r_{h'}( s_{h'},a_{h'}) -\lambda \Phi(\pi_{h'}(s_{h'})) \right) - V_{\lambda,h}^{\pi}(s_h) \right)^{\!\!2}\middle| s_h=s \right] = \pi_h \Qvar^{\pi}_{\lambda,h}(s_h) + \Var_{\pi_h} Q^{\pi}_{\lambda, h}(s_h) = \Vvar^{\pi}_{\lambda, h}(s_h).
    \end{align*}
\end{proof}

\subsection{Performance-Difference Lemma}

In this section we provide a version of performance-difference lemma (see e.g. Lemma E.15 by \citet{dann2017unifying}) for regularized Bellman equations. 

\begin{lemma}[Performance-Difference Lemma]\label{lm:performance_difference}
    Let $\cM' = (\cS, \cA, H, \{p'_h\}_{h\in [H]}, \{ r'_h\}_{h\in[H]}, s_1)$ and $\cM'' = (\cS, \cA, H, \{p''_h\}_{h\in [H]}, \{ r''_h\}_{h\in[H]}, s_1)$ be two MDPs and let $Q^{\cM, \pi}_{\lambda, h}(s,a)$ be a Q-value of policy $\pi$ in the MDP $\cM$ with regularization. Then for any $(s,a,h) \in \cS \times \cA \times [H],$
    \begin{align*}
        Q^{\cM', \pi}_{\lambda,h}(s,a) - Q^{\cM'', \pi}_{\lambda,h}(s,a) &= \E_{\cM'', \pi}\left[ \sum_{h'=h}^H r'_{h'}(s_{h'},a_{h'}) - r''_{h'}(s_{h'},a_{h'}) \ \bigg|\ (s_h,a_h) = (s,a) \right] \\
        &+ \E_{\cM'', \pi}\left[ \sum_{h'=h}^H [p'_{h'} - p''_{h'}] V^{\cM', \pi}_{\lambda, h'+1}(s_{h'}, a_{h'}) \ \bigg|\ (s_h,a_h) = (s,a) \right].
    \end{align*}
\end{lemma}
\begin{proof}
    Let us proceed by induction over $h$. For $h = H+1$ this statement is trivially true. Next we assume that it holds for any $h' > h$. By regularized Bellman equations
    \begin{align*}
        Q^{\cM', \pi}_{\lambda,h}(s,a) - Q^{\cM'', \pi}_{\lambda,h}(s,a) &= \left[r'_{h}(s,a) - r''_h(s,a)\right] + p'_{h} V^{\cM', \pi}_{\lambda,h}(s,a) - p''_h V^{\cM'', \pi}_{\lambda,h}(s,a) \\
        &= \left[r'_{h}(s,a) - r''_h(s,a)\right] + [p'_{h} - p''_h]V^{\cM', \pi}_{\lambda,h+1}(s,a) - p''_h \left[V^{\cM'', \pi}_{\lambda,h+1} - V^{\cM', \pi}_{\lambda,h+1}\right](s,a) \\
        &=  \left[r'_{h}(s,a) - r''_h(s,a)\right] + [p'_{h} - p''_h]V^{\cM', \pi}_{\lambda,h+1}(s,a) + p''_h \left[V^{\cM', \pi}_{\lambda,h+1} - V^{\cM'', \pi}_{\lambda,h+1}\right](s,a).
    \end{align*}
    Next we notice that 
    \[
        V^{\cM', \pi}_{\lambda,h+1}(s) - V^{\cM'', \pi}_{\lambda,h+1}(s) = \pi Q^{\cM', \pi}_{\lambda,h+1}(s) - \lambda \Phi(\pi) - \pi Q^{\cM'', \pi}_{\lambda,h+1}(s) + \lambda \Phi(\pi) = \pi \left[ Q^{\cM', \pi}_{\lambda,h+1} -  Q^{\cM'', \pi}_{\lambda,h+1}\right](s)
    \]
    since the regularizer is cancelled out. Thus, we can rewrite this difference as follows
    \begin{align*}
        Q^{\cM', \pi}_{\lambda,h}(s,a) - Q^{\cM'', \pi}_{\lambda,h}(s,a) &= \E_{\pi, \cM''}\biggl[ r'_h(s_h,a_h) - r''_h(s_h,a_h) + [p'_{h} - p''_h]V^{\cM', \pi}_{\lambda,h+1}(s_h,a_h)  \\
        &+ \left[Q^{\cM', \pi}_{\lambda,h+1}(s_{h+1},a_{h+1}) - Q^{\cM'', \pi}_{\lambda,h+1}(s_{h+1},a_{h+1})\right] \ \bigg|\ (s_h,a_h)= (s,a) \biggl].
    \end{align*}
    By induction hypothesis we conclude the statement.
\end{proof}
\newpage
\section{Sample Complexity for MTEE and Regularized MDPs}\label{app:regularized_mdp}

In this section we describe the general setting of regularized MDPs, not only entropy-regularized.

\subsection{Preliminaries}

First we define class of regularizers we are interested in. For more exposition on this definition, see \citet{bubeck2015convex}. 
\begin{definition}\label{def:mirror_map}
    Let $\Phi \colon \simplex_{\cA} \to \R$ be a proper closed strongly-convex function. We will call $\Phi$ a mirror-map if the following holds 
\begin{itemize}
    \item $\Phi$ is $1$-strongly convex with respect to norm $\norm{\cdot}$;
    \item $\nabla \Phi$ takes all possible values in $\R^{\cA}$;
    \item $\nabla \Phi$ diverges on the boundary of $\simplex_{\cA}$: $\lim_{x \in \partial \simplex_{\cA}} \norm{\nabla \Phi(x)} = +\infty$;
\end{itemize}
\end{definition}

We explain three main examples of a such regularizers.
\begin{itemize}
    \item The negative Shannon entropy $\Phi(\pi) = -\cH(\pi)$ for $\cH(\pi) = \sum_{a\in\cA} \pi_a \log\left(\frac{1}{\pi_a}\right)$ satisfies the Definition~\ref{def:mirror_map} for $\ell_1$-norm;
    \item The negative Tsallis entropy $\Phi(\pi) = - \frac{1}{q}\cT_{q}(\pi)$ for $\cT_q(\pi) = \frac{1}{q - 1} \left(1 - \sum_{a\in\cA} \pi_a^{q} \right)$ satisfied the Definition~\ref{def:mirror_map} for $\ell_2$ norm for every $q\in(0,1)$. In particular, $q=0.5$ corresponds to the choice by \citet{grill2019planning} in Appendix E that is tightly connected to the UCB algorithm;
    \item For any other fixed policy $\pi'\in\simplex_{\cA}$ we can choose $\Phi(\pi) = \KL(\pi, \pi') = \sum_{a\in\cA} \pi_a \log\left( \frac{\pi_a}{\pi'_a}\right)$ that inherits all the properties from the choice of the negative entropy.
\end{itemize}

Let $\cM = (\cS, \cA, \{p_h\}_{h\in[H]}, \{r_h\}_{h\in[H]}, s_1)$ be a finite-horizon MDP, where $r_h(s,a)$ is a deterministic reward function. For simplicity we assume that $0 \leq r_h(s,a) \leq r_{\max}$ for any $(h,s,a) \in [H] \times \cS \times \cA$. 

Then we can define entropy-augmented rewards as follows
\[
    r_{\kappa,h}(s,a) = r_h(s,a) + \kappa \cH(p_h(s,a)).
\]
This definition is required to cover the following case of practical interest. Let $\kappa = \lambda$ and $\Phi(\pi) = -\cH(\pi)$, then we obtain the following representation for the $\lambda$-regularized value function
\[
    V^{\pi}_{\lambda,1}(s_1) = V^\pi_1(s_1) + \lambda \TE(q^\pi),
\]
where $V^\pi_1(s_1)$ is a usual value function for a MDP $\cM$. For $r_{\max} = 0$  and $\kappa=\lambda=1$ we recover just a trajectory entropy $V^{\pi}_{\lambda,1}(s_1) = \TE(q^\pi)$.

Next we define a convex conjugate to $\lambda \Phi$ as $F_\lambda \colon \R^{\cA} \to \R$
\[
    F_\lambda (x) = \max_{\pi \in \simplex_{\cA}} \{ \langle \pi, x \rangle - \lambda \Phi(\pi) \}
\]
and, with a sight abuse of notation extend the action of this function to the $Q$-function as follows
\[
    \Vstar_{\lambda,h}(s) = F_{\lambda}(\Qstar_{\lambda,h})(s) = \max_{\pi \in \Delta_{\cA}}\left\{ \pi \Qstar_{\lambda,h}(s) - \lambda \Phi(\pi) \right\}.
\]

Thanks to the fact that $\Phi$ satisfies Definition~\ref{def:mirror_map}, we have exact formula for the optimal policy by Fenchel-Legendre transform
\[
    \pi^\star_h(s) = \argmax_{\pi \in \Delta_{\cA}}\left\{ \pi \Qstar_{\lambda,h}(s) - \lambda \Phi(\pi) \right\} = \nabla F_\lambda(\Qstar_{\lambda,h}(s,\cdot)).
\]
Notice that we have $\nabla F_\lambda(\Qstar_{\lambda,h}(s,\cdot)) \in \simplex_{\cA}$ since the gradient of $\Phi$ diverges on the boundary of $\simplex_{\cA}$. For entropy regularization this formula become the softmax function.

Finally, it is known that the smoothness property of $F_{\lambda}$ plays a key role in reduced sample complexity for planning in regularized MDPs \cite{grill2019planning}. For our general setting we have that since $\lambda \Phi$ is $\lambda$-strongly convex with respect to $\norm{\cdot}$, then $F_{\lambda}$ is $1/\lambda$-strongly smooth with respect to a dual norm $\norm{\cdot}_*$
\[
    F_\lambda(x) \leq F_\lambda(x') + \langle \nabla F_\lambda(x'), x-x' \rangle + \frac{1}{2\lambda} \norm{x - x'}_*^2.
\]
Let us define $\Rphi$ as a maximal possible value of $\vert \Phi\vert$. Without loss of generality assume that $\Phi \leq 0$. In this case we define $\Rmax = r_{\max} + \kappa \log(S) + \lambda \Rphi $ as an upper bound of an about of reward obtain at the one step. By this definition we have $0 \leq V^{\pi}_{\lambda,h}(s) \leq H \Rmax$ for any $h\in[H], s \in \cS$ and any policy $\pi$.

Also, since all norms in $\R^\cA$ are equivalent, we define a constant $r_A$ that defined for a dual norm $\norm{\cdot}_*$ as follows
\begin{equation}\label{eq:norm_equivalence}
    \norm{\cdot }_* \leq r_A \cdot \norm{\cdot}_\infty.
\end{equation}
For example, for $\ell_2$-norm $r_A = \sqrt{A}$ and for $\ell_1$-norm $r_A = A$. In the case $\Phi = -\cH$ we have $r_A = 1$ since the entropy is $1$-strongly convex with respect to a $\ell_1$-norm, thus the dual norm is exactly a $\ell_\infty$-norm.

The rest of this section is devoted to obtain the sample complexity guarantees for \UCBVIEnt algorithm with \textit{a regularization-agnostic stopping rule} \eqref{eq:def_tau_agnostic} with the gap notion \eqref{eq:def_gap_agnostic}. In this case Theorem~\ref{th:reg_agnostic_sample_complexity} gives us $\tcO\left(\frac{H^3SA}{\varepsilon^2} + \frac{H^3 S^2A}{\varepsilon}\right)$ sample complexity guarantee ignoring $\Rmax$ and poly-logarithmic factors. Notably, this sample complexity result does depend directly on $1/\lambda$, so small regularization does not affect the complexity.

In Section~\ref{app:fast_rates_regularized} we present another algorithm \RFExploreEnt, based on ideas of reward-free exploration, that achieves sample complexity of order $\tcO(\poly(S,A,H)/(\lambda\varepsilon)$. As a particular example, it yields an algorithm for the MTEE problem with sample complexity $\tcO\left( \poly(S,A,H) / \varepsilon \right)$ by taking as a regularizer negative entropy,  $\lambda = \kappa = 1$ and $r_{\max} = 0$. Moreover, in Section~\ref{app:reg_visitation_entropy_proofs} we apply this algorithm to achieve sample complexity of order $\tcO(H^2 S A /\varepsilon^2)$ for the maximum visitation entropy exploration problem.

\subsection{\UCBVIEnt Algorithm}

In this section we describe \UCBVIEnt algorithm to solve MTEE problem, however as we show in the proofs, it is capable to work with general regularized MDPs.

We now describe our algorithm \algMTEE for MTEE. Since one only needs to solve  regularized Bellman equations to obtain a maximum trajectory entropy policy, we can use an algorithm of the same flavor as the ones proposed for best policy identification \citep{tirinzoni2022optimistic, menard2021fast, kaufmann2020adaptive, dann2019policy}. In particular, \algMTEE is close to the \BPIUCRL algorithm by \citet{kaufmann2020adaptive} and can be characterized by the following rules.

\paragraph{Sampling rule} As sampling rule we use an optimistic policy $\pi^{t+1}$ obtained by optimistic planning in the regularized MDP
\begin{align}
  \uQ_h^t(s,a) &=  \clip\Big(\cH\big(\hp_h(s,a)\big)+b_h^{\cH,t}(s,a) + \hp^t_h \uV_{h+1}^t(s,a)+ b_h^{p,t}(s,a),0,\log(SA)H\Big)\,,\nonumber\\
  \uV_h^t(s) &= \max_{\pi\in\Delta_A}  \pi \uQ_h^t (s) + \cH(\pi)\,,\label{eq:optimistic_planning_TE}\\
  \pi_h^{t+1}(s) &= \argmax_{\pi\in\Delta_A}  \pi \uQ_h^t (s) + \cH(\pi)\nonumber\,,
\end{align}
with $\uV^t_{H+1} = 0$ by convention  where $b^{\cH,t}$, $b^{p,t}$ are bonuses for the entropy  and the transition probabilities, respectively. Precisely, we use bonuses of the form 
\begin{align*}
        b_h^{\cH,t}(s,a) &=  \sqrt{\frac{2 \beta^{\cH}(\delta, n^t_h(s,a))}{n^t_h(s,a)}} + \min\left(\frac{\beta^{\KL}(\delta, n^t_h(s,a))}{n^t_h(s,a)},\, \log(S) \!\!\right)\,,\\
        b^{p,t}_h(s,a) &\triangleq b^{B,t}_h(s,a) + b^{\corr,t}_h(s,a),\\
        b^{B,t}_h(s,a) &\triangleq 3\sqrt{\Var_{\hp^t_h}(\uV^t_{h+1})(s,a) \frac{\beta^{\conc}(\delta, n^t_h(s,a))}{n^t_h(s,a)}} + \frac{9H^2 \log(SA) \beta^{\KL}(\delta, n^t_h(s,a))}{n^t_h(s,a)}, \\
        b^{\corr,t}_h(s,a) &\triangleq \frac{1}{H} \hp^t_h(\uV^t_{h+1} - \lV^t_{h+1})(s,a).
\end{align*}
for some functions $\beta^{\cH}, \beta^{\KL}$ and $\beta^{\conc}$ and $\lV^t$ is a lower confidence bound on the optimal value function defined in Appendix~\ref{app:conf_intervals_ucbvient}.

\paragraph{Stopping and decision rule} To define the stopping rule, we first recursively build an upper-bound on the difference between the value of the optimal policy and the value of the current policy $\pi^{t+1}$, 
\begin{align}\label{eq:upper_bound_gap}
    \begin{split}
    G^t_h(s,a) \triangleq \clip\biggl( 2 b^{B,t}_h(s,a) + 2 b^{\cH, t}_h(s,a) &+ \frac{4 H^2 \log(SA) \beta^{\KL}(\delta, n^t_h(s,a))}{n^t_h(s,a)} \\
    & + \left(1 + \frac{3}{H} \right) \hp^t_h \left[\pi^{t+1}_{h+1} G^t_{h+1}\right](s,a),  0, H\Rmax \biggl),
    \end{split}
\end{align}

where $\lV^t$ is a lower-bound on the optimal value function defined in Appendix~\ref{app:conf_intervals_ucbvient} and $G_{H+1}^t(s,a) = 0$ by convention.
Then the stopping time $\tau= \inf\{ t \in \N : \pi^{t+1} G^t_{1}(s_1)  \leq \varepsilon \}$ corresponds to the first episode when this upper-bound is smaller than $\epsilon$. At this episode we return the Markovian policy $\hpi = \pi^{\tau+1}$.

\begin{algorithm}
\centering
\caption{\algMTEE}
\label{alg:UCBVIEnt}
\begin{algorithmic}[1]
  \STATE {\bfseries Input:} Target precision $\epsilon$, target probability $\delta$, threshold functions $\beta^{\cH},\beta^{p}$, $\beta^{\conc}$.
      \WHILE{ true}
       \STATE Compute $\pi^{t}$ by optimistic planning with \eqref{eq:optimistic_planning_TE}.
       \STATE Compute bound on the gap $G_1^{t-1}(s,a)$ with \eqref{eq:upper_bound_gap}.
       \STATE \textbf{if} $\pi^t G_1^{t-1}(s_1)\leq \epsilon$ \textbf{then break}
      \FOR{$h \in [H]$}
        \STATE Play $a_h^t\sim \pi_h^t(s^t_h)$
        \STATE Observe $s_{h+1}^t\sim p_h(s_h^t,a_h^t)$
      \ENDFOR
    \STATE{ Update counts $n^t$, transition estimates $\hp^t$ and episode number $t\gets t+1$.}
   \ENDWHILE
   \STATE Output policy $\hpi = \pi^t$.
\end{algorithmic}
\end{algorithm}
The complete procedure is described in Algorithm~\ref{alg:UCBVIEnt}. We now state our main theoretical result for \algMTEE. We prove that for the well calibrated threshold functions $\beta^{\cH}, \beta^{\KL}$ and $\beta^{\conc},$ the \algMTEE is  $(\epsilon, \delta)$-PAC for MTEE and  provide a
high-probability upper bound on its sample complexity. 
\begin{theorem}
\label{th:ucbvi_sample_complexiy}
Let $\beta^{\KL}, \beta^{\conc}$ and $\beta^{\cH}$ be defined in Lemma~\ref{lem:proba_master_event} of Appendix~\ref{app:regularized_mdp}. Fix  $\epsilon>0$ and $\delta\in(0,1),$ then the \algMTEE algorithm is $(\epsilon,\delta)$-PAC for MTEE.   Moreover, the optimal policy is given by $\hpi = \pi^{\tau+1}$ where 
\[
\tau = \tcO\left(\frac{H^3SA}{\epsilon^2} + \frac{H^3S^2A}{\varepsilon}\right).
\]
with probability at least $1-\delta.$
Here $\tcO$ hides poly-logarithmic factors in $\epsilon, \delta, H,S,A$.
\end{theorem}

See Theorem~\ref{th:mtee_sample_complexity} in Appendix~\ref{app:regularized_mdp} for a precise bound and a proof. Basically, this result is a simple corollary of the general result on regularized MDPs.

\paragraph{Space and time complexity} Since \UCBVIEnt is a model based algorithm, its space-complexity is of order $\cO(HS^2A)$ whereas its time-complexity for one episode is of order $\cO(HS^2A)$ because of the optimistic planning.

\subsection{Concentration Events}
Following the ideas of \cite{menard2021fast}, we define the following concentration events. 

Let $\beta^{\KL}, \beta^{\conc}, \beta^{\cnt}, \beta^{\cH}: (0,1) \times \N \to \R_{+}$ be some functions defined later on in Lemma \ref{lem:proba_master_event}. We define the following favorable events
\begin{align*}
\cE^{\KL}(\delta) &\triangleq \Bigg\{ \forall t \in \N, \forall h \in [H], \forall (s,a) \in \cS\times\cA: \quad \KL(\hp^{\,t}_h(s,a), p_h(s,a)) \leq \frac{\beta^{\KL}(\delta, n^{\,t}_h(s,a))}{n^{\,t}_h(s,a)} \Bigg\},\\
\cE^{\conc}(\delta) &\triangleq \Bigg\{\forall t \in \N, \forall h \in [H], \forall (s,a)\in\cS\times\cA: \\
&\qquad|(\hp_h^t -p_h) \Vstar_{\lambda,h+1}(s,a)| \leq \sqrt{2 \Var_{p_h}(\Vstar_{\lambda,h+1})(s,a)\frac{\beta^{\conc}(\delta,n_h^t(s,a))}{n_h^t(s,a)}} + 3 H \Rmax \frac{\beta^{\conc}(\delta,n_h^t(s,a))}{n_h^t(s,a)}\Bigg\},\\
\cE^{\cnt}(\delta) &\triangleq \Bigg\{ \forall t \in \N, \forall h \in [H], \forall (s,a) \in \cS\times\cA: \quad n^t_h(s,a) \geq \frac{1}{2} \upn^t_h(s,a) - \beta^{\cnt}(\delta) \Bigg\},\\
\cE^{\cH}(\delta) &\triangleq \Bigg\{\forall t \in \N, \forall h \in [H], \forall (s,a)\in\cS\times\cA: \\
&\qquad  \vert \cH(\hp^t_h(s,a)) - \cH(p_h(s,a)) \vert \leq  \sqrt{\frac{2 \beta^{\cH}(\delta, n^t_h(s,a))}{n^t_h(s,a)}} + \left(\frac{\beta^{\KL}(\delta, n^t_h(s,a))}{n^t_h(s,a)} \wedge \log(S) \right)
\Bigg\}.
\end{align*}
We also introduce two intersections of these events of interest, $\cG(\delta) \triangleq \cE^{\KL}(\delta) \cap \cE^{\conc}_B(\delta) \cap \cE^{\cnt}(\delta) \cap \cE^{\cH}(\delta)$. We  prove that for the right choice of the functions $\beta^{\KL}, \beta^{\conc},\beta^{\cnt}, \beta^{\cH}$ the above events hold with high probability.
\begin{lemma}
\label{lem:proba_master_event}
For any $\delta \in (0,1)$ and for the following choices of functions $\beta,$
\begin{align*}
    \beta^{\KL}(\delta, n) & \triangleq \log(4SAH/\delta) + S\log\left(\rme(1+n) \right), \\
    \beta^{\conc}(\delta, n) &\triangleq \log(4SAH/\delta) + \log(4\rme n(2n+1)) ,\\
    \beta^{\cnt}(\delta) &\triangleq \log(4SAH/\delta), \\
    \beta^{\cH}(\delta,n) &\triangleq \log^2(n)\left(\log(4SAH/\delta) + \log(n(n+1))\right),
\end{align*}
it holds that
\begin{align*}
\P[\cE^{\KL}(\delta)]&\geq 1-\delta/4, \qquad\qquad \P[\cE^{\conc}(\delta)]\geq 1-\delta/4, \\
\P[\cE^\cnt(\delta)]&\geq 1-\delta/4,  \qquad \P[\cE^{\cH}(\delta)]\geq 1-\delta/4.
\end{align*}
In particular, $\P[\cG(\delta)] \geq 1-\delta$.
\end{lemma}
\begin{proof}
Applying Theorem~\ref{th:max_ineq_categorical} and the union bound over $h \in [H], (s,a) \in \cS \times \cA$ we get $\P[\cE^{\KL}(\delta)]\geq 1-\delta/4$.  Next, Theorem~\ref{th:bernstein} and the union bound over $h \in [H], (s,a) \in \cS \times \cA$ yield $\P[\cE^{\conc}(\delta)]\geq 1 - \delta/4$. By Theorem~\ref{th:bernoulli-deviation} and union bound,  $\P[\cE^{\cnt}(\delta)]\geq 1 - \delta/4$.  Finally, by Theorem~\ref{th:entropy_concentration} and union bound over $(s,a,h) \in \cS \times \cA \times [H]$ $\P[\cE^{\cH}(\delta)]\geq 1-\delta/4$. The union bound over four prescribed events concludes $\P[\cG(\delta)] \geq 1 - \delta$.
\end{proof}

\begin{lemma}\label{lem:reg_directional_concentration}
      Assume conditions of Lemma \ref{lem:proba_master_event}. Then on event $\cE^{\KL}(\delta)$, for any $f \colon \cS \to [0, H \Rmax]$, $t \in \N, h \in [H], (s,a) \in \cS \times \cA$,
      \begin{align*}
            [p_h - \hp_h^t]f(s,a) &\leq \frac{1}{H} \hp^t_h f(s,a) + 2H\Rmax \left(\frac{2H \beta^{\KL}(\delta, n^{\,t}_h(s,a))}{n^{\,t}_h(s,a)} \wedge 1 \right), \\
            [\hp_h^t -p_h]f(s,a) &\leq \frac{1}{H} p_h f(s,a) + 2H\Rmax \left(\frac{2H \beta^{\KL}(\delta, n^{\,t}_h(s,a))}{n^{\,t}_h(s,a)} \wedge 1 \right). 
      \end{align*}
\end{lemma}
\begin{proof}
    Let us start from the first statement.  We apply the second inequality of Lemma~\ref{lem:Bernstein_via_kl} and Lemma~\ref{lem:switch_variance_bis} to obtain
    \begin{align*}
        [p_h - \hp_h^t]f(s,a) &\leq \sqrt{2\Var_{p_h}[f](s,a) \cdot \KL(\hp_h^t, p_h) } \\
        &\leq 2\sqrt{\Var_{\hp^t_h}[f](s,a) \cdot \KL(\hp_h^t, p_h) } +  3 H \Rmax \KL(\hp^t_h, p_h).
    \end{align*}
    Since $0 \leq f(s) \leq  H\Rmax$ we get
    \[
        \Var_{\hp^t_h}[f](s,a) \leq \hp^t_h[f^2](s,a) \leq  H\Rmax \cdot \hp^t_h f(s,a).
    \]
    Finally, applying $2\sqrt{ab} \leq a+b, a, b \geq 0$, we obtain the following inequality
    \begin{align*}
        (\hp_h^t -p_h)f(s,a) &\leq \frac{1}{H} \hp^t_h f(s,a) + 4H^2\Rmax \KL(\hp_h^t, p_h).
    \end{align*}
    Definition of $\cE^{\KL}(\delta)$ implies the part of the statement. At the same time we have a trivial bound since $f(s) \in [0, H \Rmax]$
    \[
        [p_h - \hp^t_h] f(s,a) \leq 2H\Rmax \leq \frac{1}{H} \hp^t_h f(s,a) + 2H\Rmax.
    \]

    To prove the second statement, apply the first inequality of Lemma~\ref{lem:Bernstein_via_kl} and proceed similarly.
\end{proof}

\begin{lemma}\label{lem:empirical_bernstein}
    Assume conditions of Lemma \ref{lem:proba_master_event} and assume that $\beta^{\conc}(\delta) \leq \beta^{\KL}(\delta)$. Then conditioned on event $\cG(\delta)$, for any $U \colon \cS \to [0, H \Rmax]$, $t \in \N, h \in [H], (s,a) \in \cS \times \cA$,
      \begin{align*}
            \vert (\hp_h^t -p_h)\Vstar_{\lambda,h+1}(s,a) \vert &\leq 3 \sqrt{\Var_{\hp^t_{h+1}}(U)(s,a) \frac{\beta^{\conc}(\delta,n_h^t(s,a))}{n_h^t(s,a)}} + \frac{9H^2\Rmax \beta^{\KL}(\delta, n^t_h(s,a))}{n^t_h(s,a)} \\
            &+ \frac{1}{H} \hp^t_h \vert U - \Vstar_{\lambda,h+1} \vert(s,a).
      \end{align*}
\end{lemma}
\begin{proof}
    First, we apply the definition of event $\cE^{\conc}(\delta)$
    \[
        \vert (\hp_h^t -p_h)\Vstar_{\lambda,h+1}(s,a) \vert \leq \sqrt{2 \Var_{p_h}(\Vstar_{\lambda,h+1})(s,a)\frac{\beta^{\conc}(\delta,n_h^t(s,a))}{n_h^t(s,a)}} + 3 H \Rmax \frac{\beta^{\conc}(\delta,n_h^t(s,a))}{n_h^t(s,a)}.
    \]
    Next we apply Lemma~\ref{lem:switch_variance_bis} and Lemma~\ref{lem:switch_variance} and obtain
    \begin{align*}
        \Var_{p_h}(\Vstar_{\lambda,h+1})(s,a) &\leq 2 \Var_{\hp^t_h}(\Vstar_{\lambda,h+1})(s,a) + 4H^2 \Rmax^2 \KL(\hp^t_h(s,a), p_h(s,a)) \\
        &\leq 4\Var_{\hp^t_{h+1}}(U)(s,a) + 4H\Rmax \hp^t_h \vert U - \Vstar_{\lambda,h+1} \vert(s,a) +  4H^2 \Rmax^2 \KL(\hp^t_h(s,a), p_h(s,a)).
    \end{align*}
    Thus, by inequality $\sqrt{a+b} \leq \sqrt{a} + \sqrt{b}$.
    \begin{align*}
        \vert (\hp_h^t -p_h)\Vstar_{\lambda,h+1}(s,a) \vert &\leq 3 \sqrt{\Var_{\hp^t_{h+1}}(U)(s,a) \frac{\beta^{\conc}(\delta,n_h^t(s,a))}{n_h^t(s,a)}} \\
        &+ 3 \sqrt{H\Rmax \hp^t_h \vert U - \Vstar_{\lambda,h+1}\vert(s,a) \cdot \frac{\beta^{\conc}(\delta,n_h^t(s,a))}{n_h^t(s,a)}} \\
        &+ 3H\Rmax\sqrt{\KL(\hp^t_h(s,a), p_h(s,a)) \cdot \frac{\beta^{\conc}(\delta,n_h^t(s,a))}{n_h^t(s,a)}} \\
        &+ 3H\Rmax \frac{\beta^{\conc}(\delta,n_h^t(s,a))}{n_h^t(s,a)}.
    \end{align*}
    By inequality $2\sqrt{ab} \leq a+b$ we have
    \[
        3 \sqrt{H\Rmax \hp^t_h \vert U - \Vstar_{\lambda,h+1}\vert(s,a) \cdot \frac{\beta^{\conc}(\delta,n_h^t(s,a))}{n_h^t(s,a)}} \leq \frac{1}{H}\hp^t_h \vert U - \Vstar_{\lambda,h+1}\vert(s,a) + \frac{9 H^2 \Rmax\beta^{\conc}(\delta,n_h^t(s,a))}{4 n^t_h(s,a)}.
    \]
    By the definition of the event $\cE^{\KL}(\delta)$ and the fact $\beta^{\conc }(\delta) \leq \beta^{\KL}(\delta)$ we have
    \[
        \sqrt{\KL(\hp^t_h(s,a), p_h(s,a)) \cdot \frac{\beta^{\conc}(\delta,n_h^t(s,a))}{n_h^t(s,a)}} \leq \frac{\beta^{\KL}(\delta, n^t_h(s,a))}{n^t_h(s,a)}.
    \]
\end{proof}

\subsection{Confidence Intervals}\label{app:conf_intervals_ucbvient}

Similar to \citet{azar2017minimax,zanette2019tighter,menard2021fast}, we define the upper confidence bound for the optimal regularized  Q-function of two types: with Hoeffding bonuses and with Bernstein bonuses. 

Let us define empirical estimate of entropy-augmented rewards as follows
\[
    \hat r^t_{\kappa,h}(s,a) = r_h(s,a) + \kappa \cH(\hp^t_h(s,a)).
\]
Then we have the following sequences defined as follows
\begin{align*}
    \uQ^{t}_{h}(s,a) &= \clip\left( \hat r^t_{\kappa,h}(s,a) + \hp^t_h \uV^t_h(s,a) + b^{p,t}_h(s,a) + \kappa b^{\cH, t}_h(s,a),0,H\Rmax \right) \\
    \pi^{t+1}_h(s) &= \max_{\pi \in \simplex_\cA} \{ \pi \uQ^t_h(s) -\lambda \Phi(\pi)\}, \\
    \uV^t_h(s) &= \cH(\pi^{t+1}_h(s)) + \pi^{t+1}_h \uQ^t_h(s)  \\
    \uV^t_{H+1}(s) &= 0,
\end{align*}
and the lower confidence bound as follows
\begin{align*}
    \lQ^{t}_{h}(s,a) &= \clip\left( \hat r^t_{\kappa,h}(s,a) + \hp^t_h \lV^t_h(s,a) - b^{p,t}_h(s,a) - \kappa b^{\cH, t}_h(s,a),0,H \Rmax\right) \\
    \lV^t_h(s) &= \max_{\pi \in \simplex_\cA }\{ \pi \lQ^t_h(s) - \lambda\Phi(\pi)  \} \\
    \lV^t_{H+1}(s) &= 0,
\end{align*}
where the Bernstein bonuses for transitions are defined as follows
\begin{align}\label{eq:bernstein_transition_bonuses}
    \begin{split}
        b^{p,t}_h(s,a) &\triangleq b^{B,t}_h(s,a) + b^{\corr,t}_h(s,a),\\
        b^{B,t}_h(s,a) &\triangleq 3\sqrt{\Var_{\hp^t_h}(\uV^t_{h+1})(s,a) \frac{\beta^{\conc}(\delta, n^t_h(s,a))}{n^t_h(s,a)}} + \frac{9H^2 \Rmax \beta^{\KL}(\delta, n^t_h(s,a))}{n^t_h(s,a)}, \\
        b^{\corr,t}_h(s,a) &\triangleq \frac{1}{H} \hp^t_h(\uV^t_{h+1} - \lV^t_{h+1})(s,a).
    \end{split}
\end{align}
The entropy bonuses are defined bellow
\begin{align}\label{eq:entropy_bonuses}
    b^{\cH, t}_h(s,a) &\triangleq \sqrt{\frac{2 \beta^{\cH}(\delta, n^t_h(s,a))}{n^t_h(s,a)}} + \left(\frac{\beta^{\KL}(\delta, n^t_h(s,a))}{n^t_h(s,a)} \wedge \log(S) \right).
\end{align}

\begin{theorem}\label{th:reg_confidence_intervals}
    Let $\delta \in (0,1)$. Assume Bernstein bonuses \eqref{eq:bernstein_transition_bonuses}. Then on event $\cG(\delta)$ for any $t \in \N$, $(h,s,a) \in [H]\times \cS \times \cA$ it holds
    \[
        \lQ^t_h(s,a) \leq \Qstar_{\lambda,h}(s,a) \leq \uQ^t_h(s,a), \qquad \lV^t_h(s) \leq \Vstar_{\lambda,h}(s) \leq \uV^{t}_h(s,a).
    \]
\end{theorem}
\begin{proof}
    Proceed by induction over $h$. For $h = H+1$ the statement is trivial. Now we assume that inequality holds for any $h' > h$ for a fixed $h \in [H]$. Fix a timestamp $t \in \N$ and a state-action pair $(s,a)$ and assume that $\uQ^t_h(s,a) < H\Rmax$, i.e. no clipping occurs. Otherwise the inequality $\Qstar_{\lambda,h}(s,a) \leq \uQ^t_h(s,a)$ is trivial. In particular, it implies $n^t_h(s,a) > 0$.

    In this case, by Bellman equations \eqref{eq:opt_reg_bellman_equation} we have
    \begin{align*}
        [\uQ^t_h - \Qstar_{\lambda,h}](s,a) &= \underbrace{r_h(s,a) + \kappa \cH(\hp^t_h(s,a)) - r_h(s,a) - \kappa \cH(p_h(s,a)) + \kappa b^{\cH,t}_h(s,a)}_{T_1} \\
        &+ \underbrace{\hp^t_h \uV^t_{h+1}(s,a) - p_h \Vstar_{\lambda,h+1}(s,a) + b^{p,t}_h(s,a)}_{T_2}.
    \end{align*}
    By the definition of event $\cE^{\cH}(\delta)$ that is subset of $\cG(\delta)$ we have $T_1 \geq 0$. 
    To show that $T_2 \geq 0$, we start from induction hypothesis
    \begin{align*}
        T_2 &\geq [\hp^t_h - p_h] \Vstar_{\lambda,h+1}(s,a) + b^{p,t}_h(s,a).
    \end{align*}
    Next we apply Lemma~\ref{lem:empirical_bernstein} with $U = \uV^t_{h+1}$ and definition of transition bonuses we have 
    \[
        T_2 \geq - \frac{1}{H} \hp^t_h \vert \uV^t_{h+1} - \Vstar_{\lambda, h+1} \vert (s,a) + \frac{1}{H} \hp^t_h [\uV^t_{h+1} - \lV^t_{h+1}](s,a).
    \]
    By induction hypothesis we have $\uV^t_{h+1}(s) \geq  \Vstar_{\lambda, h+1}(s) \geq \lV^t_{h+1}(s)$, thus $T_2 \geq 0$.

    To prove the second inequality on $Q$-function, we assume $\lQ^t_h(s,a) > 0$ and, as a consequence, $n^t_h(s,a) > 0$. Thus we have
    \begin{align*}
        [\lQ^t_h - \Qstar_{\lambda,h}](s,a) = \underbrace{\cH(\hp^t_h(s,a)) - \cH(p_h(s,a)) - b^{\cH,t}_h(s,a)}_{T_1'} + \underbrace{\hp^t_h \lV^t_{h+1}(s,a) - p_h \Vstar_{\lambda, h+1}(s,a) - b^{p,t}_h(s,a)}_{T_2'}.
    \end{align*}
    Again, by the definition of event $\cE^{\cH}(\delta)$ we have $T_1' \leq 0$ and, by induction hypothesis
    \[
        T_2' \leq [\hp^t_h - p_h] V^{\cH,\star}_{h+1}(s,a) - b^{p,t}_h(s,a).
    \]
    We again apply Lemma~\ref{lem:empirical_bernstein} with $U = \uV^t_{h+1}$
    \[
        T_2' \leq \frac{1}{H} \hp^t_h \vert \uV^t_{h+1} - \Vstar_{\lambda,h+1} \vert (s,a) - \frac{1}{H} \hp^t_h[\uV^t_{h+1} - \lV^t_{h+1}](s,a).
    \]
    We conclude the statement by induction hypothesis for $h' = h+1$.

    Finally, we have to show the inequality for $V$-functions. To do it, we use the fact that $V$-functions are computed by $F_{\lambda}$ applied to $Q$-functions
    \[
        \lV^t_h(s) = F_{\lambda}(\lQ^t_h)(s), \quad \Vstar_{\lambda,h}(s) = F_{\lambda}(\Qstar_{\lambda,h})(s), \quad \uV^t_h(s) = F_{\lambda}(\uQ^t_h)(s).
    \] 
    Notice that $\nabla F_\lambda$ takes values in a probability simplex, thus, all partial derivatives of $F_\lambda$ are non-negative and therefore $F_\lambda$ is monotone in each coordinate. Thus, since $\lQ^t_h(s,a) \leq \Qstar_{\lambda,h}(s,a) \leq \uQ^t_h(s,a)$, we have the same inequality $\lV^t_h(s) \leq \Vstar_{\lambda,h}(s) \leq \uV^t_h(s)$.
\end{proof}

\subsection{Regularization-Agnostic Stopping Rule}

In this section we provide guarantees for the so-called \textit{regularization-agnostic gap}: this notion of gap does not influenced by regularization except the changing of the range of value functions and basically mimics \UCBVIBPI algorithm by \citet{menard2021fast} in definition of the similar gap.

Let us define $G^t_{H+1}(s,a) \triangleq 0$ for all $s,a$ and
\begin{align}\label{eq:def_gap_agnostic}
    \begin{split}
    G^t_h(s,a) &\triangleq \clip\biggl( 2 b^{B,t}_h(s,a) + \frac{4 H^2\Rmax \beta^{\KL}(\delta, n^t_h(s,a))}{n^t_h(s,a)} + 2 \kappa b^{\cH, t}_h(s,a) + \left(1 + \frac{3}{H} \right) \hp^t_h \left[\pi^{t+1}_{h+1} G^t_{h+1}\right](s,a), \\
    &\qquad 0, H\Rmax \biggl),
    \end{split}
\end{align}
where $b^{B,t}_h(s,a)$ is defined in \eqref{eq:bernstein_transition_bonuses}. For this notion of the gap we can define the stopping rule as follows
\begin{align}\label{eq:def_tau_agnostic}
    \tau = \min\{ t \in \N : \pi^{t+1}_1 G^t_1(s_1) \leq \varepsilon \}.
\end{align}

The lemma below justifies this choice of the stopping time.
\begin{lemma}\label{lem:reg_agnostic_stopping_rule}
    Assume the choice of Bernstein bonuses \eqref{eq:bernstein_transition_bonuses} and let the event $\cG(\delta)$ defined in Lemma~\ref{lem:proba_master_event} holds. Then for all $t \in \N$ we have
    \[
        \Vstar_{\lambda,1}(s_1) - V^{\pi^{t+1}}_{\lambda,1}(s_1) \leq \pi^{t+1}_1 G^t_1(s_1).
    \]
\end{lemma}
\begin{proof}
    Following \cite{menard2021fast}, we start by defining the following quantities
    \begin{align*}
        \tQ^t_h(s,a) &\triangleq \clip\left( \hat r^t_{\kappa,h}(s,a) + \hp^t_h \tV^t_{h+1}(s,a) - b^{p,t}_h(s,a) - \kappa b^{\cH,t}_h(s,a) , 0, r_{\kappa,h}(s,a) + p_h \tV^t_{h+1}(s,a) \right),  \\
        \tV^t_h(s) &\triangleq  \pi^{t+1}_h \tQ^t_h(s) - \lambda \Phi(\pi^{t+1}_h(s)), \\
        \tV^t_{H+1}(s) &\triangleq 0.
    \end{align*}
    By Theorem~\ref{th:reg_confidence_intervals} and Lemma~\ref{lem:tQ_properties} we have
    \begin{align*}
        \Vstar_{\lambda,1}(s_1) - V^{\pi^{t+1}}_{\lambda,1}(s_1) &\leq \uV^{t}_1(s_1) - V^{\pi^{t+1}}_{\lambda,1}(s_1) \leq \uV^{t}_1(s_1) - \tV^t_1(s_1) \\
        &=  \pi^{t+1}_1 \uQ^t_1(s_1) - \lambda \Phi(\pi^{t+1}_1(s_1))- \pi^{t+1}_1 \tQ^t_1(s_1) + \lambda \Phi(\pi^{t+1}_1(s_1)) = \pi^{t+1}_1 [\uQ^t_1 - \tQ^t_1](s_1).
    \end{align*}
    Therefore, it is enough to show that for any $(h,s,a) \in [H] \times \cS \times \cA$
    \[
        [\uQ^t_h - \tQ^t_h](s,a) \leq G^t_h(s,a), \qquad  [\uV^t_h - \tV^t_h](s) \leq \pi^{t+1}_h G^t_h(s).
    \]
    Proceed by backward induction over $h$. The case $h=H+1$ is trivial, thus we may assume that the statement holds for any $h' > h$ for a fixed $h$. Also fix $(s,a) \in \cS \times \cA$.

    Notice that if $G^t_h(s,a) = H\Rmax$, then the inequality is trivially true. Therefore we may assume that $G^t_h(s,a) < H\Rmax$ and, consequently, $n^t_h(s,a) > 0$. Now we have to separate cases.
    \paragraph{First case.}
    In this case we have $\tQ^t_h(s,a) = r_{\kappa,h}(s,a) + p_h \tV^t_{h+1}(s,a)$, i.e. maximal clipping occurs. Therefore
    \begin{align*}
        \uQ^t_h(s,a) - \tQ^t_h(s,a) &= r_h(s,a) + \kappa \cH(\hp^t_h(s,a)) + \kappa b^{\cH,t}_h(s,a) - r_h(s,a) - \kappa\cH(p_h(s,a)) \\
        &+ \hp^t_h \uV^t_{h+1}(s,a) - p_h \tV^t_{h+1}(s,a) + b^{p,t}_h(s,a).
    \end{align*}
    By the definition of the event $\cE^{\cH}(\delta) \subseteq \cG(\delta)$ we have
    \[
        \kappa \cH(\hp^t_h(s,a)) + \kappa b^{\cH,t}_h(s,a)  - \kappa\cH(p_h(s,a)) \leq 2\kappa b^{\cH,t}_h(s,a),
    \]
    for the next term we have
    \[
        \hp^t_h \uV^t_{h+1} - p_h \tV^t_{h+1}(s,a) = \hp^t_h [\uV^t_{h+1} - \tV^t_{h+1}(s,a)] + [\hp^t_h - p_h] \Vstar_{\lambda,h+1} (s,a) + [p_h - \hp^t_h][ \Vstar_{\lambda,h+1}- \tV^t_{h+1} ](s,a).
    \]
    By induction hypothesis 
    \[
        \hp^t_h [\uV^t_{h+1} - \tV^t_{h+1}(s,a)] \leq \hp^t_h [ \pi^{t+1}_{h+1} G^t_{h+1}](s,a).
    \]
    Next we apply Lemma~\ref{lem:empirical_bernstein} with $U = \uV^t_{h+1}(s,a)$ and Theorem~\ref{th:reg_confidence_intervals}
    \[
        [\hp^t_h - p_h] \Vstar_{\lambda,h+1}(s,a) \leq b^{p,t}_h(s,a).
    \]
    Finally, we apply Lemma~\ref{lem:reg_directional_concentration} and obtain
    \[
        [p_h - \hp^t_h][\Vstar_{\lambda,h+1}- \tV^t_{h+1} ](s,a) \leq \frac{1}{H} \hp^t_h[\Vstar_{\lambda,h+1} - \tV^t_{h+1} ](s,a) + \frac{4 H^2 \Rmax \cdot \beta^{\KL}(\delta,n^t_h(s,a))}{n^t_h(s,a)}.
    \]
    Summing all these bounds up, we have
    \begin{align*}
        \uQ^t_h(s,a) - \tQ^t_h(s,a) &\leq \hp^t_h [ \pi^{t+1}_{h+1} G^t_{h+1}](s,a) + 2\kappa b^{\cH,t}_h(s,a) + 2b^{p,t}_h(s,a) \\
        &+ \frac{1}{H} \hp^t_h[ \Vstar_{\lambda, h+1}- \tV^t_{h+1} ](s,a) + \frac{4 H^2 \Rmax \cdot \beta^{\KL}(\delta,n^t_h(s,a))}{n^t_h(s,a)}.
    \end{align*}
    Notice that by Theorem~\ref{th:reg_confidence_intervals} and the induction hypothesis 
    \[
        \frac{1}{H} \hp^t_h[\Vstar_{\lambda, h+1}- \tV^t_{h+1} ](s,a) \leq \frac{1}{H} \hp^t_h[\uV^{t}_{h+1}- \tV^t_{h+1} ](s,a) \leq  \frac{1}{H} \hp^t_h [ \pi^{t+1}_{h+1} G^t_{h+1}](s,a),
    \]
    and by decomposing the transition bonus \eqref{eq:bernstein_transition_bonuses} to Bernstein bonus and correction term and applying Lemma~\ref{lem:tQ_properties}
    \begin{align*}
        b^{p,t}_h(s,a) &= b^{B,t}_h(s,a) + \frac{1}{H} \hp^t_h[\uV^t_{h+1} - \lV^t_{h+1}](s,a) \leq  b^{B,t}_h(s,a) + \frac{1}{H} \hp^t_h[\uV^{t}_{h+1}- \tV^t_{h+1} ](s,a) \\
        &\leq  b^{B,t}_h(s,a) + \frac{1}{H}\hp^t_h [ \pi^{t+1}_{h+1} G^t_{h+1}](s,a),
    \end{align*}
    thus
    \begin{align*}
        \uQ^t_h(s,a) - \tQ^t_h(s,a) &\leq \left(1 + \frac{3}{H} \right)\hp^t_h [ \pi^{t+1}_{h+1} G^t_{h+1}](s,a) + 2 \kappa b^{\cH,t}_h(s,a) + 2b^{B,t}_h(s,a) \\
        &+ \frac{4 H^2 \Rmax \cdot \beta^{\KL}(\delta,n^t_h(s,a))}{n^t_h(s,a)} = G^t_h(s,a).
    \end{align*}
    
    \paragraph{Second case.} In this case we have $\tQ^t_h(s,a) = \hat r^t_{\lambda,h}(s,a) + \hp^t_h \tV^t_{h+1}(s,a) - b^{p,t}_h(s,a) - \kappa b^{\cH,t}_h(s,a)$. Thus

    \[
        \uQ^t_h(s,a) - \tQ^t_h(s,a) \leq 2 \kappa b^{\cH,t}_h(s,a) + 2 b^{B,t}_h(s,a) + \hp^t_h[\uV^{t}_{h+1} - \tV^t_{h+1}](s,a) + \frac{2}{H}\hp^t_h[\uV^{t}_{h+1} - \lV^t_{h+1}](s,a).
    \]
    By Lemma~\ref{lem:tQ_properties} and induction hypothesis we have 
    \[
        \uQ^t_h(s,a) - \tQ^t_h(s,a) \leq 2\kappa b^{\cH,t}_h(s,a) + 2 b^{B,t}_h(s,a) + \left(1 + \frac{2}{H} \right)\hp^t_h [\pi^{t+1}_{h+1} G^t_{h+1}](s,a)  \leq G^t_h(s,a).
    \]

    \paragraph{Conclusion.} From the two cases above we conclude
    \[
        [\uQ^t_h - \tQ^t_h](s,a) \leq G^t_h(s,a).
    \]
    Moreover, we have
    \[
        \uV^t_h(s) - \tV^t_h(s) = \pi^{t+1}_h \uQ^t_h(s) -\lambda\Phi(\pi^{t+1}_h(s))  \pi^{t+1}_h \tQ^t_h(s) + \lambda\Phi(\pi^{t+1}_h(s))= \pi^{t+1}_h [\uQ^t_h - \tQ^t_h](s) \leq \pi^{t+1}_h G^t_h(s).
    \]
    The last inequality concludes the statement of Lemma~\ref{lem:reg_agnostic_stopping_rule}.
\end{proof}

\begin{lemma}\label{lem:tQ_properties}
    Under the choice of Bernstein bonuses \eqref{eq:bernstein_transition_bonuses}, on event $\cG(\delta)$ for any $t \in \N$ and any $(h,s,a) \in [H] \times \cS \times \cA$
    \[
        \tQ^t_h(s,a) \leq \min\{ Q^{\pi^{t+1}}_{\lambda,h}(s,a), \lQ^t_h(s,a) \}, \qquad
        \tV^t_h(s) \leq \min\{ V^{\pi^{t+1}}_{\lambda,h}(s), \lV^t_h(s) \}.
    \]
\end{lemma}
\begin{proof}
    Proceed by backward induction over $h$. The case $h=H+1$ is trivially true, assume that the statement holds for any $h' > h$ for a fixed $h$. Also let us fix $t,s,a$.  By induction hypothesis we have
    \[
        \tQ^t_h(s,a) \leq r_{\kappa,h}(s,a) + p_h \tV^t_{h+1} \leq r_{\kappa,h}(s,a) + p_h V^{\pi^{t+1}}_{\lambda, h+1}(s,a) = Q^{\pi^{t+1}}_{\lambda,h}(s,a).
    \]
    In the same manner
    \begin{align*}
        \tQ^t_h(s,a) &\leq \hat r^t_{\kappa,h}(s,a) + \hp^t_h \tV^t_{h+1} - b^{p,t}_h(s,a) - \kappa b^{\cH,t}_h(s,a) \\
        &\leq \hat r^t_{\kappa,h}(s,a) + \hp^t_h \lV^t_{h+1} - b^{p,t}_h(s,a) - \kappa b^{\cH,t}_h(s,a) = \lQ^t_h(s,a).
    \end{align*}
    Next, we prove the same inequalities for $V$-functions
    \[
        \tV^t_h(s) = \pi^{t+1}_h(s) \tQ^t_h(s,a) - \lambda \Phi(\pi^{t+1}_h(s))\leq  \pi^{t+1}_h(s) Q^{\pi^{t+1}}_{\lambda,h}(s,a) -\lambda \Phi(\pi^{t+1}_h(s)) = V^{\pi^{t+1}}_{\lambda,h}(s),
    \]
    and
    \begin{align*}
        \tV^t_h(s) = \pi^{t+1}_h \tQ^t_h(s)  -\lambda \Phi(\pi^{t+1}_h(s)) \leq  \pi^{t+1}_h \lQ^{t}_h(s) -\lambda \Phi(\pi^{t+1}_h(s)) \leq \max_{\pi \in \simplex_{\cA}}\left\{ \pi \lQ^{t}_h(s) -\lambda \Phi(\pi)\right\} = \lV^t_h(s).
    \end{align*}
\end{proof}

After defining a proper quantity for a stopping rule we may proceed with the final proof for sample-complexity of the presented algorithm \UCBVIEnt.

\begin{theorem}\label{th:reg_agnostic_sample_complexity}
    Let $\delta \in (0,1)$. Then $\UCBVIEnt$ algorithm with Bernstein bonuses \eqref{eq:bernstein_transition_bonuses} and a regularization-agnostic stopping rule $\tau$ is $(\varepsilon,\delta)$-PAC for the best policy identification in regularized MDPs. 
    
    Moreover, with probability at least $1-\delta$ the stopping time $\tau$ is bounded as follows
    \[
        \tau = \cO\left( \frac{H^3SA \Rmax^2 \log(SAH/\delta) L^4}{\varepsilon^2} + \frac{H^3SA (\log(SAH/\delta) + SL) \cdot L}{\varepsilon} \right),
    \]
    where $L = \cO(\log(SAH\Rmax/\varepsilon)) + \log\log(SAH/\delta))$.
\end{theorem}
\begin{proof}
    Notice that if $\tau = 0$, then our sample complexity bound is trivial, thus we assume that $\tau > 0$.
    Let us start from deriving an upper bound for $G^t_h(s,a)$ for $t < \tau, h \in [H], (s,a) \in \cS \times \cA$. 
    \begin{align*}
        G^t_h(s,a) &\leq 2 b^{B,t}_h(s,a) + 2 \kappa b^{\cH,t}_h(s,a) + \frac{4 H^2 \Rmax \beta^{\KL}(\delta, n^t_h(s,a))}{n^t_h(s,a)} + \left(1 + \frac{3}{H} \right) \hp^t_h [\pi^{t+1}_{h+1} G^t_{h+1}](s,a) \\
        &\leq 6\sqrt{\Var_{\hp^t_h}[\uV^t_{h+1}](s,a) \frac{\beta^{\conc}(\delta, n^t_h(s,a))}{n^t_h(s,a)}} + 2\kappa \sqrt{\frac{2\beta^{\cH}(\delta, n^t_h(s,a))}{n^t_h(s,a)}} + \frac{23 H^2 \Rmax \beta^{\KL}(\delta, n^t_h(s,a))}{n^t_h(s,a)}  \\
        &+ \left(1 + \frac{3}{H} \right)[\hp^t_h - p_h] [\pi^{t+1}_{h+1} G^t_{h+1}](s,a) + \left(1 + \frac{3}{H} \right)p_h [\pi^{t+1}_{h+1} G^t_{h+1}](s,a).
    \end{align*}
    By Lemma~\ref{lem:reg_directional_concentration}
    \[
        [\hp^t_h - p_h] [\pi^{t+1}_{h+1} G^t_{h+1}](s,a) \leq \frac{1}{H} p_h [\pi^{t+1}_{h+1} G^t_{h+1}](s,a) + \frac{4 H^2 \Rmax \beta^{\KL}(\delta, n^t_h(s,a))}{n^t_h(s,a)}.
    \]
    Also we have to replace the variance of the empirical model with the real variance of the value function for $\pi^{t+1}$ in order to apply the law of total variance (Lemma~\ref{lem:law_of_total_variance}). 

    Apply Lemma~\ref{lem:switch_variance} and Lemma~\ref{lem:switch_variance_bis}
    \begin{align*}
        \Var_{\hp^t_h}[\uV^t_{h+1}](s,a) &\leq 2\Var_{p_h}[\uV^t_{h+1}](s,a) + \frac{4H^2\Rmax^2\beta^{\KL}(\delta,n^t_h(s,a))}{n^t_h(s,a)} \\
        &\leq 4 \Var_{p_h}[V^{ \pi^{t+1}}_{\lambda, h+1}](s,a) + 2H\Rmax p_h [ \uV^t_{h+1} - V^{\pi^{t+1}}_{\lambda, h+1}](s,a) + \frac{4H^2\Rmax^2\beta^{\KL}(\delta,n^t_h(s,a))}{n^t_h(s,a)} .
    \end{align*}
    In the proof of Lemma~\ref{lem:reg_agnostic_stopping_rule} it was proven that
    \[
        [ \uV^t_{h+1} - V^{\pi^{t+1}}_{\lambda, h+1}](s) \leq \pi^{t+1}_{h+1} G^t_{h+1}(s),
    \]
    thus, combining with an inequality $\sqrt{a+b} \leq \sqrt{a} + \sqrt{b}$
    \begin{align*}
        6\sqrt{\Var_{\hp^t_h}[\uV^t_{h+1}](s,a) \frac{\beta^{\conc}(\delta, n^t_h(s,a))}{n^t_h(s,a)}} &\leq 12 \sqrt{\Var_{p_h}[V^{\pi^{t+1}}_{\lambda,h+1}](s,a) \frac{\beta^{\conc}(\delta, n^t_h(s,a))}{n^t_h(s,a)}} \\
        &+ 6\sqrt{ p_h[\pi^{t+1}_{h+1} G^t_{h+1}](s,a) \frac{2H\Rmax\beta^{\conc}(\delta,n^t_h(s,a))}{n^t_h(s,a)}}\\
        &+\frac{12 H \Rmax \beta^{\KL}(\delta, n^t_h(s,a)}{n^t_h(s,a)} 
    \end{align*}
    To bound the second term, we use inequality $2\sqrt{ab} \leq a + b$
    \[
        6\sqrt{ p_h[\pi^{t+1}_{h+1} G^t_{h+1}](s,a) \frac{2H\Rmax\beta^{\conc}(\delta,n^t_h(s,a))}{n^t_h(s,a)}} \leq \frac{3}{H}p_h[\pi^{t+1}_{h+1} G^t_{h+1}](s,a) + \frac{3H^2\Rmax\beta^{\conc}(\delta,n^t_h(s,a))}{n^t_h(s,a)}.
    \]
    Finally, we have the following bound on $G^t_{h}(s,a)$
    \begin{align*}
        G^t_h(s,a) &\leq 12 \sqrt{\Var_{p_h}[V^{\pi^{t+1}}_{\lambda, h+1}](s,a) \frac{\beta^{\conc}(\delta, n^t_h(s,a))}{n^t_h(s,a)}} + 2\kappa\sqrt{\frac{2\beta^{\cH}(\delta, n^t_h(s,a))}{n^t_h(s,a)}} \\
        &+ \frac{54 H^2 \Rmax \beta^{\KL}(\delta, n^t_h(s,a))}{n^t_h(s,a)} + \left(1 + \frac{10}{H}\right) p_h[\pi^{t+1}_{h+1} G^t_{h+1}](s,a).
    \end{align*}
    Notice that his inequality could be rewritten in the following form
    \begin{align*}
        G^t_h(s,a) &\leq \E_{\pi^{t+1}}\biggl[12 \sqrt{\Var_{p_h}[V^{\pi^{t+1}}_{\lambda,h+1}](s,a) \frac{\beta^{\conc}(\delta, n^t_h(s,a))}{n^t_h(s,a)}} + 2\kappa\sqrt{\frac{2\beta^{\cH}(\delta, n^t_h(s,a))}{n^t_h(s,a)}} \\
        &+ \frac{54 H^2 \Rmax \beta^{\KL}(\delta, n^t_h(s,a))}{n^t_h(s,a)} + \left(1 + \frac{10}{H} \right) G^t_{h+1}(s_{h+1},a_{h+1}) \biggl| (s_h,a_h) = (s,a) \biggl ],
    \end{align*}
    thus by rolling out we have
    \begin{align*}
        \pi^{t+1}_1 G^t_1(s_1) &\leq \E_{\pi^{t+1}}\biggl[ \underbrace{12 \sum_{h=1}^H \left(1 + \frac{10}{H}\right)^{h} \sqrt{\Var_{p_h}[V^{\pi^{t+1}}_{\lambda,h+1}](s_h,a_h) \frac{\beta^{\conc}(\delta, n^t_h(s_h,a_h))}{n^t_h(s_h,a_h)}}}_{\termA} \\
        &+ \underbrace{2\kappa\sum_{h=1}^H \left(1 + \frac{10}{H}\right)^{h} \sqrt{\frac{2\beta^{\cH}(\delta, n^t_h(s_h,a_h))}{n^t_h(s_h,a_h)}}}_{\termB} \\
        &+ \underbrace{\sum_{h=1}^H \left(1 + \frac{10}{H}\right)^{h} \frac{54 H^2 \Rmax \beta^{\KL}(\delta, n^t_h(s_h,a_h))}{n^t_h(s_h,a_h)}}_{\termC} \biggl| s_1\biggl ],
    \end{align*}
    where $(1+10/H)^h \leq \rme^{10}$ for any $h \in [H]$. Now we bound each term separately. 

    \paragraph{Term $\termA$.} To bound this term, we apply Cauchy-Schwarz inequality
    \begin{align*}
        \termA &\leq 12 \rme^{10} \sum_{(h,s,a) \in [H] \times \cS \times \cA} d^{\pi^{t+1}}_h(s,a) \sqrt{\Var_{p_h}[V^{\pi^{t+1}}_{\lambda,h+1}](s,a) \frac{\beta^{\conc}(\delta, n^t_h(s,a))}{n^t_h(s,a)}} \\
        &\leq 12 \rme^{10} \sqrt{\sum_{(h,s,a) \in [H] \times \cS \times \cA} d^{\pi^{t+1}}_h(s,a) \Var_{p_h}[V^{\pi^{t+1}}_{\lambda,h+1}](s,a)} \cdot \sqrt{\sum_{(h,s,a) \in [H] \times \cS \times \cA} d^{\pi^{t+1}}_h(s,a) \frac{\beta^{\conc}(\delta, n^t_h(s,a))}{n^t_h(s,a)}}.
    \end{align*}
    For the first multiplier we apply the law of total variance (Lemma~\ref{lem:law_of_total_variance})
    \begin{align*}
        \sum_{h,s,a} d^{\pi^{t+1}}_h(s,a) \Var_{p_h}[V^{\pi^{t+1}}_{\lambda,h+1}](s,a) &\leq \sum_{h,s,a} d^{\pi^{t+1}}_h(s,a) \Var_{p_h}[V^{\pi^{t+1}}_{\lambda,h+1}](s,a) + \sum_{h,s}d^{\pi^{t+1}}_h(s) \Var_{\pi^{t+1}_h}[Q^{\pi^{t+1}}_{\lambda,h}](s) \\
        &= \Vvar^{\cH, \pi^{t+1}}_1(s_1) \leq H^2\Rmax^2.
    \end{align*}
    Therefore, 
    \[
        \termA \leq 24\rme^{10} H\Rmax\sqrt{\sum_{(h,s,a) \in [H] \times \cS \times \cA} d^{\pi^{t+1}}_h(s,a) \frac{\beta^{\conc}(\delta, n^t_h(s,a))}{n^t_h(s,a)}}.
    \]

    \paragraph{Term $\termB$.} For this term we may apply Jensen's inequality
    \begin{align*}
        \termB &\leq 2\kappa H\rme^{10} \E_{\pi^{t+1}} \left[ \frac{1}{H}\sum_{h=1}^H \sqrt{\frac{2\beta^{\cH}(\delta, n^t_h(s_h,a_h))}{n^t_h(s_h,a_h)}} \bigg| s_1 \right] \leq \kappa \sqrt{8H}\rme^{10}\sqrt{\sum_{h,s,a} d^{\pi^{t+1}}_h(s,a) \frac{\beta^{\cH}(\delta, n^t_h(s,a))}{n^t_h(s,a)}}.
    \end{align*}

    By summing up and replacing counts by pseudo-counts by Lemma~\ref{lem:cnt_pseudo} we obtain
    \begin{align*}
        \pi^{t+1}_1 G^t_1(s_1) &\leq 48\rme^{10} H\Rmax\sqrt{\sum_{(h,s,a) \in [H] \times \cS \times \cA} d^{\pi^{t+1}}_h(s,a) \frac{\beta^{\conc}(\delta, \upn^t_h(s,a))}{\upn^t_h(s,a) \vee 1}} \\
        &+ 4\kappa \rme^{10} \sqrt{2H}\sqrt{\sum_{h,s,a} d^{\pi^{t+1}}_h(s,a) \frac{\beta^{\cH}(\delta, \upn^t_h(s,a))}{\upn^t_h(s,a) \vee 1}} \\
        &+  4\rme^{10} H^2 \Rmax  \sum_{h,s,a} d^{\pi^{t+1}}_h(s,a)  \frac{\beta^{\KL}(\delta, \upn^t_h(s,a))}{\upn^t_h(s,a) \vee 1}.
    \end{align*}

    The last step is to notice that for $t < \tau$ we have $\pi^{t+1}_1 G^t_1(s_1) \geq \varepsilon$, thus summing over all $t < \tau$ we have
    \begin{align*}
        (\tau-1) \varepsilon &\leq 48\rme^{10} H\Rmax \sum_{t=1}^{\tau-1}\sqrt{\sum_{(h,s,a) \in [H] \times \cS \times \cA} d^{\pi^{t+1}}_h(s,a) \frac{\beta^{\conc}(\delta, \upn^t_h(s,a))}{\upn^t_h(s,a) \vee 1}} \\
        &+ 4\kappa \rme^{10} \sqrt{2H} \sum_{t=1}^{\tau-1}\sqrt{\sum_{h,s,a} d^{\pi^{t+1}}_h(s,a) \frac{\beta^{\cH}(\delta, \upn^t_h(s,a))}{\upn^t_h(s,a) \vee 1}} \\
        &+  4\rme^{10} H^2 \Rmax  \sum_{t=1}^{\tau-1} \sum_{h,s,a} d^{\pi^{t+1}}_h(s,a)  \frac{\beta^{\KL}(\delta, \upn^t_h(s,a))}{\upn^t_h(s,a) \vee 1}.
    \end{align*}

    Also we notice that $\beta^{\KL}(\delta, \cdot), \beta^{\conc}(\delta, \cdot), \beta^{\cH}(\delta, \cdot)$ are monotone, thus we may replace $\upn^t_h(s,a)$ with a stopping time $\tau$. Thus, by Jensen's inequality
    \begin{align*}
        (\tau-1) \varepsilon &\leq 48\rme^{10} H\Rmax \sqrt{(\tau-1) \cdot \beta^{\conc}(\delta, \tau-1)}\sqrt{\sum_{t=1}^{\tau-1} \sum_{(h,s,a) \in [H] \times \cS \times \cA} d^{\pi^{t+1}}_h(s,a) \frac{1}{\upn^t_h(s,a)\vee 1} } \\
        &+ 4\kappa \rme^{10} \sqrt{2H \beta^{\cH}(\delta, \tau) \cdot (\tau-1)} \sqrt{\sum_{t=1}^{\tau-1} \sum_{h,s,a} d^{\pi^{t+1}}_h(s,a) \frac{1}{\upn^t_h(s,a) \vee 1}} \\
        &+  4\rme^{10} H^2 \Rmax\beta^{\KL}(\delta, \tau-1)  \sum_{t=1}^{\tau-1} \sum_{h,s,a} d^{\pi^{t+1}}_h(s,a)  \frac{1}{\upn^t_h(s,a) \vee 1}.
    \end{align*}
    Furthermore, notice 
    \[
        \sum_{t=1}^{\tau-1} d^{\pi^{t+1}}_h(s,a)  \frac{1}{\upn^t_h(s,a) \vee 1} = \sum_{t=1}^{\tau-1} \frac{\upn^{t+1}_h(s,a) - \upn^t_h(s,a)}{\upn^t_h(s,a) \vee 1},
    \]
    thus Lemma~\ref{lem:sum_1_over_n} is applicable:
        \begin{align*}
        (\tau-1) \varepsilon &\leq 96\rme^{10} \Rmax \sqrt{(\tau-1)H^3SA \log(\tau) \cdot \beta^{\conc}(\delta, \tau-1)} \\
        &+ 2\kappa\rme^{10} \sqrt{2H^2 SA \log^3(\tau) \beta^{\cH}(\delta, \tau) \cdot (\tau-1)} \\
        &+  8\rme^{10} H^3SA \log(\tau) \Rmax\beta^{\KL}(\delta, \tau-1).
    \end{align*}
    By the definitions of $\beta^{\KL}, \beta^{\conc}, \beta^{\cH}$ we have the following inequality
    \begin{align*}
        (\tau-1) \varepsilon &\leq 96\rme^{10} \Rmax \sqrt{(\tau-1)H^3SA \log(\tau) \cdot (\log(16SAH/\delta) + 2 \log(\rme \tau)) } \\
        &+ 12\kappa\rme^{10} \sqrt{H^2 SA (\tau-1) \log^3(\tau) \cdot (\log(4SAH/\delta) + 2\log(\rme \tau) ) } \\
        &+  16\rme^{10} H^3SA \log(\tau) \Rmax (\log(4SAH/\delta) + S\log(\rme \tau)).
    \end{align*}
    Under assumption $\tau \geq 2$ we can proceed with the further simplifications
    \begin{align}\label{eq:tau_inequality}
        \begin{split}
            \tau \varepsilon &\leq 216\rme^{10} \Rmax \sqrt{\tau H^3SA \log^3(\tau) \cdot (\log(16 SAH/\delta) + 2 \log(\rme \tau)) } \\
        &+  32\rme^{10} H^3SA \log(\tau) \Rmax (\log(16SAH/\delta) + S\log(\rme \tau)).
        \end{split}
    \end{align}

    Let us define the following constants
    \[
        A = 216\rme^{10} \Rmax \cdot \sqrt{\frac{H^3SA}{\varepsilon^2}}, \quad B =\log(16 SAH/\delta), \quad C = \frac{32\rme^{10} \cdot H^3 SA \Rmax }{\varepsilon}.
    \]
    Then inequality~\eqref{eq:tau_inequality} has the following form
    \[
        \tau \leq A\sqrt{\tau(B + 2\log(\rme \tau)) \cdot \log^3(\tau)} + C(B + S \log(\rme \tau))\log \tau.
    \]
    First, we obtain a loose inequality on $\tau$. Let us use the inequality $\log(x) \leq x^\beta / \beta$ for any $x > 0, \beta > 0$ with different $\beta$ for each logarithm
    \begin{align*}
        \tau &\leq A \sqrt{216\tau (B + 4(\rme \tau)^{1/4}) \tau^{1/2}}+ 4C(B + 8S/3 (\rme \tau)^{3/8}) \tau^{1/4} \\
        \Rightarrow \tau^{3/4} &\leq \tau^{3/8} \left( 6A\sqrt{6(B + 4 \rme^{1/4})} + 12CS \rme^{3/8} \right) + 4CB.
    \end{align*}
    
    Notice that the solution to the inequality $x^2 \leq ax + b$ could be upper-bounded as follows
    \[
        x \leq \frac{a + \sqrt{a^2 + 4b}}{2} \leq a + \sqrt{b},
    \]
    thus
    \[
        \tau^{3/8} \leq \left( 6A\sqrt{6(B + 4 \rme^{1/4})} + 12CS \rme^{3/8} \right) + 2\sqrt{CB}.
    \]
    Define $L = 8/3 \log\left( 6A\sqrt{6(B + 4 \rme^{1/4})} + 12CS \rme^{3/8} + 2\sqrt{CB}\right) = \cO(\log(SAH\Rmax/\varepsilon) + \log\log(SAH/\delta) )$ and we have $\log(\tau) \leq L$. Then we have that the solution to \eqref{eq:tau_inequality} is a subset of solutions to
    \[
        \tau \leq A\sqrt{\tau(B + 2(1+L)) \cdot L^3} + C(B + S(1+L))L,
    \]
    solving this inequality we obtain the bound
    \[
        \tau \leq 2A^2(B + 2(1+L))L^3 + 2CB(S(1+L))L.
    \]    
\end{proof}

After this general result we state the bound for the MTEE problem that was stated in the main text.
\begin{theorem}\label{th:mtee_sample_complexity}
    For all $\varepsilon > 0$ and $\delta \in (0,1)$ the \UCBVIEnt algorithm is $(\varepsilon,\delta)$-PAC for MTEE. Moreover, with probability at least $1-\delta$
    \[
        \tau \leq \cO\left( \frac{H^3 SA \log^2(SA) \log(SAH/\delta) \cdot L^4}{\varepsilon^2} + \frac{H^3 SA(\log(SAH/\delta) + SL) \cdot L}{\varepsilon} \right),
    \]
    where $L = \log(SAH/\varepsilon) + \log\log(SAH/\delta)$.
\end{theorem}
\begin{proof}
    Fix $\Phi(\pi) = -\cH(\pi), \kappa=\lambda=1$ and $r_{\max} = 0$. Since $\cH(\pi)$ is $1$-strongly convex with respect to $\ell_1$-norm, we have $r_A = 1$. Also we automatically have $\Rmax = \log(SA)$. In this setting, Theorem~\ref{th:reg_agnostic_sample_complexity} yields the desired statement.
\end{proof}

\newpage
\section{Fast Rates for MTEE and Regularized MDPs}\label{app:fast_rates_regularized}

In this section we describe an algorithm that will achieve $\tcO(\poly(S,A,H)/\varepsilon)$ sample complexity for regularized MDPs. Additionally, we show that this algorithm could be used for reward-free exploration under regularization.

\subsection{\RFExploreEnt Algorithm}

We leverage the reward-free exploration approach by \citet{jin2020reward-free}. Our algorithm is split into two phases: the first phase is devoted to reward-free exploration, and on the second phase the collected samples are used to build estimates of transition probabilities and entropy of transitions. The main idea is that regularization allows us to collect much smaller number of samples to control the policy error.

\paragraph{Exploration phase} We first learn a policy that visit uniformly the MDP. To this aim for each state $s'\in\cS$ at each step $h'\in[H]$ we build the reward that put one on this state at step $h$ and zeros everywhere else $r_h(s,a) = \ind\{(s,h)=(s',h')\}$. We note that the reward function does not depend on action taken.
We then run the \EULER algorithm for $N_0$ episodes in the MDP equipped with the reward $r$ and denote by $\tilde{\Pi}_{s',h'}$ the set of $N_0$ policies used by \EULER to interact with the MDP. We modify this set of policies by forcing to act uniformly at the goal state $s'$ into the set 

\[\Pi_{s',h'} = \Bigg\{\pi'_{h}(a|s) = \begin{cases} 1/A &\text{if } s=s', h=h'\\ \pi_h(s,a)&\text{else} \end{cases}:\ \pi\in\tilde{\Pi}_{s',h'}\Bigg\}\,.\]
We define the (non-Markovian) policy $\pi^{\mathrm{mix}}$ as the uniform mixture of the policies $\{\pi\in\Pi_{s,h}, (s,h)\in\cS\times[H]\}$ we just constructed. As proved by \citet{jin2020reward-free} the policy $\pi^{\mathrm{mix}}$ is built such that it will visit almost uniformly all the states that can be reached in the MDP from the initial state. Before precising this property we need to introduce the notion of significant state.

\begin{definition}\label{def:significant_states}
    A state $s$ at step $h$ is called $\varepsilon'$-significant if there exists a policy $\pi$ such that the visitation probability of $s$ under policy $\pi$ is greater than $\varepsilon'$:
    \[
        \max_{\pi} d^\pi_h(s) \geq \varepsilon'.
    \]
    The set of all $\varepsilon'$-significant state-step pairs is called $S_{\varepsilon'}$.
\end{definition}
We reproduce here the result by \citet{jin2020reward-free} that shows that the policy $\pi^{\mathrm{mix}}$ will visit any significant state with a large enough probability.
\begin{theorem}[Theorem 3.3 by \citealt{jin2020reward-free}]\label{th:rf_explore_sampling} There exists an absolute constant $c > 0$ such that for any $\varepsilon' > 0$ and $\delta \in (0,1)$, if we set the parameter $N_0 \geq c S^2 A H^4 L/\varepsilon'$ where $L = \log(SAH/(\delta \varepsilon'))$, then with probability at least $1-\delta/3$ the following event holds
\[
    \cE^{\RFExplore}(\delta, \varepsilon')  = \left\{ \forall (s,h) \in S_{\varepsilon'}, \forall a \in \cA, \forall \pi:  \frac{d^\pi_h(s,a)}{\mu_h(s,a)} \leq 2 SAH \right\},
\]
where we denote the visitation distribution of policy $\pi^{\mathrm{mix}}$ by $\mu_h(s,a)= d^{\pi^{\mathrm{mix}}}_h(s,a)$.
\end{theorem}
\begin{remark} Note that the space complexity of \RFExploreEnt is very large since we need to store all the intermediate policies in order to construct $\pi^{\mathrm{mix}}$.
\end{remark}

This policy $\pi^{\mathrm{mix}}$ is then used to collect $N$ new independent trajectories $(z_n)_{n\in[N]}$ where $z_n=  (s^n_1,a^n_1,\ldots,s^n_H, a^n_H, s^n_{H+1})$ by following the policy $\pi^{\mathrm{mix}}$ in the original MDP. Next define the set $\cD= \{(s^n_h,a^n_h,s^n_{h+1}) ,h\in[H], n \in[N]\}$ consisting of the transitions in the sampled trajectories.

% \todoDa{Add sample complexity information}
% \todoPi{TODO Put the definition of the event in the theorem. And define $\pi^{\mathrm{mix}}$ and $\mu_h(s,a) = d_h^{\pi^{\mathrm{mix}}}(s,a)$.}
% \db{what is $N_0$ ? how is it related to complexity of the algorithm ?}
% \todoDa{Parameter of the algorithm, will be additional $SHN_0$ (basically, it is a number of iteration of Euler algorithm with reward that will be $1$ on visitation of state-step pair $(s,h)$)}
% \db{What does it mean ``$N$ trajectories $\{ z_n\}_{n=1}^N$ sampled i.i.d. from a distribution $\mu$'' ? what are the values of $z_n$, pairs ?}
% \todoDa{It is trajectories: $z_n = (s_1, a_1, \ldots, s_H, a_H, s_{H+1})$. $\mu_h$ is a marginal distribution for state-action pairs of a mixture policy by algorithm of \cite{jin2020reward-free}. }
% \todoPi{Ok maybe we need more context from \citet{jin2020reward-free} before giving the theorem. Then we can talk about the datset only in terms of transitions (and rescale quantities by $H$.}
% \db{yes, more context would be good. otherwise many things are not clear, e.g. mixture policy }

\paragraph{Planning phase} Given the transitions collected in the exploration phase we estimate the transition probability distributions and then plan in the estimated MDP with the Bellman equations for MTEE to obtain a maximum trajectory entropy policy.

Using the dataset $\cD$, we construct estimates of transition probabilities $\{\hp_h\}_{h\in[H]}$. We first define the number of visits of a state action pair $(s,a)$ at step $h$ and the number of transitions for $(s,a)$ at step $h$ to a states $s'$ observed in the dataset $\cD$,
\[
    n_h(s,a) = \sum_{n=1}^N \ind\{ (s^n_h, a^n_h) = (s,a) \} \quad  n_h(s'|s,a) = \sum_{n=1}^N \ind\{ (s^n_h, a^n_h, s^n_{h+1}) = (s,a,s') \},
\]
The transitions are estimated using the maximum likelihood method:
\begin{align}\label{eq:hp_construction}
    \hp_h(s'|s,a) = \begin{cases}
        \frac{n_h(s'|s,a)}{n_h(s,a)} & n_h(s,a) > 0 \\
        \frac{1}{S} & n_h(s,a) = 0
    \end{cases}\,.
\end{align}
% \db{how exactly ? because the trajectories are obtained using iid data, you get constant conditional probabilities or ?}
% \todoDa{We can treat $(s_h,a_h)$ generated from any policy $\pi$ as $(s_h,a_h) \sim d^\pi_h$ i.i.d. over trajectories because trajectories are independent. So, at each layer we will have independent sample}
% \db{Does it mean that $(s_h,a_h)$ are independent for different $h$ ? I still do not understand how you  recover the conditional probabilities $s_{h+1}$ given $s_h$ and $a_h$ from $d^\pi_h.$  }
% \todoDa{They will be dependent of course. If you know policy and $d^\pi_h$, you cannot reconstruct the model (at least it is not clear to me), but why it is needed? We can reconstruct the model from the joint distribution over trajectories: $q^{\pi}(m) = \pi_1(s_1) \prod_{i=2}^H p_{i-1}(s_{i+1}|s_i, a_i)$. I don't understand what is unclear there.}
% and solve regularized Bellman equations with respect to this policy. 
% For a given policy $\pi$, we call $\hQ^\pi_h(s,a)$ the Q-values computed with estimated transitions and estimated entropy of the transitions. They could be defined by the following version of Bellman equations
Given these estimates, we can define an empirical version of the regularized Bellman equations for MTEE
\begin{align}\label{eq:hQ_definition}
    \begin{split}
        \hQ^{\pi}_{\lambda,h}(s,a) &= r_h(s,a) + \kappa \cH(\hp_h(s,a)) + \hp_h \hV^{\pi}_{\lambda, h+1}(s,a) \\
        \hV^{\pi}_{\lambda, h}(s) &= \pi \hQ^{\pi}_{\lambda, h}(s) - \lambda \Phi(\pi).
    \end{split}
\end{align}

Then the output policy $\hpi$ is the solution to the optimal regularized Bellman equations
\begin{align}\label{eq:hQ_opt_definition}
    \begin{split}
        \hQ^{\star}_{\lambda,h}(s,a) &= r_h(s,a) + \kappa \cH(\hp_h(s,a)) + \hp_h \hV^{\star}_{\lambda, h+1}(s,a) \\
        \hV^{\star}_{\lambda, h}(s) &= \max_{\pi}\left\{ \pi \hQ^{\star}_{\lambda, h}(s) - \lambda \Phi(\pi) \right\} \\
        \hpi_h(s) &= \argmax_{\pi}\left\{ \pi \hQ^{\star}_{\lambda, h}(s) - \lambda \Phi(\pi) \right\}.
    \end{split}
\end{align}
We call this algorithm \RFExploreEnt. Notably, we can extend this algorithm to the setting of the changing rewards by solving \eqref{eq:hQ_opt_definition} with new reward functions $r_h(s,a)$. The detailed description of the algorithm is presented in Algorithm~\ref{alg:RFExploreEnt}. The only difference between our algorithm and \RFExplore by \citet{jin2020reward-free} is the use of a smaller number of trajectories $N$ and solving regularized Bellman equations instead of usual one.

\subsection{Concentration Events}

In this section we describe all required concentration events.

Let $\beta^{\conc}\colon (0,1) \times \N \to \R_{+}$ and $\beta^{\cnt} \colon (0,1) \to \R_+$ be some functions defined later on in Lemma \ref{lem:fast_traj_proba_master_event}. We define the following favorable events
\begin{align*}
\cE^{\conc}(N,\delta) &\triangleq \Bigg\{ \forall h\in [H], \forall G \colon \cS \to [0, H\Rmax], \forall \nu \colon \cS \to \cA: \\
&\qquad\quad \E_{(s,a) \sim \mu_{h}}\left[ \left(\left[ \hp_{h'} - p_{h'}  \right] G(s,a) \right)^2 \ind\{\nu(s) = a \} \right] \leq  \frac{CH^2 \Rmax^2 S \cdot \beta^{\conc}(\delta, N)}{N}\Bigg\}\,,\\
\cE^{\cH}(N,\delta) &\triangleq \Bigg\{\forall h \in [H]: \E_{(s,a) \sim \mu_{h}}\left[ (\cH(\hp_{h}(s,a)) - \cH(p_{h}(s,a)))^2 \right] \leq \frac{12S^2 A \log^2(SN) \cdot \beta^{\cH}(\delta)}{N}
\Bigg\}\,,
\end{align*}
where $C$ is a some absolute constant.
We also introduce two intersections of these events of interest and $\cE^{\RFExplore}(\delta)$, defined in Theorem~\ref{th:rf_explore_sampling}: $\cG(N,\delta,\varepsilon') \triangleq \cE^{\conc}(N,\delta) \cap \cE^{\cH}(\delta) \cap \cE^{\RFExplore}(\delta, \varepsilon')$. We  prove that for the right choice of the functions $\beta^{\conc}, \beta^{\cH}$ the above events hold with high probability.
\begin{lemma}
\label{lem:fast_traj_proba_master_event}
For any $\delta \in (0,1), \varepsilon' > 0, N \in \N$ and for the following choices of functions $\beta,$
\begin{align*}
    \beta^{\conc}(\delta, N) &\triangleq \log(3AH\Rmax N/\delta),\\
    \beta^{\cH}(\delta) &\triangleq \log(12SAH/\delta),
\end{align*}
it holds that
\begin{align*}
 \P[\cE^{\conc}(\delta)]\geq 1-\delta/3, \qquad \P[\cE^{\cH}(\delta)]\geq 1-\delta/3.
\end{align*}
In particular, $\P[\cG(\delta, N,\varepsilon')] \geq 1-\delta$.
\end{lemma}
\begin{proof}
    Holds from an application of Lemma~\ref{lem:sampling_square_value_error_bound}, Lemma~\ref{lem:sampling_entropy_bound}, Theorem~\ref{th:rf_explore_sampling} and union bound.
\end{proof}

\subsection{Sample Complexity Proof}

In this section we provide the sample complexity result of \RFExploreEnt algorithm in the simple BPI setting and in the reward free setting.

\begin{theorem}\label{th:rf_explore_ent_sample_complexity}
    Algorithm \RFExploreEnt with parameters $N_0 = \Omega\left( \frac{H^7 S^3 A r_A^2  \cdot \Rmax^2 \cdot L}{\varepsilon \lambda}\right)$ and $N \geq \Omega\left( \frac{ H^6 S^3 A r_A^2 \Rmax^2 L^3 }{\varepsilon \lambda}\right)$ is $(\varepsilon,\delta)$-PAC for the best policy identification in regularized MDPs, where $L = \log(SAH/(\varepsilon \lambda \delta))$. The sample complexity is bounded by
    \[
        \tcO\left( \frac{H^8 S^4 A r_A^2 \Rmax^2}{\varepsilon \lambda} \right).
    \]
\end{theorem}
\begin{proof}
    Let us start from exploiting the strong convexity of the regularizer. This property is given by Lemma~\ref{lem:policy_error_decomposition}
    \[
        \Vstar_{\lambda,1}(s_1) - V^{\hpi}_{\lambda,1}(s_1) \leq \frac{r_A^2}{2\lambda} \sum_{h=1}^H \E_{\hpi}\left[\max_{a\in \cA } \left( \hQ^{\hpi}_{\lambda,h} - \Qstar_{\lambda,h} \right)^2(s_h,a) \middle| s_1 \right].
    \]
    Next we study each separate term in this decomposition. By the definition of $\hpi$ and $\pistar$ we have
    \[
    \hQ^{\pistar}_{\lambda,h}(s,a) - Q^{\pistar}_{\lambda,h}(s,a) \leq \hQ^{\hpi}_{\lambda,h}(s,a) - \Qstar_{\lambda,h}(s,a) \leq \hQ^{\hpi}_{\lambda,h}(s,a) - Q^{\hpi}_{\lambda,h}(s,a),
    \]
    thus
    \[
        \left( \hQ^{\pistar}_{\lambda,h}(s,a) - Q^{\pistar}_{\lambda,h}(s,a)\right)^2 \leq \max\left\{ \left( \hQ^{\hpi}_{\lambda,h}(s,a) - \Qstar_{\lambda,h}(s,a) \right)^2, \left(\hQ^{\hpi}_{\lambda,h}(s,a) - Q^{\hpi}_{\lambda,h}(s,a) \right)^2  \right\},
    \]
    and by an inequality $\max\{a,b\} \leq a+b$ for positive $a,b$ we have
    \begin{align*}
        \left( \hQ^{\hpi}_{\lambda,h}(s,a) - \Qstar_{\lambda,h}(s,a) \right)^2 &\leq \left(\hQ^{\hpi}_{\lambda,h}(s,a) - Q^{\hpi}_{\lambda,h}(s,a) \right)^2  + \left( \hQ^{\pistar}_{\lambda,h}(s,a) - Q^{\pistar}_{\lambda,h}(s,a)  \right)^2.
    \end{align*}
    Therefore, the policy error decomposes as follows
    \begin{align*}
        \Vstar_{\lambda,1}(s_1) - V^{\hpi}_{\lambda,1}(s_1) &\leq \frac{r_A^2}{2\lambda}\sum_{h=1}^H \biggl( \E_{\hpi}\left[ \max_{a\in \cA}\left(\hQ^{\hpi}_{\lambda,h}(s_h,a) - Q^{\hpi}_{\lambda,h}(s_h,a) \right)^2  \mid s_1 \right] \\
        &\qquad\qquad + \E_{\hpi}\left[ \max_{a\in \cA}\left( \hQ^{\pistar}_{\lambda,h}(s_h,a) - Q^{\pistar}_{\lambda,h}(s_h,a)  \right)^2  \mid s_1 \right]\biggl).
    \end{align*}
    Next we assume that the event $\cG(N,\delta,\varepsilon')$ holds for the values $N$ and $\varepsilon'$ that will be specified later.  Then Lemma~\ref{lem:q_square_bound} applied $2H$ times yields
    \begin{align*}
        \Vstar_{\lambda,1}(s_1) - V^{\hpi}_{\lambda,1}(s_1) &\leq \frac{r_A^2}{\lambda}\biggl( \frac{48 S^3 H^4 A \Rmax^2 \log^2(N) \log(12SAH/\delta)}{N} \\
        &+ \frac{4C H^6 \Rmax^2 S^2 A \cdot (\log(3AH\Rmax/\delta) + \log(N))}{N} + 2 S H^3 \Rmax^2 \cdot \varepsilon' \biggl).
    \end{align*}
    Next we take
    \[
        \varepsilon' = \frac{\lambda \varepsilon}{4r_A^2 S H^3 \Rmax^2},
    \]
    that requires to take $N_0 \geq \frac{cH^7 S^3 A r_A^2  \cdot \Rmax^2 \cdot L}{\varepsilon \lambda}$ for an absolute constant $c > 0$ and $L = \log(SAH/\delta) + \log(1/(\varepsilon \lambda))$. This yields $\tcO(H^8 S^4 A r_A^2 \Rmax^2 / (\varepsilon \lambda))$ sample complexity of the first phase, since we need $N_0$ samples for each $(s,h) \in \cS \times [H]$. Under this choice, we have
    \[
        \Vstar_{\lambda,1}(s_1) - V^{\hpi}_{\lambda,1}(s_1) \leq \varepsilon/2 + \frac{r^2_A}{\lambda N} \cdot \left( (48 + 4C) S^3 A H^6 \Rmax^2 \log^2(N) \cdot \log(12SAH \Rmax/\delta) \right).
    \]
    To make the second part smaller than $\varepsilon$, we have to analyze the following inequality
    \[
        \log^2(N) / N \cdot B \leq \varepsilon
    \]
    and upper bound its minimal solution that we will call $N^\star$. To do it, we first use a simple numeric bound $\log(N) \leq 4 N^{1/4}$ and obtain a simple estimate $N \geq 256 B^2 / \varepsilon^2$, thus the minimal solution $N^\star \leq 256 B^2 / \varepsilon^2$.
    Therefore, we can assume that $\log(N^\star) \leq 2\log(16B/\varepsilon) = \cO(\log(SAH\Rmax/\varepsilon + \log(1/\delta))$, thus taking
    \[
        N \geq N^\star = \Omega\left( \frac{ H^6 S^3 A r_A^2 \Rmax^2 \log^2(SAH\Rmax/\varepsilon + \log(1/\delta)) \cdot \log(SAH \Rmax/\delta) }{\varepsilon \lambda}\right).
    \]
    is enough to guarantee that the policy error is smaller than $\varepsilon$.
\end{proof}

Notice that in the proof we do not rely on one particular reward function, since the only we need is conditioning on event $\cG(\delta)$ that does not depend on the particular reward function.
\begin{corollary}\label{cor:rf_explore_ent_rf_sample_complexity}
    Algorithm \RFExploreEnt for a  choice $N_0 = \Omega\left( \frac{H^7 S^3 A r_A^2  \cdot \Rmax^2 \cdot L}{\varepsilon \lambda}\right)$ and $N \geq \Omega\left( \frac{ H^6 S^3 A r_A^2 \Rmax^2 L^3 }{\varepsilon \lambda}\right)$ outputs $\varepsilon$-optimal policies for an arbitrary number of reward functions in regularized MDPs. The sample complexity is bounded by
    \[
        \tcO\left( \frac{H^8 S^4 A r_A^2 \Rmax^2}{\varepsilon \lambda} \right).
    \]
\end{corollary}

Finally, we provide a formal proof for application of this algorithm to the MTEE problem, that is a simple application of the results above.
\begin{theorem}\label{th:mtee_fast_rates}
    Algorithm \RFExploreEnt with parameters $N_0 = \Omega\left( \frac{H^7 S^3 A \cdot L^3}{\varepsilon}\right)$ and $N = \Omega\left( \frac{ H^6 S^3 A  L^5 }{\varepsilon}\right)$ is $(\varepsilon,\delta)$-PAC for the MTEE problem, where $L = \log(SAH/(\varepsilon \delta))$. The total sample complexity $SHN_0 + N$ is bounded by
    \[
        \tcO\left( \frac{H^8 S^4 A}{\varepsilon} \right).
    \]
\end{theorem}
\begin{proof}
    Fix $\Phi(\pi) = -\cH(\pi), \kappa = \lambda = 1$ and $r_{\max} = 0$. By 1-strong convexity of $-\cH(\pi)$ with respect to $\ell_1$-norm, its dual is 1-strongly convex with respect to $\ell_\infty$ norm, yielding $r_A = 1$. Also we have $\Rmax = \log(SA)$, thus by \thmref{th:rf_explore_ent_sample_complexity} we conclude the statement.
\end{proof}

\subsection{Technical Lemmas}

\begin{lemma}\label{lem:policy_error_decomposition}
    Let $\pi$ be a greedy policy with respect to regularized Q-values $\uQ_{\lambda,h}(s,a): \pi_h(s) = \argmax_{\pi} \{ \pi \uQ_{\lambda,h}(s) - \lambda \Phi(\pi) \}$. Then the following error decomposition holds
    \[
        \Vstar_{\lambda,1}(s_1) - V^{\pi}_{\lambda,1}(s_1) \leq \frac{r_A^2}{2\lambda}\E_{\pi}\left[ \sum_{h=1}^H \max_{a\in \cA } \left( \uQ_{\lambda,h} - \Qstar_{\lambda,h} \right)^2(s_h,a) \mid s_1 \right],
    \]
    where $r_A$ is a constant defined in \eqref{eq:norm_equivalence}.
\end{lemma}
\begin{proof}
    First, we formulate the statement dependent of $h$
    \begin{equation}\label{eq:policy_error_h}
        \Vstar_{\lambda,h}(s) - V^{\pi}_{\lambda,h}(s) \leq \frac{r_A^2}{2\lambda}\E_{\pi}\left[ \sum_{h'=h}^H \max_{a\in \cA } \left( \uQ_{\lambda,h'} - \Qstar_{\lambda,h'} \right)^2(s_{h'},a) \mid s_h=s \right].
    \end{equation}
    Notice that for $h=1$ and $s=s_1$ the initial statement is recovered. We proceed by induction over $h$. The initial case $h=H+1$ is trivial, next we assume that the statement \eqref{eq:policy_error_h} is true for any $h' > h$.

    We start the analysis from understanding the policy error by applying the smoothness of $F_\lambda$ for any $h$.
    \begin{align*}
        \Vstar_{\lambda,h}(s) - V^{\pi}_{\lambda,h}(s) &= F_\lambda(\Qstar_{\lambda,h}(s, \cdot)) - \left(\pi_h Q^{\pi}_{\lambda,h}(s, \cdot)  -\lambda \Phi(\pi_h(s)) \right) \\
        &\leq F_\lambda(\uQ_h)(s) + \langle \nabla F_\lambda(\uQ_h(s,\cdot)), \Qstar_{\lambda,h}(s,\cdot) - \uQ_h(s,\cdot)  \rangle + \frac{1}{2\lambda} \norm{\uQ_h - \Qstar_{\lambda,h}}_*^2(s) \\
        &- \left(\pi_h Q^{\pi}_{\lambda,h}(s, \cdot)  -\lambda \Phi(\pi_h(s)) \right).
    \end{align*}
    Next we recall that
    \[
         \pi_h(s) = \nabla F(\uQ_h(s,\cdot)), \quad F(\uQ_h)(s)  =  \pi_h \uQ_h(s) - \lambda \Phi(\pi_h(s)),
    \]
    thus we have
    \[
        F(\uQ_h)(s) - \left( \pi_h Q^{\pi}_{\lambda,h}(s, \cdot) - \lambda \Phi(\pi_h(s))  \right) = \pi_h [ \uQ_h - Q^{\pi}_{\lambda,h}](s)
    \]
    and, by Bellman equations
    \begin{align*}
        \Vstar_{\lambda,h}(s) - V^{\pi}_{\lambda,h}(s) &\leq \pi_h \left[ \Qstar_{\lambda,h} - Q^{\pi}_{\lambda,h} \right] (s) + \frac{1}{2\lambda} \norm{\uQ_h - \Qstar_{\lambda,h}}_*^2(s) \\
        &\leq \pi_h p_h \left[ \Vstar_{\lambda,h+1} - V^{\pi}_{\lambda,h+1} \right] (s) + \frac{1}{2\lambda} \norm{\uQ_h - \Qstar_{\lambda,h}}_*^2(s).
    \end{align*}
    Applying norm equivalence \eqref{eq:norm_equivalence} we have
    \[
        \Vstar_{\lambda,h}(s) - V^{\pi}_{\lambda,h}(s) \leq \E_{\pi}\left[ \frac{r_A^2}{2\lambda} \norm{\uQ_h - \Qstar_{\lambda,h}}_\infty^2(s_h) + \Vstar_{\lambda,h+1}(s_{h+1}) - V^{\pi}_{\lambda,h+1}(s_{h+1})  \mid s_h = s\right].
    \]
    By induction hypothesis we conclude the statement.
\end{proof}

\begin{lemma}\label{lem:q_square_bound}
    For any policy $\pi$ the following holds on event $\cG(N,\delta,\varepsilon')$ defined in Lemma~\ref{lem:fast_traj_proba_master_event} for any $h \in [H]$
    \begin{align*}
        \E_{\hpi}\left[ \max_{a\in \cA} \left( \hQ^{\pi}_{\lambda,h} - Q^\pi_{\lambda,h} \right)^2(s_h,a) \mid s_1 \right] &\leq \frac{48 S^3 H^3 A \Rmax^2 \log^2(N) \beta^{\cH}(\delta)}{N} + \frac{4C H^5 \Rmax^2 S^2 A \cdot \beta^{\conc}(\delta,N)}{N} \\
        &+ 2 S H^2 \Rmax^2 \cdot \varepsilon'.
    \end{align*}
\end{lemma}
\begin{proof}
    By performance-difference Lemma~\ref{lm:performance_difference} and form of the rewards stated in \eqref{eq:hQ_definition} we have for any $(s,a,h)\in \cS \times\cA\times[H]$
    \begin{align*}
        \hQ^{\pi}_{\lambda,h}(s,a) - Q^\pi_{\lambda,h}(s,a) &= \kappa \E_{\pi}\left[ \sum_{h'=h}^H \cH(\hp_{h'}(s_{h'},a_{h'})) - \cH(p_{h'}(s_{h'}, a_{h'})) \mid (s_h,a_h) = (s,a) \right] \\
        &+ \E_{\pi}\left[ \sum_{h'=h}^H \left[ \hp_{h'} - p_{h'}\right] \hV^{\pi}_{\lambda, h'+1}(s_{h'}, a_{h'}) \mid (s_h,a_h) = (s,a)\right].
    \end{align*}
    
    Next we analyze all required expectation for one fixed value $h\in[H]$. By Jensen's inequality and a simple algebraic inequality $(a+b)^2 \leq 2a^2 + 2b^2$
    \begin{align*}
        \left(\hQ^{\pi}_{\lambda,h}(s,a) - Q^\pi_{\lambda,h}(s,a)\right)^2 &\leq 2\kappa^2 \E_{\pi}\left[ \left(\sum_{h'=h}^H \cH(\hp_{h'}(s_{h'},a_{h'})) - \cH(p_{h'}(s_{h'}, a_{h'})\right)^2 \bigg| (s_h,a_h) = (s,a) \right] \\
        &+ 2\E_{\pi}\left[ \left( \sum_{h'=h}^H \left[ \hp_{h'} - p_{h'}\right] \hV^{\pi}_{\lambda, h'+1}(s_{h'}, a_{h'})\right)^2 \bigg| (s_h,a_h) = (s,a)\right].
    \end{align*}
    By Cauchy–Schwarz inequality we have the final form
    \begin{align}
        \begin{split}\label{eq:performance_difference_dq}
        \left(\hQ^{\pi}_{\lambda,h}(s,a) - Q^\pi_{\lambda,h}(s,a)\right)^2 &\leq 2H \E_{\pi}\biggl[ \sum_{h'=h}^H \kappa^2\left(\cH(\hp_{h'}(s_{h'}, a_{h'})) - \cH(p_{h'}(s_{h'}, a_{h'})) \right)^2 \\
        &+ \sum_{h'=h}^H \left(\left[ \hp_{h'} - p_{h'}\right] \hV^{\pi}_{\lambda, h'+1}\right)^2(s_{h'}, a_{h'}) \mid (s_h,a_h) = (s,a) \biggl]
        \end{split}
    \end{align}
    To connect this conditional expectation in the right-hand side of \eqref{eq:performance_difference_dq} with conditional expectation over $\hpi,$ we define the following policy
    \[
        \tpi_{h'}(a|s) = \begin{cases}
            \hpi_{h'}(a|s) & h' < h \\
            \ind\{ a = \argmax_{a\in \cA} \left\{ \left(\hQ^{\pi}_h(s,a) - Q^\pi_h(s,a)\right)^2  \right\} & h' = h \\
            \pi_h(a|s) & h' > h.
        \end{cases}
    \]
    Since we do not change the policy $\pi$ for steps greater than $h$, we can replace $\pi$ with $\tpi$ in \eqref{eq:performance_difference_dq}. Additionally, since this policy is equal to $\hpi$ for the first $h-1$ steps, the distribution $d^{\hpi}_h(s_h)$ is equal to $d^{\tpi}_h(s_h)$:
    \begin{align*}
        d^{\hpi}_h(s_h) &= \sum_{s_1,a_1,\ldots,s_{h-1},a_{h-1}} \hpi_1(a_1|s_1) \left( \prod_{h'=1}^{h-1} p_{h'-1}(s_{h'}|s_{h'-1},a_{h'-1}) \hpi_{h'}(a_{h'}|s_{h'}) \right) \cdot p_{h-1}(s_h|s_{h-1},a_{h-1}) \\
        &= \sum_{s_1,a_1,\ldots,s_{h-1},a_{h-1}} \tpi_1(a_1|s_1) \left( \prod_{h'=1}^{h-1} p_{h'-1}(s_{h'}|s_{h'-1},a_{h'-1}) \tpi_{h'}(a_{h'}|s_{h'}) \right) \cdot p_{h-1}(s_h|s_{h-1},a_{h-1}) = d^{\tpi}_h(s_h).
    \end{align*}
    Therefore
    \begin{align*}
        \E_{\hpi}\left[ \max_{a\in \cA} \left( \hQ^{\pi}_{\lambda,h} - Q^\pi_{\lambda,h} \right)^2(s_h,a) \mid s_1 \right] &= \sum_{s} d^{\hpi}_h(s) \max_{a\in \cA}\left( \hQ^{\pi}_{\lambda,h} - Q^\pi_{\lambda,h} \right)^2(s,a) \\
        &=\sum_{s} d^{\tpi}_h(s) \max_{a\in \cA}\left( \hQ^{\pi}_{\lambda,h} - Q^\pi_{\lambda,h} \right)^2(s,a) \\
        &=  \E_{\tpi}\left[ \max_{a\in \cA} \left( \hQ^{\pi}_{\lambda,h} - Q^\pi_{\lambda,h} \right)^2(s_h,a) \mid s_1 \right] = \E_{\tpi}\left[  \left( \hQ^{\pi}_{\lambda,h} - Q^\pi_{\lambda,h} \right)^2(s_h,a_h) \mid s_1 \right].
    \end{align*}
    Next we show that in \eqref{eq:performance_difference_dq} we can make a change of policy from $\pi$ to $\tpi$. It is enough to show that the required marginal distributions are equal for all $h' \geq h$, i.e. for any $(s,a) \in \cS\times \cA$ it holds $\P_{\pi}[(s_{h'},a_{h'}) | (s_h,a_h)] = \P_{\tpi}[(s_{h'},a_{h'}) | (s_h,a_h)]$. For $h'=h$ this probability is an indicator on $(s_h,a_h)$, so it does not depend on policy. For the general case $h'>h$ we can use Markov property and imply
    \begin{align*}
        \P_{\pi}[(s_{h'},a_{h'}) | (s_h,a_h)] &= \sum_{(s_{h+1},a_{h+1},\ldots,s_{h'-1},a_{h'-1})} \prod_{i=h}^{h'-1} \pi_{i+1}(a_{i+1}|s_{i+1}) p_{i}(s_{i+1} | s_i, a_i) \\
        &= \sum_{(s_{h+1},a_{h+1},\ldots,s_{h'-1},a_{h'-1})} \prod_{i=h}^{h'-1} \tpi_{i+1}(a_{i+1}|s_{i+1}) p_{i}(s_{i+1} | s_i, a_i) 
        &= \P_{\tpi}[(s_{h'},a_{h'}) | (s_h,a_h)].
    \end{align*}
    Therefore, we can make change of measure in \eqref{eq:performance_difference_dq} and obtain
    \begin{align*}
        \E_{\hpi}\left[ \max_{a\in \cA} \left( \hQ^{\pi}_{\lambda,h} - Q^\pi_{\lambda,h} \right)^2(s_h,a) \mid s_1 \right]
        &\leq 2H \E_{\tpi} \biggl[ \E_{\tpi}\left[\sum_{h'=h}^H \kappa^2\left(\cH(\hp_{h'}(s_{h'}, a_{h'})) - \cH(p_{h'}(s_{h'}, a_{h'})) \right)^2   \mid s_h, a_h\right] \\
        &\quad+ \E_{\tpi}\left[\sum_{h'=h}^H  \left(\left[ \hp_{h'} - p_{h'}\right] \hV^{\pi}_{\lambda, h'+1}\right)^2(s_{h'}, a_{h'})\mid s_h,a_h\right]\ \bigg|\ s_1 \biggl].
    \end{align*}
    By the properties of conditional expectation we can eliminate the inner expectation and obtain the following upper bound
    \begin{align*}
        \E_{\hpi}\left[ \max_{a\in \cA} \left( \hQ^{\pi}_{\lambda,h} - Q^\pi_{\lambda,h} \right)^2(s_h,a) \mid s_1 \right] &\leq 2H\kappa^2 \sum_{h'=h}^H \E_{\tpi}\left[ \left(\cH(\hp_{h'}(s_{h'}, a_{h'})) - \cH(p_{h'}(s_{h'}, a_{h'})) \right)^2  \ \big|\ s_1 \right] \\
        &+ 2H \sum_{h'=h}^H \E_{\tpi}\left[ \left(\left[ \hp_{h'} - p_{h'}\right] \hV^{\pi}_{\lambda, h'+1}\right)^2(s_{h'}, a_{h'}) \ \bigg|\ s_1\right].
    \end{align*}
    Applying Lemma~\ref{lem:change_measure} and the definition of $\cE^{\cH}(\delta)$ we have
    \begin{align*}
        \E_{\tpi}\left[ \left(\cH(\hp_{h'}(s_{h'}, a_{h'})) - \cH(p_{h'}(s_{h'}, a_{h'})) \right)^2  \mid s_1 \right] &\leq 2SAH \E_{(s,a) \sim \mu_{h'}}\left[ \left(\cH(\hp_{h'}(s, a)) - \cH(p_{h'}(s, a)) \right)^2  \right] + S \log^2(S) \varepsilon' \\
        &\leq \frac{24 S^3 H A \log^2(SN) \beta^{\cH}(\delta)}{N} + S \log^2(S) \varepsilon'.
    \end{align*}

    In the same way by Lemma~\ref{lem:change_measure} and the definition of event $\cE^{\conc}(N,\delta)$
    \begin{align*}
        \E_{\tpi}\left[ \left(\left[ \hp_{h'} - p_{h'}\right] \hV^{\pi}_{\lambda, h'+1}\right)^2(s_{h'}, a_{h'}) \mid s_1\right] &\leq 2SAH\E_{(s,a) \sim \mu_{h'}}\left[ \left(\left[ \hp_{h'} - p_{h'}\right] \hV^{\pi}_{\lambda, h'+1}\right)^2(s, a)\right] + SH^2 \Rmax^2 \varepsilon' \\
        &\leq \frac{2C H^3 \Rmax^2 S^2 A \cdot \beta^{\conc}(\delta, N) }{N} + SH^2 \Rmax^2 \varepsilon'.
    \end{align*}
    Combining these two upper bounds, we have
    \begin{align*}
        \E_{\hpi}\left[ \max_{a\in \cA} \left( \hQ^{\pi}_{\lambda,h} - Q^\pi_{\lambda,h} \right)^2(s_h,a) \mid s_1 \right] &\leq \frac{48 S^3 H^3 A \kappa^2 \log^2(SN) \beta^{\cH}(\delta)}{N} + \frac{4C H^5 \Rmax^2 S^2 A \cdot \beta^{\conc}(\delta,N)}{N} \\
        &+ S (H^2 \Rmax^2 + \kappa^2 \log^2(S))\varepsilon'.
    \end{align*}
    Since $\kappa^2 \log^2(s) \leq \Rmax^2$ and $H \geq 1$, we conclude the statement.
\end{proof}

\begin{lemma}\label{lem:change_measure}
    For any bounded function $f \colon \cS \times \cA \to \R_+, f(s,a) \leq B$ for any policy $\pi$ and step $h$ on event $\cE^{\RFExplore}(\delta, \varepsilon')$ the following holds
    \[
        \E_{\pi}\left[ f(s_h,a_h) | s_1 \right] \leq 2SAH\E_{(s,a) \sim \mu_h}\left[ f(s,a) \right] + BS \varepsilon'.
    \]
\end{lemma}
\begin{proof}
    Recall $S_{\varepsilon', h}$ be a set of all $\varepsilon'$-significant states (see Definition~\ref{def:significant_states}) at step $h$. Then we can rewrite this expectation as follows
    \[
        \E_{\pi}\left[ f(s_h,a_h) | s_1 \right] = \sum_{a \in \cA, s \in S_{\varepsilon', h}} d^\pi_h(s,a) f(s,a) + \sum_{a \in \cA, s \not \in S_{\varepsilon', h}} d^\pi_h(s,a) f(s,a).
    \]
    For the first sum by Theorem~\ref{th:rf_explore_sampling} we have $d^\pi_h(s,a) \leq 2SAH \mu_h(s,a)$, thus
    \[
        \sum_{a \in \cA, s \in S_{\varepsilon', h}} d^\pi_h(s,a) f(s,a)  \leq 2SAH \sum_{(s,a) \in \cS \times \cA} \mu_h(s,a) f(s,a) = 2SAH \E_{(s,a) \sim \mu_h}\left[ f(s,a) \right].
    \]
    For the second sum we apply $f(s,a) \leq B$ and the fact that for all states that are not $\varepsilon'$-significant under any policy $d^\pi_h(s) \leq \varepsilon'$:
    \[
         \sum_{a \in \cA, s \not \in S_{\varepsilon', h}} d^\pi_h(s,a) f(s,a) \leq B  \sum_{a \in \cA, s \not \in S_{\varepsilon', h}} d^\pi_h(s,a) = B \sum_{s \not \in S_{\varepsilon', h}} d^\pi_h(s) \leq BS \varepsilon'.
    \]
\end{proof}
\newpage
\section{Faster Rates for Visitation Entropy}
\label{app:reg_visitation_entropy_proofs}

\subsection{Algorithm description}
Let us start from the description of the modified algorithm \regalgMVEE. It has a similar game-theoretical foundation as it aims at solving the following minimax game
\begin{align*}
    \max_{d\in\cK_p} \VE(d) &= \max_{d\in\cK_p} \min_{\bd\in\cK}\sum_{(h,s,a)} d_h(s,a) \log \frac{1}{\bd_h(s,a)} =  \min_{\bd\in\cK} \max_{d\in\cK_p} \sum_{(h,s,a)} d_h(s,a) \log \frac{1}{\bd_h(s,a)}.
\end{align*}
As for usual \algMVEE, there are two players in the  game. On the one hand, the min player, or forecaster player, tries to predict which state-action pairs the max player will visit to minimize $\KL(d_h,\bd_h)$.  On the other hand, the max player, or sampler player, is rewarded for visiting state-action pairs that the forecaster player did not predict correctly.

We now describe the algorithm \regalgMVEE\ for MVEE. In this algorithm, we let a forecaster player and a sampler player compete for $T$ episodes long. Let us first define the two players.
\paragraph{Forecaster-player} The forecaster player remains exactly the same as for usual \algMVEE algorithm, see the corresponding section in the main text.

\paragraph{Regularized Sampler-player} For the sampler player we exploit strong convexity of visitation entropy. The running time of the sampler player will be divided onto two stages, as it was done in \RFExploreEnt algorithm.

\paragraph{Exploration phase}

Before the start of the game, the sampler-player uses some preprocessing time in order to explore the environment to learn a simple (non-markovian) preliminary exploration policy $\pi^{\mathrm{mix}}$. This policy is used to construct an accurate enough estimates of transition probabilities. This policy is obtained, as in \RFExplore and \RFExploreEnt, by learning for each state-action pair $(s,ah)$, a policy that reliably reaches this action pair $(s,a)$ at step $h$. This can be done by running any regret minimization algorithm for the sparse reward function putting reward one at state $s$ at step $h$ and zero otherwise. The policy $\pi^{\mathrm{mix}}$ is defined as the mixture of the aforementioned policies.

\paragraph{Planning phase}

The second phase is starting during the running time of the algorithm. Since \RFExploreEnt algorithm is essentially reward-free in a sense of working with an arbitrary reward functions.

For each episode $t$ during the game we define the empirical regularized Bellman equations 
\begin{align}
\begin{split}\label{eq:empirical_regularized_planning_VE}
\hQ_h^t(s,a) &=  \log\frac{1}{\bd_h^{t+1}(s)} + \hp_h^{\,t} \hV^t_{h+1}(s,a) \\
\hV_h^t(s) &= \max_{\pi \in \simplex_{\cA}}\{ \pi\hQ_h^t(s,a) + \cH(\pi) \},
\end{split}
\end{align}
where  $\hV_{H+1}^t = 0$. The sampler player then follows $d^{\pi^{t+1}}$ where $\pi^{t+1}$ is greedy with respect to the regularized Q-values, that is, $\pi_h^{t+1}(s) \in\argmax_{\pi\in\Delta_A} \{ \pi\hQ_h^t(s) + \cH(\pi) \}$. This choice of sampler player will be clear in the analysis below. 

\paragraph{Sampling rule} At each episode $t,$ the policy $\pi^t$ of the sampler-player is used as a sampling rule to generate a new trajectory.

\paragraph{Decision rule} After $T$ episodes we output a non-Markovian policy $\hpi$ defined as the mixture of the policies $\{\pi^t\}_{t\in[T]}$, that is, to obtain a trajectory from $\hpi$ we first sample uniformly at random $t\in[T]$ and then follow the policy $\pi^t$. Note that the visitation distribution of $\hpi$ is exactly the average $d^{\hpi} = (1/T)\sum_{t\in[T]} d^{\pi^t}$.

Remark that the stopping rule of \regalgMVEE is deterministic and equals to $\tau = T$. The complete procedure is detailed in Algorithm~\ref{alg:regMVEE}.

\begin{algorithm}[h!]
\centering
\caption{\regalgMVEE}
\label{alg:regMVEE}
\begin{algorithmic}[1]
  \STATE {\bfseries Input:} Number of episodes $T$, number of exploration episodes $N_0$, number of transition samples $N$, prior counts $n_0$.
  \STATE \textcolor{blue}{\# Preliminary exploration}
  \FOR{$(s',h') \in \cS \times [H]$}
        \STATE Form rewards $r_h(s,a) = \ind\{ s=s', h=h'\}$.
        \STATE Run \EULER \citep{zanette2019tighter} with rewards $r_h$ over $N_0$ iterates and collect all policies $\Pi_{s',h'}$.
        \STATE Modify $\pi \in \Pi_{s',h'}:\ \pi_{h'}(a|s') = 1/A$ for all $a\in \cA$.
    \ENDFOR
    \STATE Construct a uniform mixture policy $\pi^{\mathrm{mix}}$ over all $\{ \Pi_{s,h} : (s,h) \in \cS \times [H] \}$.
    \STATE Sample $N$ independent trajectories $\{z_n\}_{n\in[N]}$ using  $\pi^{\mathrm{mix}}$ in the original MDP.
    \STATE Construct from $\{z_n\}_{n\in[N]}$ the estimates $\hp_h$ as in \eqref{eq:hp_construction}.
      \FOR{$t \in[T]$}
      \STATE \textcolor{blue}{\# Forecaster-player}
      \STATE Update pseudo counts $\bn_h^{t-1}(s,a)$ and predict $\bd_h^t(s,a)$. 
      \STATE \textcolor{blue}{\# Sampler-player}
      \STATE Compute $\pi^t$ by regularized planning \eqref{eq:empirical_regularized_planning_VE} with rewards $\log\big(1/ \bd_h^t(s)\big)$ and entropy regularization.
    \STATE \textcolor{blue}{\# Sampling}
      \FOR{$h \in [H]$}
        \STATE Play $a_h^t\sim \pi_h^t(s_h^t)$
        \STATE Observe $s_{h+1}^t\sim p_h(s_h^t,a_h^t)$
      \ENDFOR
    \STATE{ Update counts and transition estimates.}
   \ENDFOR
   \STATE Output $\hpi$ the uniform mixture of $\{\pi^t\}_{t\in[T]}$.
\end{algorithmic}
\end{algorithm}

\subsection{Analysis}

We first define the regrets of each players obtained by playing $T$ times the games. For the forecaster-player, for any $\bd\in\cK,$ we define 
\[
\regret_{\fore}^T(\bd) \triangleq \sum_{t=1}^T \sum_{h,s,a} \td_h^t(s,a) \left(\log\frac{1}{\bd_h^t(s,a)} -\log\frac{1}{\bd_h(s,a)}\right)
\]
where $\td_h^t(s,a) \triangleq \ind\big\{(s_h^t,a_h^t)=(s,a)\big\}$ is a sample from $d_h^{\pi^t}(s,a)$.
Similarly for the sampler-player, for any $d\in\cK_p,$ we define a \textit{regularized regret}
\begin{small}
\[
\regret_{\samp}^T(d)\triangleq \sum_{t=1}^T \left( \sum_{h,s,a} \big[ d_h(s,a) - d_h^{\pi^t}(s,a) \big] \log\frac{1}{\bd_h^t(s,a)} - \sum_{h,s} \left[ d_h(s) \KL(\pi(s), \bar{\pi}^t_h(s)) -  d^{\pi^t}_h(s) \KL(\pi^t_h(s), \bar{\pi}^t_h(s)) \right] \right)\,,
\]
\end{small}
\!where corresponding policies are defined as $\pi_h(a|s) = d_h(s,a) / d_h(s)$ and $\bar{\pi}^t_h(a|s) = \bd^t_h(s,a) / \bd^t_h(s)$ for $d_h(s) = \sum_{a} d_h(s,a)$ and $\bd^t_h(s) = \sum_{a} \bd^t_h(s,a)$.

Recall that the visitation distribution of the policy $\pi$ returned by \regalgMVEE is the average of the visitation distributions of the sampler-player 
$d_h^{\hpi}(s,a) = \hd^{\,T}_h(s,a) \triangleq (1/T) \sum_{t=1}^T d_h^{\pi^t}(s,a)$.  We also denote by $\rd^T_h(s,a)\triangleq (1/T) \sum_{t=1}^T \td^t(s,a)$ the average of the 'sample' visitation distributions.

We now relate the difference between the optimal visitation entropy and the visitation entropy of the outputted policy $\hpi$ to the regrets of the two players. 
Indeed, using $\cH(p) = \sum_{i\in[n]} p_i \log(1/q_i) -\KL(p,q)$ for all $(p,q)\in(\Delta_n)^2$ and
\begin{align*}
    \KL(d^{\pistar}_h, \bd_h^t) &= \sum_{s,a} d^{\pistar}_h(s,a) \log\left( \frac{d^{\pistar}_h(s,a)}{\bd_h^t(s,a)} \right) \\
    &= \sum_{s} d^{\pistar}_h(s) \log\left( \frac{d^{\pistar}_h(s)}{\bd_h^t(s)} \right) + \sum_{s} d^{\pistar}_h(s) \sum_{a} \pistar_h(a|s) \log\left( \frac{\pistar_h(a|s)}{\bar{\pi}_h^t(a|s)} \right) \\
    &\geq \sum_{s} d^{\pistar}_h(s) \KL(\pistar_h(s), \bar{\pi}^t_h(s))\,.
\end{align*}    
This inequality could be treated as a strong convexity of visitation entropy with respect to trajectory entropy since $\KL(d^{\pistar}_h, \bd^t_h)$ is a \textit{Bregman divergence} with respect to $\VE$, and the final average of $\KL(\pistar_h(s), \bpi^t_h(s))$ is a Bregman divergence with respect to $\TE$ (up to linearities).

Applying this inequality, we have
\begin{align*}
T\big(\VE(d^{\pistar}) -\VE(d^{\hpi})\big) &\leq \sum_{t=1}^T \left( \sum_{h,a,s} d_h^{\pistar}(s,a) \log\frac{1}{\bd_h^t(s,a)} - \sum_{h,s}  d_h^{\pistar}(s) \KL(\pistar_h(s), \bar{\pi}^t_h(s))  \right) \\
& \quad- \sum_{t=1}^T \td^t_h(s,a) \log\frac{1}{\rd_h^{\,T}(s,a)} + T\big(\VE(\rd^T) - \VE(\hd^T)\big)\\
& \leq \regret_{\samp}^T(d^{\pistar})+ \underbrace{\sum_{t=1}^T \sum_{h,s,a} \big(d_h^{\pi^t}(s,a) - \td_h^t(s,a) \big) \log\frac{1}{\bd_h^t(s,a)}}_{\mathrm{Bias}_1} \\
& \quad + \regret_{\fore}^T(\rd^T) + \underbrace{T\big(\VE(\rd^T) - \VE(\hd^T)\big)}_{\mathrm{Bias}_2}\,.
\end{align*}
It remains to upper bound each terms separately in order to obtain a bound on the gap. Notably, only the sampler player result changes in comparison to \algMVEE.

\subsection{Regret of the Sampler-Player}

We start from introducing new notation. Let $\cM_t = (\cS, \cA, \{ p_h \}_{h\in[H]}, \{r^t_h\}_{h\in[H]}, s_1)$ be a sequence of entropy-regularized MDPs where reward defined as $r^t_h(s,a) = \log(1/ \bd^t_h(s))$. Define $Q^{\pi, t}_h(s,a)$ and $V^{\pi, t}_h(s,a)$ as a action-value and value functions of a policy $\pi$ on a MDP $\cM_t$. Notice that the value-function of initial state in this case could be written as follows (see Appendix~\ref{app:reg_bellman_eq})
\begin{align*}
    V^{\pi,t}_1(s_1) &= \sum_{h,s,a} d^{\pi}_h(s,a) \log\left( \frac{1}{\bd^t_h(s,a)} \right) - \sum_{h,s} d^\pi_h(s) \KL(\pi_h(s), \bar{\pi}^t_h(s)) \\
    &= \sum_{h,s,a} d^\pi_h(s,a) \log\left( \frac{1}{\bd^t_h(s)} \right) + \sum_{h,s} d^\pi_h(s) \cH(\pi_h(s)),
\end{align*}
also see Appendix~\ref{app:regularized_mdp} for more exposition. Therefore, the regret for the sampler-player could be rewritten in the terms of the regret for this sequence of entropy-regularized MDPs
\[
    \regret_{\samp}^T(d^\pi) =  \sum_{t=1}^T V^{\pi,t}_1(s_1) - V^{\pi^t,t}_1(s_1).
\]

We notice that our approach does not gives a regret minimizer algorithm in a classical sense, however analysis shows us that we can control the sum of policy error with respect to \textit{any} reward function.

\begin{lemma}\label{lem:reg_regret_sampler}
    Let $N_0 = \Omega\left( \frac{H^7 S^3 A  \cdot \log^2(T+SA) \cdot L}{\varepsilon}\right)$ and $N = \Omega\left( \frac{ H^6 S^3 A \log^2(T+SA) L^3 }{\varepsilon}\right).$ Then with probability at least $1-\delta/2,$ the regret of the sampler player is bounded as
    \[
        \regret_{\samp}^T(d^{\pistar}) \leq \varepsilon/2 \cdot T
    \]
    after 
    \[
         \tcO\left( \frac{H^8 S^4 A}{\varepsilon} \right)
    \]
    episodes of pure exploration.
\end{lemma}
\begin{proof}
    From  Corollary~\ref{cor:rf_explore_ent_rf_sample_complexity} under the choice of parameters $\lambda = 1, \kappa = 0$ and reward function $r^t_h(s,a) = \log(1/\bd^t_h(s))$ for each iteration, that is bounded by $\log(T+SA)$, we have that for any reward function the sub-optimality gap is bounded by $\varepsilon/2$. The total number of episodes of pure exploration is equal to $N_0SH + N$.
\end{proof}

\subsection{Proof of Theorem~\ref{th:fast_MVEE_sample_complexity}}

We state the version of this theorem with all prescribed dependencies factors.
\begin{theorem}\label{th:fast_MVEE_sample_complexity_full}
Fix some $\epsilon > 0$ and $\delta\in(0,1)$. Then for $n_0=1,$ 
\[
    N_0 = \Omega\left( \frac{H^7 S^3 A  \cdot L^3}{\varepsilon}\right),
    \quad
    N = \Omega\left( \frac{ H^6 S^3 A L^5 }{\varepsilon}\right), 
    \quad
    T = \Omega\left( \frac{H^2 S A L^3}{\varepsilon^2} + \frac{H^2 S^2 A^2 L^2}{\varepsilon}\right)
    \]
with $L = \log(SAH/\delta\varepsilon),$ the algorithm \regalgMVEE is $(\epsilon,\delta)$-PAC. Its total sample complexity is equal to $SH \cdot N_0 + N + T,$ that is,
\[
    \tau = \tcO\left( \frac{H^2 SA}{\varepsilon^2} + \frac{H^8 S^4 A}{\varepsilon} \right).
\]
\end{theorem}
\begin{proof}
    We start from writing down the decomposition defined in the beginning of the appendix
    \[
        T(\VE(d^{\pistarVE}) - \VE(d^{\hpi})) \leq \regret_{\samp}^T(d^{\pistarVE}) + \regret_{\fore}^T(\rd^T) + \mathrm{Bias}_1 + \mathrm{Bias}_2.
    \]
    By Lemma~\ref{lem:reg_regret_sampler} with probability at least $1-\delta/2$ it holds
    \[
        \regret_{\samp}^T(d^{\pistarVE}) \leq \varepsilon T/2.
    \]
    By Lemma~\ref{lem:regret_forecaster} 
    \[
        \regret_{\fore}^T(\rd^T) \leq HSA\log\big(\rme(T+1)\big).
    \]
    By Lemma~\ref{lem:bias_terms} with probability at least $1-\delta/2$
    \[
         \mathrm{Bias}_1 + \mathrm{Bias}_2 \leq 3\log(SAT)\left(\sqrt{TH\log(4/\delta)} + H\sqrt{SAT\log(3T)} \right).
    \]
    By union bound all these inequalities hold simultaneously with probability at least $1-\delta$. Combining all these bounds we get
    \begin{align*}
        T(\VE(d^{\pistarVE}) - \VE(d^{\hpi})) &\leq  3\log(SAT)\left(\sqrt{TH\log(4/\delta)} + H\sqrt{SAT\log(3T)} \right) \\
        &+ HSA\log(\rme (T+1)) + \varepsilon T / 2.
    \end{align*}
    Therefore, it is enough to choose $T$ such that $\VE(d^{\pistarVE}) - \VE(d^{\hpi})$ is guaranteed to be less than $\varepsilon$. In this case \regalgMVEE become automatically $(\varepsilon,\delta)$-PAC. It is equivalent to find a maximal $T$ such that
    \begin{align*}
        \varepsilon T/2 &\leq 3\log(SAT)\left(\sqrt{TH\log(4/\delta)} + H\sqrt{SAT\log(3T)} \right) +  HSA \log(\rme (T+1))).
    \end{align*}
    and add $1$ to it. We start from obtaining a loose bound to eliminate logarithmic factors in $T$.
    
    First, we assume that $T \geq 1$, thus $T+1 \leq 2T$. Additionally, let us use inequality $\log(x) \leq x^{\beta}/\beta$ for any $x > 0$ and $\beta > 0$. We obtain
    \[
        \varepsilon T \leq 48(SAT)^{1/8}\left(\sqrt{T^{3/2} H\log(4/\delta)} + H\sqrt{4 SAT^{3/2}} \right) +  16/7 \cdot HSA (2\rme T)^{7/8}
    \]
    that could be relaxed as follows
    \[
         \varepsilon T^{1/8} \leq 48 (SA)^{1/8} (H \log(4/\delta)^{1/2} + H (SA)^{5/8} + 11 HSA
    \]
    thus we can define $\gamma = 8 \log\left( (48(SA)^{1/8} (H \log(4/\delta)^{1/2} + H (SA)^{5/8} + 11 HSA)/\varepsilon \right) = \cO(L)$ for which $\log(T) \leq \gamma$. Therefore
    \[ 
        \varepsilon T/2 \leq 3 (\log(SA) + \gamma) \sqrt{T} \left( \sqrt{H\log(4/\delta)} + \sqrt{SAH^2 (\log(3) + L)} \right) + HSA (1 + 2\gamma).
    \]
    Solving this quadratic inequality, we obtain the minimal required $T$ to guarantee $\VE(d^{\pistarVE}) - \VE(d^{\hpi}) \leq \varepsilon$. In particular,
    \[
        T = \Omega\left( \frac{H^2 S A L^3}{\varepsilon^2} + \frac{H^2 S^2 A^2 L^2}{\varepsilon}\right).
    \]
\end{proof}
\newpage
%!TEX root = ../BayesUCBVI.tex
\section{Deviation Inequalities}
\label{app:deviation_ineq}

\subsection{Deviation Inequality for Categorical Distributions}

Next, we state the deviation inequality for categorical distributions by \citet[Proposition 1]{jonsson2020planning}.
Let $(X_t)_{t\in\N^\star}$ be i.i.d.\,samples from a distribution supported on $\{1,\ldots,m\}$, of probabilities given by $p\in\simplex_{m-1}$, where $\simplex_{m-1}$ is the probability simplex of dimension $m-1$. We denote by $\hp_n$ the empirical vector of probabilities, i.e., for all $k\in\{1,\ldots,m\},$
 \[
 \hp_{n,k} \triangleq \frac{1}{n} \sum_{\ell=1}^n \ind\left\{X_\ell = k\right\}.
 \]
 Note that  an element $p \in \simplex_{m-1}$ can be seen as an element of $\R^{m-1}$ since $p_m = 1- \sum_{k=1}^{m-1} p_k$. This will be clear from the context. 
%  We denote by $H(p)$ the (Shannon) entropy of $p\in\Sigma_m$,
%  \[
%  H(p) = \sum_{k=1}^m p_k \log\left(\frac{1}{p_k}\right)\cdot
%  \]
 \begin{theorem} \label{th:max_ineq_categorical}
 For all $p\in\simplex_{m-1}$ and for all $\delta\in[0,1]$,
 \begin{align*}
     \P\left(\exists n\in \N^\star,\, n\KL(\hp_n, p)> \log(1/\delta) + (m-1)\log\left(e(1+n/(m-1))\right)\right)\leq \delta.
 \end{align*}
\end{theorem}

\subsection{Deviation Inequality for Shannon Entropy}

 We denote by $\cH(p)$ the (Shannon) entropy of $p\in \simplex_{m-1}$,
 \[
    \cH(p) \triangleq \sum_{k=1}^m p_k \log\left(\frac{1}{p_k}\right) .
 \]
We will follow the ideas of \citet{paninski2003estimation}.
 
\begin{theorem}\label{th:entropy_concentration}
    For all $p \in \simplex_{m-1}$ and for all $\delta \in[0,1]$
    \[
        \P\left[ \vert \cH(\hp_n) - \cH(p) \vert \geq  \sqrt{\frac{2 \log^2(n) \cdot \log(2/\delta)}{n}} + \left( \frac{(m-1) \log(\rme (1 + n/(m-1))) + 1}{n} \wedge \log(m) \right) \right] \leq \delta.
    \]
    Moreover,
    \[
        \P\left[ \exists n:  \vert \cH(\hp_n) - \cH(p) \vert \geq  \sqrt{\frac{2 \log^2(n) \cdot (\log(2/\delta) + \log(n(n+1)))}{n}} + \left( \frac{m \log(\rme (1 + n))}{n} \wedge \log(m) \right)\right] \leq \delta.
    \]
\end{theorem}
\begin{proof}
    We start from application of McDiarmid's inequality to entropy by \citet{antos2001convergence}.
    For all $p \in \Delta_{m-1}$ with probability at least $1-\delta$ we have
    \[
        \P\left[ \vert \cH(\hat p_n) - \E[\cH(\hat p_n)]  \vert \geq \sqrt{\frac{2 \log^2(n) \cdot \log(2/\delta)}{n}}  \right] \leq \delta.
    \]
    To relate $\E[\cH(\hp_n)]$ and $\cH(p)$ we use the following observation
    \[
        \cH(\hp_n) - \cH(p) = - \KL(\hp_n, p) + \sum_{k: p_k > 0} (\hp_{n,k} - p_k) \log(1/p_k),
    \]
    therefore by taking expectation we have
    \[
        \E[\cH(\hp_n)] - \cH(p) = - \E[\KL(\hp_n, p)].
    \]
    In the following our analysis differs from \cite{paninski2003estimation} since we obtain a direct estimate on the KL-divergence using Theorem~\ref{th:max_ineq_categorical},
    \begin{align*}
        \E[n \KL(\hp_n,p)] &= \int_{0}^\infty \P[n \KL(\hp_n, p) > t]\rmd t \leq (m-1) \log(\rme(1 + n/(m-1))) + \int_0^\infty \rme^{-t} \rmd t.
    \end{align*}
    At the same time we have a trivial bound that concludes the first statement
    \[
        \E[\cH(\hp^n)] - \cH(p) \leq \log(m).
    \]

    To show the second statement of Theorem~\ref{th:entropy_concentration}, we apply the first part with $\delta'(n) = \delta/( n(n+1))$,
    \begin{small}
    \[
        \P\left[ \vert \cH(\hp_n) - \cH(p) \vert \geq  \sqrt{\frac{2 \log^2(n) \cdot \log(2/\delta'(n))}{n}} + \left( \frac{(m-1) \log(\rme (1 + n/(m-1))) + 1}{n} \wedge \log(m) \right)\right] \leq \frac{\delta}{n(n+1)},
    \]
    \end{small}
    thus by union bound over $n \in \N$ we conclude the statement.
\end{proof}

\subsection{Deviation Inequality for Sequence of Bernoulli Random Variables}

Below, we state the deviation inequality for Bernoulli distributions by \citet[Lemma F.4]{dann2017unifying}.
Let $\mathcal F_t$ for $t\in\N$ be a filtration and $(X_t)_{t\in\N^\star}$ be a sequence of Bernoulli random variables with $\P(X_t = 1 | \mathcal F_{t-1}) = P_t$ with $P_t$ being $\mathcal F_{t-1}$-measurable and $X_t$ being $\mathcal F_{t}$-measurable.

\begin{theorem}\label{th:bernoulli-deviation}
	For all $\delta>0$,
	\begin{align*}
	\P \left(\exists n : \,\, \sum_{t=1}^n X_t < \sum_{t=1}^n P_t / 2 -\log\frac{1}{\delta}  \right) \leq \delta.
	\end{align*}
\end{theorem}

\subsection{Deviation Inequality for Bounded Distributions}
Below, we state the self-normalized Bernstein-type inequality by \citet{domingues2020regret}. Let $(Y_t)_{t\in\N^\star}$, $(w_t)_{t\in\N^\star}$ be two sequences of random variables adapted to a filtration $(\cF_t)_{t\in\N}$. We assume that the weights are in the unit interval $w_t\in[0,1]$ and predictable, i.e. $\cF_{t-1}$ measurable. We also assume that the random variables $Y_t$  are bounded $|Y_t|\leq b$ and centered $\EEc{Y_t}{\cF_{t-1}} = 0$.
Consider the following quantities
\begin{align*}
		S_t \triangleq \sum_{s=1}^t w_s Y_s, \quad V_t \triangleq \sum_{s=1}^t w_s^2\cdot\EEc{Y_s^2}{\cF_{s-1}}, \quad \mbox{and} \quad W_t \triangleq \sum_{s=1}^t w_s
\end{align*}
and let $h(x) \triangleq (x+1) \log(x+1)-x$ be the Cramér transform of a Poisson distribution of parameter~1.

\begin{theorem}[Bernstein-type concentration inequality]
  \label{th:bernstein}
	For all $\delta >0$,
	\begin{align*}
		\PP{\exists t\geq 1,   (V_t/b^2+1)h\left(\!\frac{b |S_t|}{V_t+b^2}\right) \geq \log(1/\delta) + \log\left(4e(2t+1)\!\right)}\leq \delta.
	\end{align*}
  The previous inequality can be weakened to obtain a more explicit bound: if $b\geq 1$ with probability at least $1-\delta$, for all $t\geq 1$,
 \[
 |S_t|\leq \sqrt{2V_t \log\left(4e(2t+1)/\delta\right)}+ 3b\log\left(4e(2t+1)/\delta\right)\,.
 \]
\end{theorem}

\subsection{Deviation Inequalities for Expectation over Sampling Measure}

In this section we describe two inequalities that shows the deviations of quantities of empirical transition kernels constructed by an independent samples from some base distribution over states. Let $\{z_k\}_{k\in[N]}$ be a set of independent trajectories using a fixed policy $\pi$ in the original MDP: $\forall i \in [H]: a_i \sim \pi(s_i), s_{i+1} \sim p(s_i, a_i)$, and let $\mu_h(s,a) \triangleq d^\pi_h(s,a)$ be its state-action visitation distribution. 

Using this data, we construct an estimates of transition probabilities $\{\hp_h\}_{h\in[H]}$: for each state-action-step triple $(s,a,h)$ we define $n_h(s,a)$ as a number of visits in these $N$ trajectories:
\[
    n_h(s,a) = \sum_{k=1}^N \ind\{ (s^k_h, a^k_h) = (s,a) \} \quad  n_h(s'|s,a) = \sum_{k=1}^N \ind\{ (s^k_h, a^k_h, s^k_{h+1}) = (s,a,s') \},
\]
where $z_k = (s^k_1, a^k_1, \ldots, s^k_H, a^k_H, s^k_{H+1})$, and then the model constructed in a usual way as a maximum likelihood estimate
\[
    \hp_h(s'|s,a) = \begin{cases}
        \frac{n_h(s'|s,a)}{n_h(s,a)} & n_h(s,a) > 0 \\
        \frac{1}{S} & n_h(s,a) = 0
    \end{cases}\,.
\]

\begin{lemma}\label{lem:sampling_entropy_bound}
    Suppose that $\hp_h$ is the empirical transitions formed using $N$ independent trajectories $\{z_k\}_{k\in[N]}$ sampled independently using policy $\pi$ in the original MDP. Then we probability at least $1-\delta$ the following holds
    \[
        \forall h \in [H]: \E_{(s,a) \sim \mu_{h}}\left[ (\cH(\hp_{h}(s,a)) - \cH(p_{h}(s,a)))^2 \right] \leq \frac{12S^2 A \log^2(SN) \cdot \log(4SAH/\delta)}{N},
    \]
    where $\mu_h(s,a) = d^\pi_h(s,a)$.
\end{lemma}
\begin{proof}
    Define $n_h(s,a)$ be a number of samples from a fixed state-action pair $s,a$. Then by Theorem~\ref{th:entropy_concentration} we have with probability at least $1-\delta/2$ for all triples $(s,a,h) \in \cS \times \cA \times [H]$
    \[
        (\cH(\hp_{h}(s,a)) - \cH(p_{h}(s,a)))^2 \leq \frac{2 \log^2(n_h(s,a)) \cdot \log\left( 4SAH/\delta \right)  }{n_h(s,a)} + \frac{S \log(S) \log(n_h(s,a)) }{n_h(s,a)}.
    \]
    From other point of view, we have a trivial upper bound $\log^2(S)$. Thus, we obtain
    \[
        (\cH(\hp_{h}(s,a)) - \cH(p_{h}(s,a)))^2 \leq \frac{2 \log^2(S n_h(s,a)) \cdot \log\left( 4SAH/\delta \right) + S \log(S) \log(n_h(s,a) }{n_h(s,a)} \wedge 1.
    \]
    Next we define the following event
    \[
        \cE^{\cnt}(\delta) = \left\{ \forall (s,a,h) \in \cS \times \cA \times [H]:  n_h(s,a) \geq \frac{N}{2} \mu_h(s,a) - \log(2SAH/\delta) \right\}.
    \]
    By Theorem~\ref{th:bernoulli-deviation} it holds with probability at least $1-\delta/2$, then we can apply Lemma~\ref{lem:cnt_pseudo} and obtain the following bound with probability at least $1-\delta$ by union bound
    \[
        (\cH(\hp_{h}(s,a)) - \cH(p_{h}(s,a)))^2 \leq 4\frac{2 \log^2(SN) \cdot \log\left( 4SAH/\delta \right) + S \log(S) \log(N)}{(N \mu_h(s,a)) \vee 1}.
    \]
    Thus, taking expectation we have
    \begin{align*}
        \E_{(s,a) \sim \mu_{h}}\left[ (\cH(\hp_{h}(s,a)) - \cH(p_{h}(s,a)))^2 \right] &\leq \sum_{s,a} \frac{\mu_h(s,a)}{(N \mu_h(s,a)) \vee 1} \cdot 4\left(2\log^2(SN) \log(4SAH/\delta) + S \log(S)\log(N)\right) \\
        &\leq \frac{12S^2A \cdot \log^2(SN) \cdot \log(4SAH/\delta)}{N}.
    \end{align*}
\end{proof}

\begin{lemma}\label{lem:sampling_square_value_error_bound}[Lemma C.2 by \citealt{jin2020reward-free}]
      Suppose that $\hp_h$ is the empirical transitions formed using $N$ independent trajectories $\{z_k\}_{k\in[N]}$ sampled independently using policy $\pi$ in the original MDP.  Then there is an absolute constant $C > 0$ such that with probability at least $1-\delta$ the following holds for all $h \in [H]$
    \begin{align*}
        \max_{G \colon \cS \to [0,H \Rmax]} \max_{\nu \colon \cS \to \cA} &\E_{(s,a) \sim \mu_{h}}\left[ \big(\left[ \hp_{h'} - p_{h'}  \right] G(s,a) \big)^2 \ind\{\nu(s) = a \} \right] \leq  \frac{CH^2 \Rmax^2 S}{N} \log\left( \frac{AH\Rmax N}{\delta} \right),
    \end{align*}
    where $\mu_h(s,a) = d^\pi_h(s,a)$.
\end{lemma}
\newpage
%!TEX root = ../BayesUCBVI.tex
\section{Technical Lemmas}
\label{app:technical}

\subsection{Entropy Properties}

\begin{lemma}
\label{lem:prediction_game}
For any reward-free MDP $\cM = (\cS, \cA, \{p_h\}_{h\in[H]}, s_1)$ the following holds
\begin{align*}
    \max_{d\in\cK_p} \VE(d) &= \max_{d\in\cK_p} \min_{\bd\in\cK}\sum_{(h,s,a)} d_h(s,a) \log \frac{1}{\bd_h(s,a)}\\
    &=  \min_{\bd\in\cK} \max_{d\in\cK_p} \sum_{(h,s,a)} d_h(s,a) \log \frac{1}{\bd_h(s,a)}
\end{align*}
\end{lemma}
\begin{proof}
The first equality is due to 
\begin{align*}
    \VE(d)=\sum_{h=1}^H \cH(d_h) = \sum_{h=1}^H \cH(d_h) +\KL(d_h,d_h)
= \min_{\bd\in\cK}
\sum_{h=1}^H (\cH(d_h) +\KL(d_h,\bd_h))
= \min_{\bd\in\cK}
\sum_{(h,s,a)} d_h(s,a) \log \frac{1}{\bd_h(s,a)}
\end{align*}
and the second equality uses Sion's theorem \citep{sion1958general} with the fact that $\cK,\cK_p$ are compact convex sets and $(d,\bd) \mapsto -\sum_{h,s,a} d_h(s,a) \log \bd_h (s,a)$ is concave-convex.
\end{proof}

\begin{lemma}\label{lem:trajectory_entropy_formula}
    The trajectory entropy $\TE(\pi)$ has the following representation
    \[
        \TE(\pi) = \sum_{m \in \cT} q^\pi(m) \log \frac{1}{q^\pi(m)} = \E_{\pi}\left[ \sum_{h=1}^H \cH(p_h(s_h,a_h)) + \cH(\pi_h(s_h)) | s_1 \right].
    \]
\end{lemma}
\begin{proof}
    By a definition of $q^\pi(m)$:
    \[
        q^\pi(m) = \pi_1(a_1|s_1) \prod_{h=2}^{H} p_{h-1}(s_h|s_{h-1},a_{h-1}) \pi_{h}(a_{h} | s_h),
    \]
    thus
    \[
        \sum_{m \in \cT} q^\pi(m) \log \frac{1}{q^\pi(m)} = -\sum_{m \in \cT} q^\pi(m)\left( \log(\pi_1(a_1|s_1) + \sum_{h=2}^H \log(p_{h-1}(s_h|s_{h-1}, a_{h-1}) + \log(\pi_{h}(a_{h} | s_h)) \right).
    \]
    By rearranging the terms we have
    \[
        \TE(\pi) = \sum_{h=1}^H \sum_{m \in \cT} q^\pi_h(m)\left( \log  \frac{1}{p_h(s_{h+1}|s_h, a_h)} + \log \frac{1}{\pi_h(a_{h}|s_h)} \right),
    \]
    where under convention $p_H$ is a deterministic transition to $s_1$. Marginalizing $q^\pi_h(m)$ over $s_h,a_h,s_{h+1}$ we get
    \begin{align*}
        \TE(\pi) &= \sum_{h=1}^H \sum_{s,a,s'} d^\pi(s,a) p_h(s'|s,a) \left( \log  \frac{1}{p_h(s'|s, a)} + \log \frac{1}{\pi_h(a|s)} \right) \\
        &= \sum_{s,a} d^\pi_h(s,a) \sum_{h=1}^H \cH(p_h(s,a)) + \sum_s d^\pi_h(s) \cH(\pi_h(s)).
    \end{align*}
\end{proof}

\begin{lemma}
\label{lem:comparison_ent_traj_visit}
For any reward-free MDP $\cM = (\cS, \cA, \{p_h\}_{h\in[H]}, s_1)$ and any policy $\pi$ it holds
\begin{align*}
\TE(q^\pi)\leq \VE(d^\pi)\leq H \TE(q^\pi).
\end{align*}
\end{lemma}
\begin{proof}
The first inequality is a result of the non-negativity of mutual information, specifically the difference $\VE(d^\pi)-\TE(q^\pi)$ is equal to the Kullback-Leibler divergence between $q^\pi$ and the product distribution $\otimes_{h=1}^H d^\pi_h$, 
\begin{align*}
\VE(d^\pi)&=
\sum_{s,a,h}d^\pi_h(s,a) 
 \log\frac{1}{d^\pi_h(s,a)}
\\&= 
 \sum_{s,a,h}~\sum_{m=(s_1,a_1,\dots,s_H,a_H)\in\cT}\ind\{(s,a)=(s_h,a_h)\} q^\pi(m) 
 \log\frac{1}{d^\pi_h(s,a)}
 \\&= \sum_{h=1}^H~\sum_{m=(s_1,a_1,\dots,s_H,a_H)\in\cT} q^\pi(m)\log\frac{1}{d^\pi_h(s_h,a_h)} 
 \\&= \sum_{m=(s_1,a_1,\dots,s_H,a_H)\in\cT} q^\pi(m)\log\frac{1}{\prod_{h=1}^H d^\pi_h(s_h,a_h)} 
 \\&= \KL(q^\pi,\otimes_{h=1}^H d^\pi_h) + \TE(q^\pi) \\&\geq \TE(q^\pi).
\end{align*}
The second inequality is simply a consequence of the 
monotonicity of logarithm
\begin{align*}
\TE(q^\pi) &= \sum_{m\in\cT} q^\pi(m) \log\frac{1}{q^\pi(m)} 
\\&=
 \sum_{s,a,h}\frac{1}{H}\sum_{m=(s_1,a_1,\dots,s_H,a_H)\in\cT}\ind\{(s,a)=(s_h,a_h)\} q^\pi(m) \log\left(\frac{1}{q^\pi(m)}\right)
 \\&\geq
 \sum_{s,a,h}\frac{1}{H}\sum_{m=(s_1,a_1,\dots,s_H,a_H)\in\cT}\ind\{(s,a)=(s_h,a_h)\} q^\pi(m) \\
 &\qquad\qquad\qquad\qquad\qquad\qquad\qquad\cdot\log\left(\frac{1}{\sum_{m'=(s'_1,a'_1,\dots,s'_H,a'_H)\in\cT}\ind\{(s,a)=(s'_h,a'_h)\} q^\pi(m')} \right)
 \\&=
  \frac{1}{H} \VE(d^\pi),
\end{align*}
where we used that $\log(1/x) = -\log(x)$ is monotonically decreasing, and 
\[
    \sum_{m'=(s'_1,a'_1,\dots,s'_H,a'_H)\in\cT}\ind\{(s,a)=(s'_h,a'_h)\} q^\pi(m') \geq q^\pi(m)
\]
for any $m = (s_1,a_1,\dots,s_H,a_H)$ such that $(s_h,a_h) = (s,a)$.
\end{proof}

\begin{lemma}
\label{lem:MTEE_deterministic}
In a deterministic MDP, i.e. the transition probability distributions are deterministic, the maximum trajectory entropy policy is the uniform policy over actions
\[
\pistarTE_h(s,a) =1/A\,.
\]
\end{lemma}
\begin{proof}
    The proof is straightforward by using the Bellman equations for MTEE. By induction over the step $h$ we prove that $\pistarTE_h(s,a)=1/A$ and $\Vstar_h(s) = (H+1-h)\log(A)$. Assume the induction hypothesis at step $h+1$. Then we get $\Qstar_h(s,a) = 0 + p_h \Vstar_{h+1}(s,a) = (H-h)\log(A)$ which yields $\pistarTE_h(s,a)=1/A$ and $\Vstar_h(s) = \log\big(\sum_{a\in\cA} A^{H-h}\big) = (H+1-h)\log(A)$. The base case $h=H$ is immediate.
    \end{proof}

\subsection{Counts to Pseudo-counts}
\label{app:count}
Here we state Lemma~8 and Lemma~9 by \citet{menard2021fast}.
\begin{lemma}\label{lem:cnt_pseudo} On event $\cE^{\text{\normalfont cnt}}$,  for any $\beta(\delta, \cdot)$ such that $x \mapsto \beta(\delta,x)/x$ is non-increasing for $x\geq 1$,  $x \mapsto \beta(\delta,x)$ is non-decreasing $\forall  h \in [H], (s,a) \in \cS \times \cA$,
\[ \forall t \in \N^\star, \ \frac{\beta(\delta, n_h^t(s,a))}{n_h^t(s,a)}\wedge 1 \leq 4 \frac{\beta(\delta, \bar n_h^t(s,a))}{\bar n_h^t(s,a)\vee 1}\cdot\]
\end{lemma}
% \begin{proof}
% As event $\cE^{\mathrm{\normalfont cnt}}$ holds, we know that for all $t < \tau$,
% \begin{align*}n_{\ell}^{t}(s,a) \geq \frac{1}{2}\bar n_{\ell}^{t}(s,a) - \beta^{\cnt}(\delta).
% \end{align*}
% We now distinguish two cases. First, if $\beta^{\cnt}(\delta) \leq \tfrac{1}{4}\bar n_{\ell}^{t}(s,a)$, then \[\frac {\beta(n_{\ell}^t(s,a), \delta)} {n_{\ell}^t(s,a)}\wedge 1 \leq \frac {\beta(n_{\ell}^t(s,a), \delta)} {n_{\ell}^t(s,a)}\leq \frac {\beta\left(\tfrac{1}{4}\bar n_{\ell}^{t}(s,a), \delta\right)} {\tfrac{1}{4}\bar n_{\ell}^{t}(s,a)} \leq 4   \frac {\beta\left(\bar n_{\ell}^{t}(s,a), \delta\right)} {\bar n_{\ell}^{t}(s,a) \vee 1}\CommaBin\]
% where we used that $x \mapsto \beta(x,\delta)/x$ is non-increasing for $x\geq 1$,  $x \mapsto \beta(x,\delta)$ is non-decreasing, and $\beta^{\cnt}(\delta) \geq 1$. Second,
% if $\beta^{\cnt}(\delta) > \tfrac{1}{4}\bar n_{\ell}^{t}(s,a)$, a simple derivation gives that
% \[\frac {\beta(n_{\ell}^t(s,a), \delta)} {n_{\ell}^t(s,a)}\wedge 1\leq   1 < 4 \frac{\beta^{\cnt}(\delta)}{\bar n_{\ell}^{t}(s,a) \vee 1} \leq 4 \frac{\beta(\bar n_{\ell}^{t}(s,a), \delta)}{\bar n_{\ell}^{t}(s,a) \vee 1}\CommaBin\]
% where we used that $1 \leq \beta^{\cnt}(\delta)\leq \beta(0,\delta)$ and $x \mapsto \beta(x,\delta)$ is non-decreasing.
% \end{proof}

\begin{lemma}
	\label{lem:sum_1_over_n}
	 For $T\in\N^\star$ and $(u_t)_{t\in\N^\star},$ for a sequence where  $u_t\in[0,1]$ and $U_t \triangleq \sum_{l=1}^t u_\ell$, we get
	\[
		\sum_{t=0}^T \frac{u_{t+1}}{U_t\vee 1} \leq 4\log(U_{T+1}+1).
	\]
\end{lemma}
% \begin{proof}
% 	Notice that	\begin{align*}
% 		\sum_{t=0}^T \frac{u_{t+1}}{U_t\vee 1} &\leq 4 \sum_{t=0}^T \frac{u_{t+1} }{2U_t + 2} \\
% 		&\leq  4\sum_{t=0}^T \frac{U_{t+1}-U_{t}}{U_{t+1} + 1}\\
% 		&\leq 4\sum_{t=0}^T \int_{U_t}^{U_{t+1}} \frac{1}{x+1} \mathrm{d}x\\
% 		& = 4\log(U_{T+1}+1).
% 	\end{align*}
% \end{proof}

\subsection{On the Bernstein Inequality}
\label{app:Bernstein}
We restate here a Bernstein-type inequality by \citet{talebi2018variance}.
\begin{lemma}[Corollary 11 by \citealp{talebi2018variance}]\label{lem:Bernstein_via_kl}
Let $p,q\in\simplex_{S},$ where $\simplex_{S}$ denotes the probability simplex of dimension $S$. For all functions $f:\ \cS\mapsto[0,b]$ defined on $\cS$,
\begin{align*}
	p f - q f &\leq  \sqrt{2\Var_{q}(f)\KL(p,q)}+\frac{2}{3} b \KL(p,q)\\
  q f- p f &\leq  \sqrt{2\Var_{q}(f)\KL(p,q)}\,.
\end{align*}
where use the expectation operator defined as $pf \triangleq \E_{s\sim p} f(s)$ and the variance operator defined as
$\Var_p(f) \triangleq \E_{s\sim p} \big(f(s)-\E_{s'\sim p}f(s')\big)^2 = p(f-pf)^2.$
\end{lemma}

\begin{lemma}
\label{lem:switch_variance_bis}
Let $p,q\in\simplex_{S}$ and a function $f:\ \cS\mapsto[0,b]$, then
\begin{align*}
  \Var_q(f) &\leq 2\Var_p(f) +4b^2 \KL(p,q)\,,\\
  \Var_p(f) &\leq 2\Var_q(f) +4b^2 \KL(p,q).
\end{align*}
\end{lemma}
% \begin{proof}
% Let $\tp$ be the distribution of the pair of random variables $(X,Y)$ where $X,Y$ are i.i.d.\,according to the distribution $p$. Similarly, let $\tq$ be the distribution of the pair of random variables $(X,Y)$ where $X,Y$ are i.i.d.\,according to distribution $q$. Since Kullback–Leibler divergence is additive for independent distributions, we know that
% \[\KL(\tp,\tq) = 2\KL(p,q) \leq 2\KL(p,q).\]
% Using Lemma~\ref{lem:Bernstein_via_kl} for the function $g(x,y) = (f(x)-f(y))^2$ defined on $\cS^2$, such that  $0\leq g\leq b^2,$ we get
% \begin{align*}
%   |\tp g- \tq g| &\leq \sqrt{4\Var_{\tq}(g) \KL(p,q)} + \frac{4}{3}b^2 \KL(p,q)\\
%   &\leq  \sqrt{4 b^2 \KL(p,q) \tq g  } + \frac{4}{3}b^2 \KL(p,q)\\
%   &\leq \frac{1}{2} \tq g + \frac{10}{3} b^2 \KL(p,q)\,,
% \end{align*}
% where in the last line we used $2\sqrt{xy}\leq x +y$ for $x,y\geq 0$. In particular we obtain
% \begin{align*}
%   \tp g &\leq \frac{3}{2} \tq g + \frac{10}{3} b^2\KL(p,q)\\
%   \tq g &\leq 2 \tp g + \frac{20}{3}b^2 \KL(p,q) \,.
% \end{align*}
% To conclude, it remains to note that
% \[ \tp g = 2\Var_p(f) \text{ and } \tq g = 2\Var_q(f). \]
% \end{proof}

\begin{lemma}
	\label{lem:switch_variance}
	For $p,q\in\simplex_{S}$, for $f,g:\cS\mapsto [0,b]$ two functions defined on $\cS$, we have that
	\begin{align*}
 \Var_p(f) &\leq 2 \Var_p(g) +2 b p|f-g|\quad\text{and} \\
 \Var_q(f) &\leq \Var_p(f) +3b^2\|p-q\|_1,
\end{align*}
where we denote the absolute operator by $|f|(s)= |f(s)|$ for all $s\in\cS$.
\end{lemma}
% \begin{proof}
% First note that
% \[
% \Var_p(f) = p(f-g+ g-p g + p g- p f)^2 \leq 2 p(f-g - p f + p g)^2 +2 p(g-p g)^2 = 2\Var_p(f-g)+2\Var_p(g).
% \]
% From the above we can immediately conclude the proof of the first inequality with
% \[
% \Var_p(f-g) \leq p(f-g)^2 \leq b p|f-g|,
% \]
% where we used that for all $s\in\cS$, $0\leq |f(s)-g(s)| \leq b$. For the second inequality, using the Hölder inequality,
% \begin{align*}
% 	\Var_q(f) &= pf^2 - (pf)^2 +(q-p)f^2 + (pf)^2 -(qf)^2 \\
% 	&\leq \Var_p(f) + b^2\|p-q\|_1 +2b^2\|p-q\|_1 \\
% 	&\leq \Var_p(f) +3 b^2 \|p-q\|_1.
% \end{align*}
% \end{proof}

\newpage
\section{Additional Experiments}\label{app:experiments}

In this section we provide details on experiments and additional experiments. The code for experiments could be found by the following link: \url{https://github.com/d-tiapkin/max-entropy-exploration}.

First, we describe our baselines in details.
\begin{itemize}
    \item Random policy $\pi_h(a|s) = 1/A$ for any $h\in[H], (s,a) \in \cS \times \cA$;
    \item Optimal MVEE policy. First we compute the solution to the following convex program
    \begin{align*}
        \max_{d \in \cK_p} \sum_{(s,a) \in \cS \times \cA} -\left( \frac{1}{H} \sum_{h=1}^H d_h(s,a) \right) \log \left(\frac{1}{H} \sum_{h=1}^H d_h(s,a) \right),
    \end{align*}
    where $\cK_p$ is a polytope of admissible visitation distributions defined in Section~\ref{sec:setting}. This problem exactly corresponds to the setting visitation entropy by \citet{hazan2019provably}. Then we compute the optimal MVEE policy by normalization of the visitation distribution $\pi_h(a|s) = \frac{d_h(s,a)}{\sum_{a\in\cA} d_h(s,a)}$.
    \item Optimal MTEE policy that was computed by solving regularized Bellman equations with $\lambda = 1$ and the entropy of transition kernels as rewards.
\end{itemize}
The last two baselines requires the knowledge of the true transition kernel $p_h(s'|s,a)$. Additionally, we perform experiments with \RFUCRL algorithm by \citet{kaufmann2020adaptive} but the final visitation numbers were very close to the visitation numbers of \EntGame algorithm and we do not report them. The data collection procedure for \EntGame and \UCBVIEnt is following.
\begin{itemize}
    \item For \EntGame algorithm we report visitation counts of states of the algorithm during the learning. It is an appropriate choice since this counts corresponds to the empirical density $\rd^T_h(s,a) = \frac{1}{T} \sum_{t=1}^T \ind\{ (s^t_h, a^t_h) = (s,a)\}$ which converges to $d^{\hpi}_h(s,a) = \frac{1}{T} \sum_{t=1}^T d^{\pi^t}_h(s,a)$.
    \item For \UCBVIEnt algorithm we have to separate stages. At the first stage the algorithm interacts with environment to learn the final policy during $N$ interactions with the environment, and all the plots present the visitations counts only for the final policy during another $N$ interactions when the policy is fixed.
\end{itemize}

\begin{remark}
    We do not compare \EntGame with \MaxEnt and \TocUCRL because of their similarity. The only difference in the these algorithms is the way how density estimation is performed. More precisely, in the \MaxEnt algorithm the average of state-action visitation distributions of all the policies played is estimated from scratch at the end of each epoch with additional trajectories. In \EntGame we also estimate this average but with all trajectories collected until now and without need of additional fresh trajectories. Thus, \EntGame can be seen as a version of \MaxEnt that reuses past trajectories instead of collecting new ones. The latter feature is helpful in practice but will not change the rate. That is why we decided not to include \MaxEnt in the experiments. Similar comments hold for \TocUCRL.
\end{remark}

\begin{figure}[h!]
    \centering
    \includegraphics[width=0.48\linewidth]{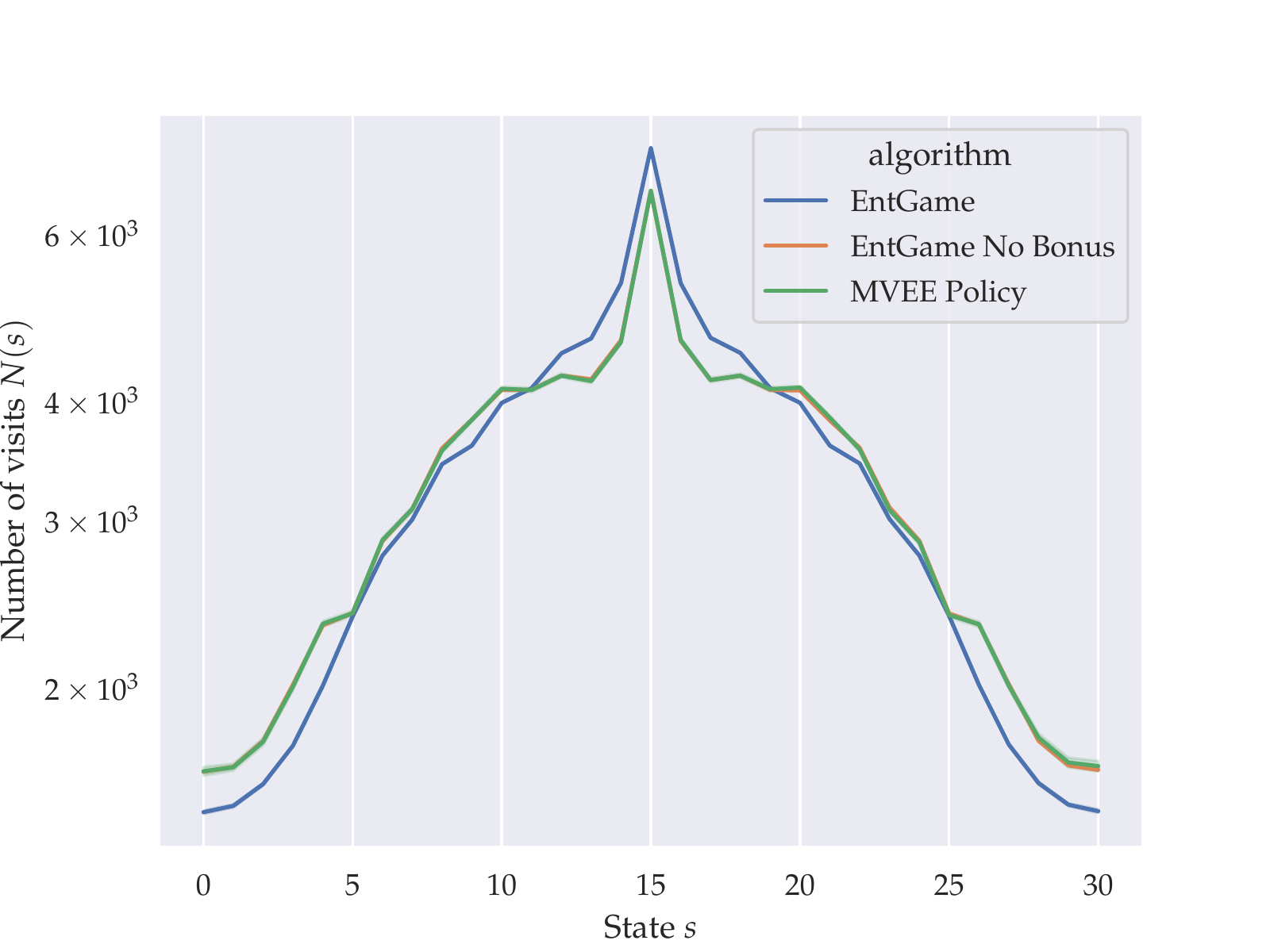}
    \includegraphics[width=0.48\linewidth]{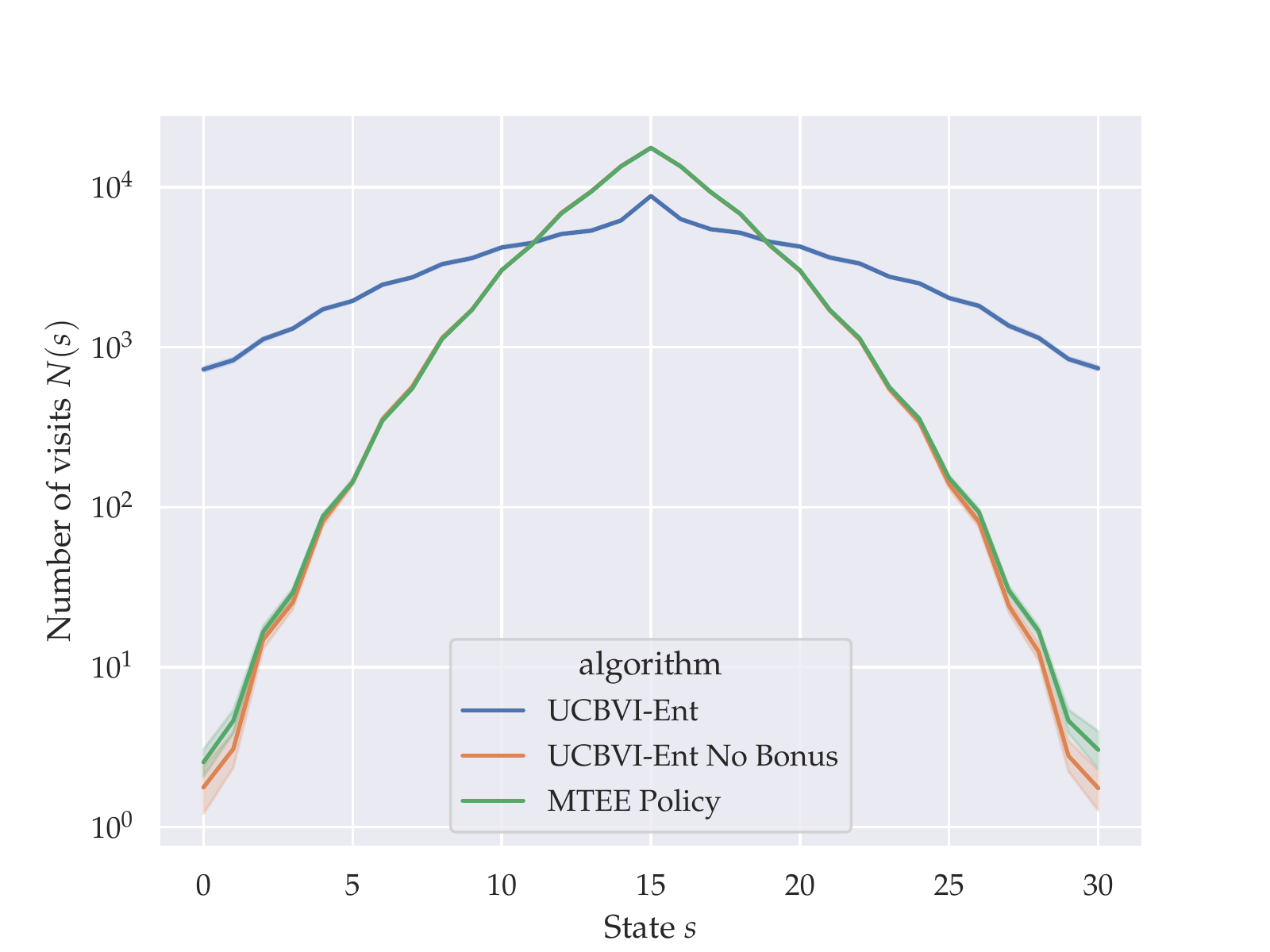}
    \caption{Number of state visits for $N=100000$ samples in the Double Chain MDP for \EntGame and \UCBVIEnt algorithms with and without bonuses.}
    \label{fig:double_chain_no_bonus}
\end{figure}

\paragraph{Double Chain.} The experiment described in Section~\ref{sec:experiments} was preformed on Double Chain environment described by \citet{kaufmann2020adaptive}. This MDP consists of states $\cS = \{0, \ldots, L -1\}$, where $L$ is the length of the chain, the actions $\cA = \{l,r\}$ which corresponds to the transition to the left (action $l$) or to the right (action $r$). Additionally, while taking the actions, there is $10\%$ probability of moving to the opposite direction.  The agent starts at the middle of the chain $s_1 = (L-1)/2$. For experiments we run each algorithm to collect $N=100000$ samples during episodes of horizon $H=20$ and then report the average and confidence intervals over $48$ random seeds.

In Section~\ref{sec:experiments} we conclude that the exploration bonuses have the important role in the state visitations of \UCBVIEnt algorithm and make it close to \RFUCRL by \citet{kaufmann2020adaptive}. As an ablation study we preform the same experiments for algorithms without bonuses. The results are presented in Figure~\ref{fig:double_chain_no_bonus}.

In particular, we see that \EntGame algorithm without bonuses converges to the optimal MVEE policy in terms of the visitation densities that is slightly more spread than \EntGame algorithm with bonuses. Notice that algorithm has its own exploration mechanism since the rewards are equal to $\log((t+SA)/(n^t(s,a)+1))$ that is tightly connected to exploration bonuses. \UCBVIEnt algorithm also has its own exploration method that induced by soft-max policies. In particular, soft-max policies are tightly connected to the Boltzmann exploration with step-size equal to $1$ \cite{sutton1990integrated}. However, it is well-known \cite{cesabianchi2017boltzmann} that the Boltzmann exploration with fixed step-size did not provide efficient way to explore.

\begin{figure}[h!]
    \centering
    \includegraphics[width=0.48\linewidth]{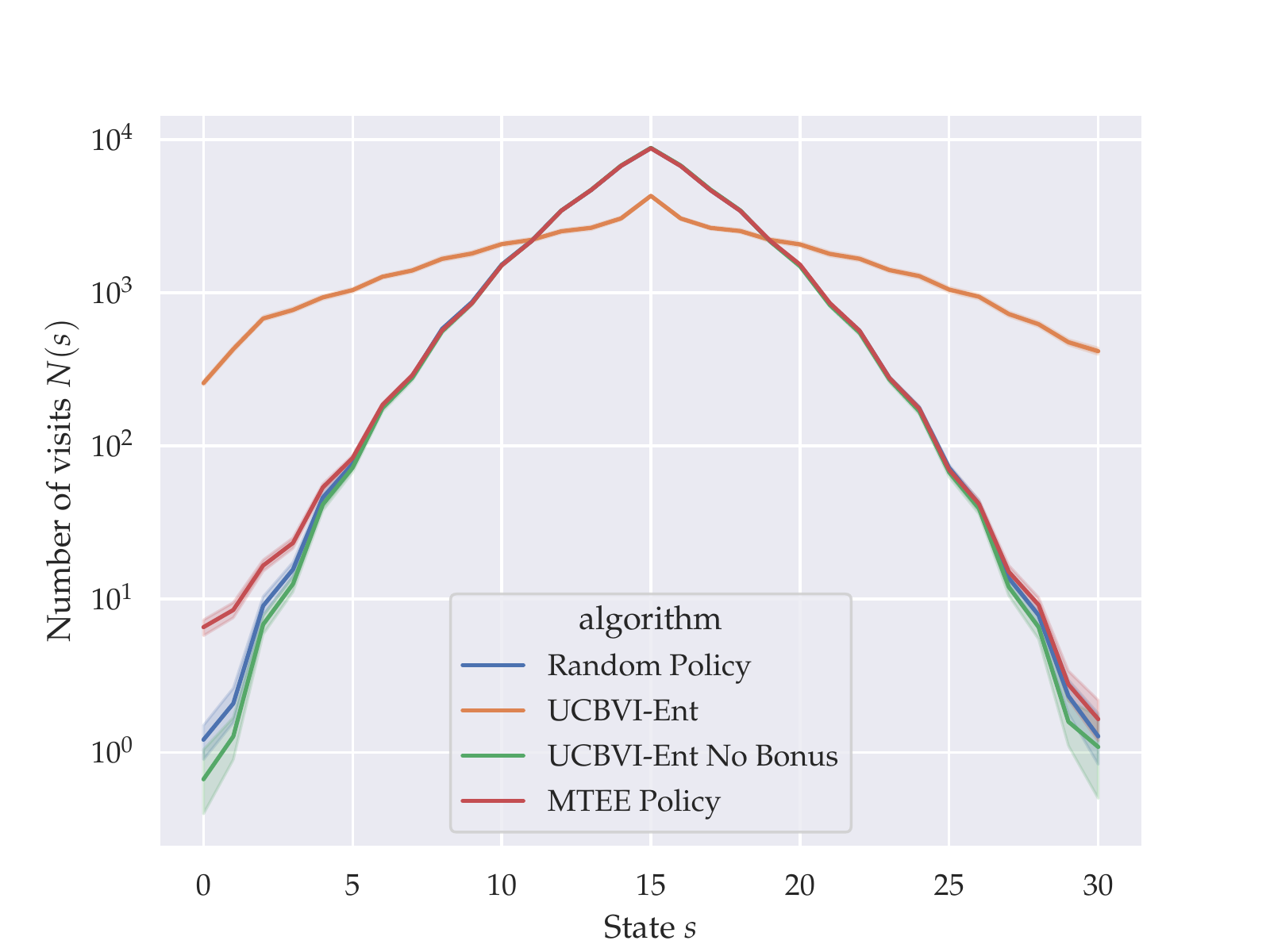}
    \caption{Number of state visits for $N=50000$ samples in the Double Chain MDP with resampling for the \UCBVIEnt algorithm with and without bonuses.}
    \label{fig:double_chain_resampling}
\end{figure}

\paragraph{Double Chain with Resampling} To test the exploration mechanism of \UCBVIEnt without bonuses, we slightly modify Double Chain environment: now the visitation of the left end of the chain leads to uniform resampling over all states. The result is presented in Figure~\ref{fig:double_chain_resampling}. 

In particular, we observe that in this situation \UCBVIEnt algorithm without bonuses still acts like a random policy due to low exploration for the ends of the chain, whereas the optimal MTEE policy visits the left part of the chain more often. This observation imply that the additional exploration mechanism for \UCBVIEnt algorithm is required and the exploration problem in regularized MDPs is non-trivial.

\paragraph{GridWorld} As an additional experiment to verify our findings we perform additional experiments on the environment Grid World as it presented in Figure~\ref{fig:gridworld}. The state space is a set of discrete points in a $21 \times 21$ grid. For each state there are 4 possible actions: left, right, up or down, and for each action there is a $5\%$ probability to move to the wrong direction. The initial state $s_1$ is the middle of the grid. For this experiment we use $N=60000$ samples and report the average over $12$ random seeds.

\begin{figure}[h!]
    \centering
    \includegraphics[width=0.9\linewidth]{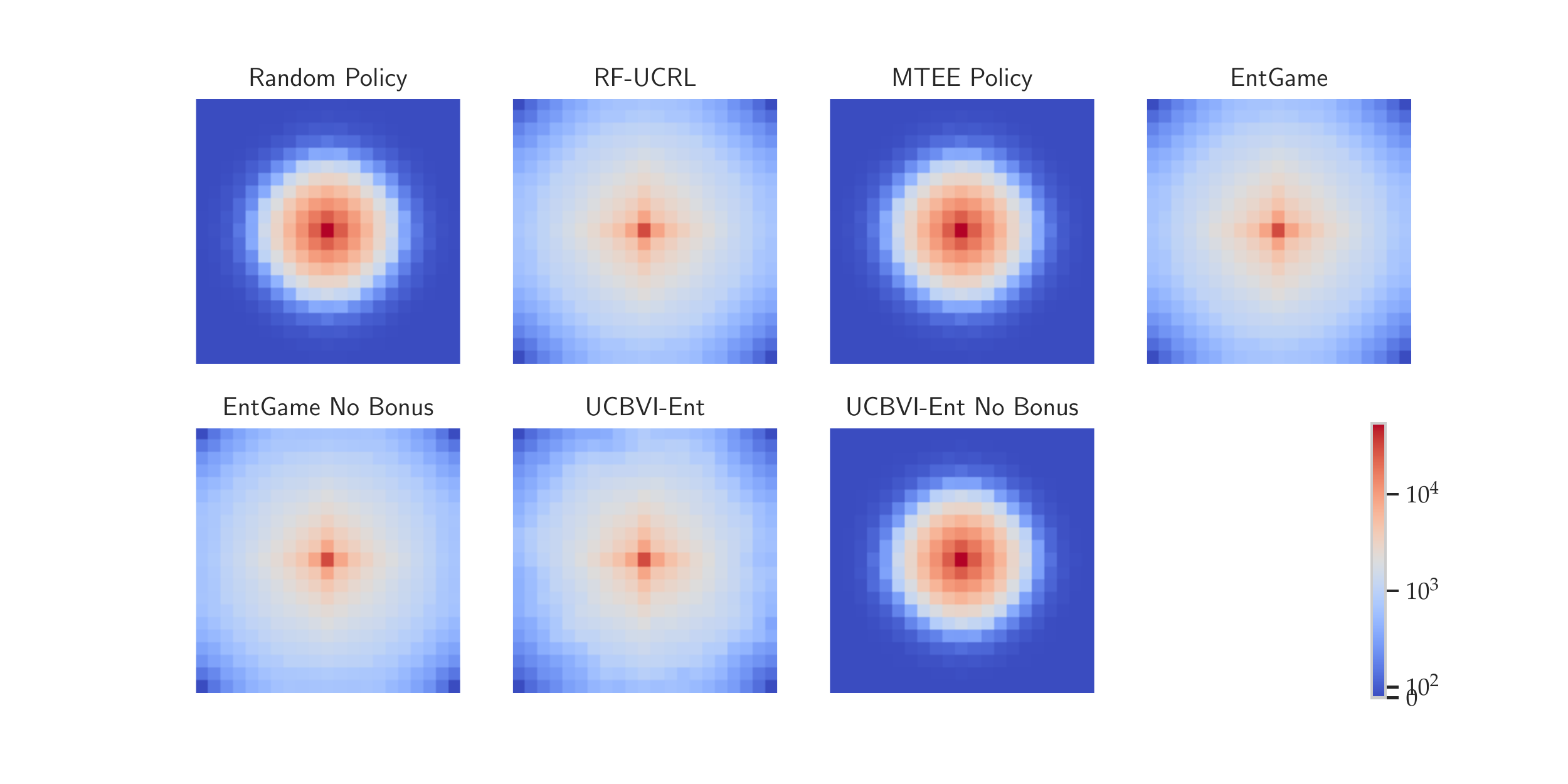}
    \caption{Number of state visits for $N=60000$ samples in the GridWorld MDP for \EntGame and \UCBVIEnt algorithms with and without bonuses.}
    \label{fig:gridworld}
\end{figure}

Here we see the similar effect as on simpler environment: the \UCBVIEnt algorithm produces slightly less "spread" policy than \EntGame or \RFUCRL algorithms. It is connected to the fact that in this case the limiting MTEE policy of \UCBVIEnt is again almost uniform policy due to near-deterministic structure of the MDP. Additionally, we remark that in this case the optimal MVEE policy is much harder to compute due to numerical issues, however, the \EntGame algorithm without bonuses produces slightly more uniform distribution that coincides with the effect we observe during experiments with a Double Chain environment.

\end{document}